\documentclass[11pt]{article}

\usepackage[margin=1.1in]{geometry}

\usepackage[english]{babel}
\usepackage{xspace}
\usepackage{algorithm,algorithmicx}
\usepackage{tcolorbox}
\usepackage{tikz}

\usepackage{framed}
\usepackage{algpseudocode}
\usepackage{amsthm,nccmath}
\usepackage{amsmath,bm}
\usepackage[normalem]{ulem}
\usepackage{bigints}
\usepackage{amssymb}
\usepackage{caption}
\usepackage{cite}
\usepackage{enumerate}
\usepackage{multirow}
\usepackage{hyperref}
\usepackage{booktabs}
\usepackage{hyperref}
\usepackage{url}
\usepackage{color, colortbl}
\usepackage{graphicx}
\usepackage{amsmath,amssymb, bm}
\usepackage{algorithm}
\usepackage{algpseudocode}
\usepackage{xcolor}
\usepackage{flushend}
\usepackage{subfigure}
\usepackage[none]{hyphenat}
\usepackage{enumitem}

\usepackage{hyperref}
\hypersetup{
	colorlinks   = true, 
	urlcolor     = blue, 
	linkcolor    = blue, 
	citecolor    = blue   
}

\usepackage{float}
\usepackage{epstopdf}
\usepackage{footnote}
\makesavenoteenv{tabular}
\makesavenoteenv{table}
\usepackage{bbm}
\usepackage{amsthm}
\usepackage[makeroom]{cancel}
\usepackage{stmaryrd}
\usepackage{amsmath,bm}
\usepackage{amssymb}
\usepackage{mathtools, cuted}
\usepackage{kantlipsum,setspace}

\theoremstyle{plain}
\newtheorem{theorem}{Theorem}[section]
\newtheorem{lem}[theorem]{Lemma}
\newtheorem*{lem*}{Lemma}

\newtheorem{cor}{Corollary}
\newtheorem*{cor*}{Corollary}

\theoremstyle{definition}
\newtheorem{defn}{Definition}[section]
\newtheorem*{defn*}{Definition}
\newtheorem{assump}{Assumption}

\theoremstyle{remark}
\newtheorem{rem}{Remark}
\newtheorem*{rem*}{Remark}

\newcommand{\aname}{{\textsf{STEM}}}
\title{\bf STEM: A Stochastic Two-Sided Momentum Algorithm Achieving Near-Optimal Sample and Communication Complexities for Federated Learning}
\date{} 
\author{\large Prashant Khanduri$^{\ast \dagger}$, Pranay Sharma$^\diamond$, Haibo Yang$^\ast$, Mingyi Hong$^\dagger$, Jia Liu$^\ast$,\\
Ketan Rajawat$^\ddagger$, and Pramod K. Varshney$^\diamond$  \\[.5cm]
	\small $^{\ast }$Department  of Electrical and Computer Engineering, \\
	\small The Ohio State University, OH, USA \\
	\small $^{\dagger}$Department  of Electrical and Computer Engineering, \\
	\small University of Minnesota, MN, USA\\
	\small $^{\diamond}$Department of Electrical Engineering and Computer Science,\\
	\small Syracuse University, NY, USA\\
	\small $^{\ddagger}$Department of Electrical Engineering,\\
	\small Indian Institute of Technology Kanpur, India\\
	\small Email: \texttt{khand095@umn.edu, psharm04@syr.edu,  yang.5952@buckeyemail.osu.edu}\\ 
	\small \texttt{mhong@umn.edu}, \texttt{liu@ece.osu.edu},
	 \texttt{ketan@iitk.ac.in}, \texttt{varshney@syr.edu}} 

\begin{document}

\maketitle

\begin{abstract}
Federated Learning (FL) refers to the paradigm where multiple worker nodes (WNs) build a joint model by using local data. Despite extensive research, for a generic non-convex FL problem, it is not clear, how to choose the WNs' and the server's update directions, the  minibatch sizes, and the local update frequency, so that the WNs use the minimum number of samples and communication rounds to achieve the desired solution. This work addresses the above question and considers a class of stochastic algorithms where the WNs perform a few local updates before communication. We show that when both the WN's and the server's directions are chosen based on a stochastic momentum estimator, the algorithm requires $\tilde{\mathcal{O}}(\epsilon^{-3/2})$ samples and  $\tilde{\mathcal{O}}(\epsilon^{-1})$ communication rounds to compute an $\epsilon$-stationary solution. To the best of our knowledge, this is the first FL algorithm that achieves such {\it near-optimal} sample and communication complexities simultaneously. Further, we show that there is a trade-off curve between local update frequencies and local minibatch sizes, on which the above sample and communication complexities can be maintained. Finally, we show that for the classical FedAvg (a.k.a. Local SGD, which is a momentum-less special case of the \aname), a similar trade-off curve exists,  
albeit with worse sample and communication complexities. Our insights on this trade-off provides guidelines for choosing the four important design elements for FL algorithms, the update frequency, directions, and minibatch sizes to achieve the best performance.  
\end{abstract}

\section{Introduction}
\label{sec: Intro}
In Federated Learning (FL), multiple worker nodes (WNs) collaborate with the goal of learning a joint model, by only using local data. Therefore it has become popular for machine learning problems where datasets are massively distributed 
\cite{Konevcny_Arxiv_2016}. In FL, the data is often collected at or off-loaded to multiple WNs which in collaboration with a server node (SN) jointly aim to learn a centralized model \cite{Li_Smola_NIPS_2014_CommunicationEfficient,Dean_NIPS_2012large}. The local WNs share the computational load and since the data is local to each WN, FL also provides some level of data privacy \cite{Leaute_JAIR_2013}. A classical distributed optimization problem that $K$ WNs aim to solve:
\begin{align}
\label{Eq: Problem}
     \min_{x \in \mathbb{R}^d} ~&\bigg\{  f(x) \coloneqq \frac{1}{K} \sum_{k = 1}^K f^{(k)}(x) \coloneqq \frac{1}{K} \sum_{k = 1}^K \mathbb{E}_{\xi^{(k)} \sim \mathcal{D}^{(k)}}[ f^{(k)}(x; \xi^{(k)})] \bigg\}.
\end{align}
where $f^{(k)} : \mathbb{R}^d \to \mathbb{R}$ denotes the smooth (possibly non-convex) objective function and $\xi^{(k)} \sim \mathcal{D}^{(k)}$ represents the sample/s drawn from distribution $\mathcal{D}^{(k)}$ at the $k^\text{th}$ WN with $k \in [K]$. When the distributions $\mathcal{D}^{(k)}$ are different across the WNs, it is referred to as the \emph{heterogeneous} data setting. 


\begin{figure}[t]

  \centering
 \subfigure[Communication complexity.]{\label{Fig: Trade-off_SAMVR}\includegraphics[width=0.45\linewidth]{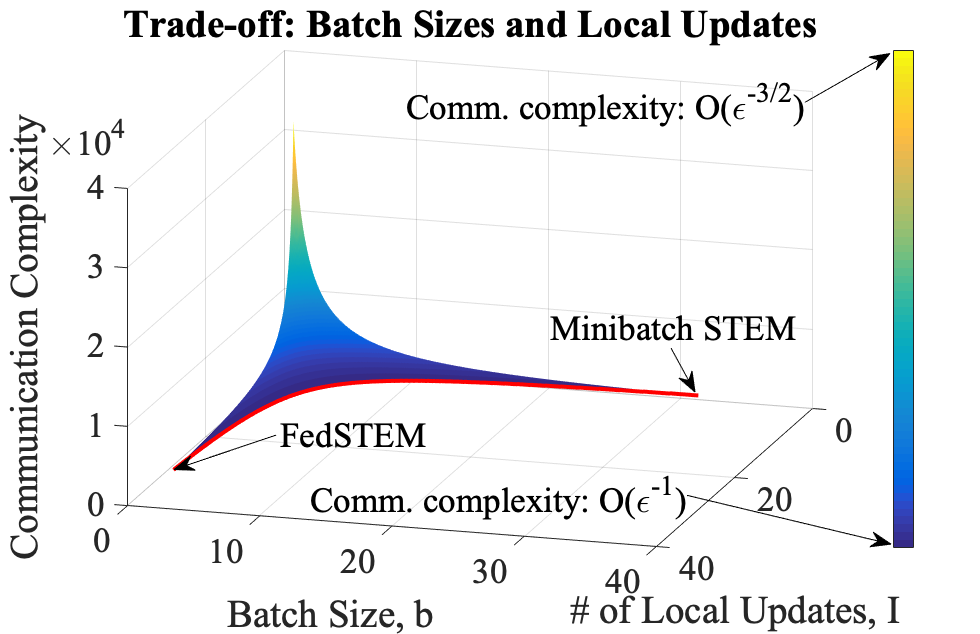}  }  
 \subfigure[Minibatch sizes vs Local Updates.]{\label{Fig: BvsI_SAMVR}\includegraphics[width=0.45\linewidth]{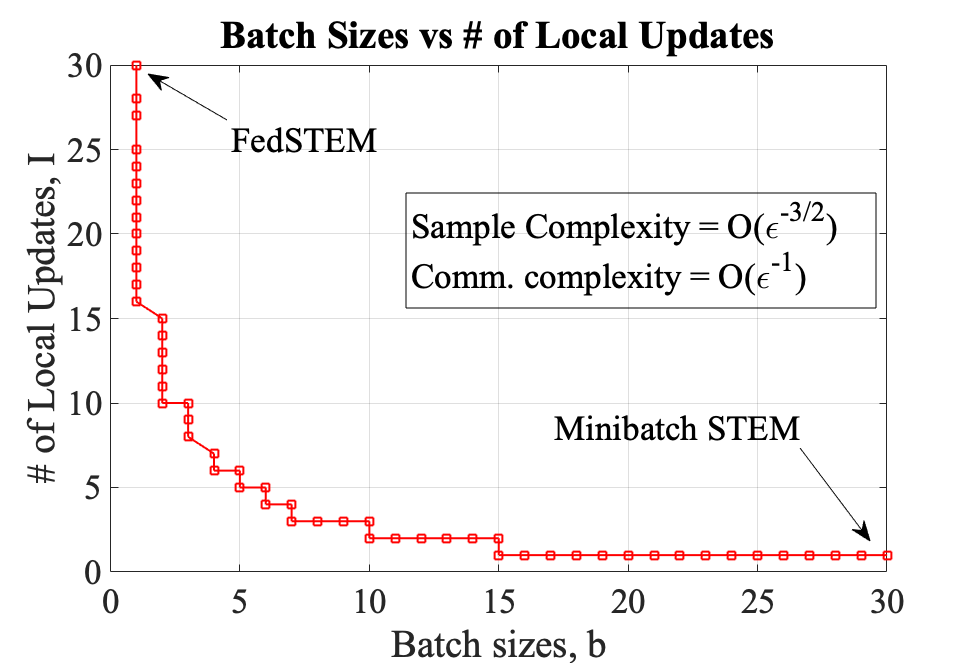}}
\caption{\small The 3D surface in (a) plots the communication complexity of the proposed \aname~for different minibatch sizes and number of local updates. The surface is generated such that each point represents  \aname~ with a particular choice of $(b,I)$, so that it requires $\tilde{\mathcal{O}}(\epsilon^{-3/2})$ samples to achieve $\epsilon$-stationarity, but a . Plot (b) shows the optimal trade off between the minibatch sizes and the number of local updates at each WN (i.e., achieving the lowest communication and sample complexities). Both  plots are generated for an accuracy of $\epsilon = 10^{-3}$ and all the constants dependent on system parameters (variance of stochastic gradients, heterogeneity parameter, optimality gap, Lipschitz constants, etc.) are assumed to be $1$. Fed~\aname~is a special case of \aname~where $\mathcal{O}(1)$ minibatch is used; Minibatch~\aname~is a special case of \aname~where $\mathcal{O}(1)$ local updates are used.} 
\label{Fig: FedAvgvsSAMVR}
\end{figure}


The optimization performance of non-convex FL algorithms is typically measured by the total number of samples accessed (cf. Definition \ref{Def: ComputationComplexity}) and the total rounds of communication (cf. Definition \ref{Def: CommunicationComplexity}) required by each WN to achieve an $\epsilon$-stationary solution (cf. Definition \ref{Def: StationaryPt}). To minimize the sample and the communication complexities, FL algorithms rely on the following {\em four} key design elements: (i) the WNs' local model update directions, (ii) Minibatch size to compute each local direction, (iii) the number of local updates before WNs share their parameters, and (iv) the SN's update direction. How to find effective FL algorithms by (optimally) designing these parameters has received significant research interest recently. 

\noindent{\bf Contributions.} The main contributions of this work are listed below: 

{\bf 1)} We propose the \underline{S}tochastic \underline{T}wo-Sid\underline{e}d \underline{M}omentum (\aname) algorithm, that utilizes a momentum-assisted stochastic gradient directions for {\it both} the WNs and SN updates. 
We show that there exists an {\it optimal} trade off between the minibatch sizes and local update frequency, 
such that on the trade-off curve \aname~requires $\tilde{\mathcal{O}}(\epsilon^{-3/2})\footnote{The notation $\tilde{\mathcal{O}}(\cdot)$ hides the logarithmic factors.}$ samples  and $\tilde{\mathcal{O}}(\epsilon^{-1})$ communication rounds  to reach an $\epsilon$-stationary solution; see Figure \ref{Fig: FedAvgvsSAMVR} for an illustration. 
These complexity results are the best achievable for first-order stochastic FL algorithms (under certain assumptions, cf. Assumption \ref{Ass: Lip_Smoothness}); see \cite{Fang_NIPS_2018_spider, Zhou_NIPS_2018_SNVRG, Cutkosky_NIPS2019, Dinh_Arxiv_2019}   and \cite{Zhang_FedPD_Arxiv_2020, Li_Arxiv_2018_FedProx}, as well as Remark \ref{rmk:complexity} of this paper for discussions regarding optimality. 
To the best of our knowledge, \aname~is the first algorithm which -- (i) {\it simultaneously} achieves the optimal sample and communication complexities for FL and (ii) can optimally trade off the minibatch sizes and the local update frequency. 

{{\bf 2)}  A momentum-less special case of our STEM result further reveals some interesting insights of the classical FedAvg algorithm (a.k.a. the Local SGD) \cite{Mcmahan_PMLR_2017,Yu_Zhu_2018parallel,Woodworth_Minibatch_Arxiv_2020}.
Specifically, we show that for FedAvg, there also exists a trade-off between the minibatch sizes and the local update frequency, on which it  requires $\mathcal{O}(\epsilon^{-2})$ samples and $\mathcal{O}(\epsilon^{-3/2})$ communication rounds to achieve an $\epsilon$-stationary solution.} 

Collectively, our insights on the trade-offs provide practical guidelines for choosing different design elements for FL algorithms.
{\renewcommand{\arraystretch}{1.2}\begin{table*}[t]
		\centering\small
	 {\begin{tabular}{c c c c c c} 
				\toprule  
				 {Algorithm} & Work & Sample &  Communication & Minibatch $(b)$   & Local Updates $(I)$ /round    \\  
				\midrule  	 
			 \hline
				 \multirow{3}{*}{FedAvg$^\diamond$}
				& \cite{Yu_Zhu_2018parallel} /\cite{Yu_Jin_Arxiv_2019linear}  & \multirow{3}{*}{$\mathcal{O}(\epsilon^{-2})$}   & $\mathcal{O}(\epsilon^{-3/2})$ & {$\mathcal{O}(1)$}  & $\mathcal{O}(\epsilon^{-1/2})$ \\    & \cite{Karimireddy_scaffold_2020}/\cite{Yang_ICLR_2021}   &   &  $\mathcal{O}(\epsilon^{-2})$ &  $\mathcal{O}(1)$ & $\mathcal{O}(1)$
				\\   
				& this work   &   & $\mathcal{O}(\epsilon^{-3/2})$ &  $\mathcal{O}\big( \epsilon^{-\frac{2 (1 - \nu)}{(4 - \nu)}} \big)$ &  $\mathcal{O}\big( \epsilon^{-\frac{3\nu}{2(4 - \nu)}} \big)$      \\
				\hline
			SCAFFOLD$^{\ast}$ & \cite{Karimireddy_scaffold_2020}  & $\mathcal{O}(\epsilon^{-2})$ &  $\mathcal{O}(\epsilon^{-2})$ & $\mathcal{O}(1)$ & $\mathcal{O}(1)$ \\
				\hline
				FedPD/FedProx$^\ddagger$ & \cite{Zhang_FedPD_Arxiv_2020}/\cite{Li_Arxiv_2018_FedProx} & $\mathcal{O}(\epsilon^{-2})$ & $\mathcal{O}(\epsilon^{-1})$ & $\mathcal{O}(1)$ & $\mathcal{O}(\epsilon^{-1})$ \\
					\hline
					MIME$^\dagger$/FedGLOMO & \cite{Karimireddy_Arxiv_2020mime}/\cite{Sanghvi_FedSTEPH_2020} & $\mathcal{O}(\epsilon^{-3/2})$ &  $\mathcal{O}(\epsilon^{-3/2})$ & $\mathcal{O}(1)$ & $\mathcal{O}(1)$   \\
					\hline  
			  \rowcolor{gray!10}	 	 \aname$^\diamond$ &     &  &  & $\mathcal{O}\big(\epsilon^{-\frac{3(1 - \nu)}{2(3 - \nu)} } \big)$  &  $\mathcal{O}\big(\epsilon^{-\frac{\nu}{(3 - \nu)}}\big)$ \\
  		  \rowcolor{gray!10}	 Fed \aname &   &     &  & $\mathcal{O}(1)$ & $\mathcal{O}(\epsilon^{-1/2})$ \\
 	\rowcolor{gray!10}		 Minibatch \aname$^*$ &  \multirow{-3}{*}{this work}  & \multirow{-3}{*}{ $\tilde{\mathcal{O}}(\epsilon^{-3/2})$} & \multirow{-3}{*}{ $\tilde{\mathcal{O}}(\epsilon^{-1})$} & ${\mathcal{O}}(\epsilon^{-1/2})$  &  $\mathcal{O}(1)$ \\
				\bottomrule  
		\end{tabular}}
		\caption{\small Comparison of FedAvg and \aname~with different FL algorithms for various choices of the minibatch sizes $(b)$ and the number of per node local updates between two rounds of communication $(I)$. \\
		$^\diamond$ $\nu \in [0,1]$ trades off $b$ and $I$; {$\nu=1$ (resp. $\nu=0$) uses multiple  (resp. $\mathcal{O}(1)$) local  updates and $\mathcal{O}(1)$ (resp. multiple) samples}. Fed~\aname~ and Minibatch~\aname~are two variants of the proposed \aname.\\  
		$^\ddagger$The data heterogeneity assumption is weaker than Assumption \ref{Ass: Unbiased_Var_Grad} (please see \cite{Zhang_FedPD_Arxiv_2020} for details).\\
		$^\dagger$Requires bounded Hessian dissimilarity to model data heterogeneity across WNs.\\
$^*$Guarantees for Minibatch \aname~ with $I = 1$ and SCAFFOLD are independent of the data heterogeneity.}
		\label{Table: Comparison} 
	\end{table*}}

\noindent{\bf Related Works.}  
FL algorithms were first proposed in the form of FedAvg  \cite{Mcmahan_PMLR_2017}, where the local update directions at each WN were chosen to be the SGD updates. Earlier works analyzed these algorithms in the homogeneous data setting 
\cite{Woodworth_Local_Arxiv_2020, Yu_Jin_PMLR_2019dynamicbatches, Wang_Joshi_Arxiv_2018cooperative,  Khaled_Arxiv_2019, Stich_Arxiv_Local_2018, Lin_Don't_Arxiv_2018,Zhou_KStep_IJCAI_2018}, 
while many recent studies have focused on designing new algorithms to deal with heterogeneous data settings, as well as problems where the local loss functions are non-convex 
\cite{Zhang_FedPD_Arxiv_2020,  Li_Arxiv_2018_FedProx, Yu_Zhu_2018parallel, Woodworth_Minibatch_Arxiv_2020,  Yu_Jin_Arxiv_2019linear, Karimireddy_scaffold_2020, Yang_ICLR_2021, Sanghvi_FedSTEPH_2020, Sattler_IEEETNN_2019, Zhao_Arxiv_FedNonIID_2018, Wang_ICLR_2019slowmo, Liang_Arxiv_2019_VRL-SGD, Sharma_Arxiv_2019, Reddi_Arxiv_2019_Adam,Koloskova_PMLR_2020}.  
In \cite{Yu_Zhu_2018parallel}, the authors showed that Parallel Restarted SGD (Local SGD or FedAvg \cite{Mcmahan_PMLR_2017}) achieves linear speed up while requiring $\mathcal{O}(\epsilon^{-2})$ samples and $\mathcal{O}(\epsilon^{-3/2})$ rounds of communication to reach an $\epsilon$-stationary solution. 
In   \cite{Yu_Jin_Arxiv_2019linear}, a Momentum SGD was proposed,  
which achieved the same sample and communication complexities as Parallel Restarted SGD \cite{Yu_Zhu_2018parallel}, without requiring that the second moments of the gradients be bounded. Further, it was shown that under the homogeneous data setting, the communication complexity can be improved to $\mathcal{O}(\epsilon^{-1})$ while maintaining the same sample complexity. The works in \cite{Karimireddy_scaffold_2020, Yang_ICLR_2021} conducted tighter analysis for FedAvg with partial WN participation 
with $\mathcal{O}(1)$ local updates and batch sizes. Their analysis showed that FedAvg's sample and communication complexities are both $\mathcal{O}(\epsilon^{-2})$. Additionally, SCAFFOLD was proposed in \cite{Karimireddy_scaffold_2020}, which utilized variance reduction based local update directions \cite{Johnson_NIPS_2013} to achieve the same sample and communication complexities as FedAvg. Similarly, VRL-SGD proposed in \cite{Liang_Arxiv_2019_VRL-SGD} also utilized variance reduction and showed improved communication complexity of $\mathcal{O}(\epsilon^{-1})$, while requiring the same computations as FedAvg. Importantly, both SCAFFOLD and VRL-SGD's guarantees were independent of the data heterogeneity. The FedProx proposed in \cite{Li_Arxiv_2018_FedProx} used a penalty based method to improve the communication complexity of FedAvg (i.e., the Parallel Restarted and Momentum SGD \cite{Yu_Jin_Arxiv_2019linear, Yu_Zhu_2018parallel}) 
to $\mathcal{O}(\epsilon^{-1})$. 
FedProx used a gradient similarity assumption to model data heterogeneity which can be stringent for many practical applications. This assumption was relaxed by FedPD proposed in \cite{Zhang_FedPD_Arxiv_2020}.  

Recently, the works \cite{Karimireddy_Arxiv_2020mime, Sanghvi_FedSTEPH_2020} 
proposed to utilize hybrid momentum gradient estimators \cite{Cutkosky_NIPS2019, Dinh_Arxiv_2019}. The MIME algorithm  \cite{Karimireddy_Arxiv_2020mime} matched the optimal sample complexity (under certain smoothness assumptions) of $\mathcal{O}(\epsilon^{-3/2})$ of the centralized non-convex stochastic optimization algorithms \cite{Fang_NIPS_2018_spider, Zhou_NIPS_2018_SNVRG, Cutkosky_NIPS2019, Dinh_Arxiv_2019}. Similarly, Fed-GLOMO \cite{Sanghvi_FedSTEPH_2020} achieved the same sample complexity while employing compression to further reduce communication. Both MIME and Fed-GLOMO required $\mathcal{O}(\epsilon^{-3/2})$ communication rounds to achieve an $\epsilon$-stationary solution. Please see Table \ref{Table: Comparison} for a summary of the above discussion.

The comparison of Local SGD (FedAvg) to Minibatch SGD for convex and strongly convex problems with homogeneous data setting was first conducted in \cite{Woodworth_Local_Arxiv_2020} and later extended to heterogeneous setting in \cite{Woodworth_Minibatch_Arxiv_2020}. It was shown that Minibatch SGD almost always dominates the Local SGD. In contrast, it was shown in \cite{Lin_Don't_Arxiv_2018} that Local SGD dominates Minibatch SGD in terms of generalization performance. {Although existing FL results are rich, but they are somehow ad hoc and there is a lack of principled understanding of the algorithms. We note that the proposed \aname~algorithmic framework provides a theoretical framework that unifies all existing FL results on sample and communication complexities.}

\noindent{\bf Notations.} The expected value of a random variable $X$ is denoted by $\mathbb{E}[X]$ and its expectation conditioned on an Event $A$ is denoted as $\mathbb{E}[X| \text{Event}~A]$. We denote by $\mathbb{R}$ (and $\mathbb{R}^d$) the real line (and the $d$-dimensional Euclidean space). The set of natural numbers is denoted by $\mathbb{N}$. Given a positive integer $K \in \mathbb{N}$, we denote $[K] \triangleq \{1,2, \ldots, K\}$. Notation $\| \cdot \|$ denotes the $\ell_2$-norm and $\langle \cdot, \cdot \rangle$ the Euclidean inner product. For a discrete set $\mathcal{B}$, $|\mathcal{B}|$ denotes the cardinality of the set. 
 
\section{Preliminaries}
\label{sec: Problem}

Before we proceed to the the algorithms, we make the following assumptions about problem \eqref{Eq: Problem}.
\begin{assump}[{Sample Gradient Lipschitz Smoothness}]
\label{Ass: Lip_Smoothness}
The stochastic functions $f^{(k)}(\cdot, \xi^{(k)})$ with $\xi^{(k)} \sim \mathcal{D}^{(k)}$ for all $k \in [K]$, satisfy the mean squared smoothness property, i.e, we have 
$$\mathbb{E} \| \nabla f^{(k)} (x ; \xi^{(k)}) -  \nabla f^{(k)} (y ; \xi^{(k)}) \|^2 \leq L^2   \| x - y \|^2~~~\text{for all}~x,y \in \mathbb{R}^d.$$
\end{assump}

\begin{assump}[{Unbiased gradient and Variance Bounds}]
\label{Ass: Unbiased_Var_Grad} 

    \noindent (i) {Unbiased Gradient.} The stochastic gradients computed at each WN are unbiased
    $$\mathbb{E}[\nabla f^{(k)} (x ; \xi^{(k)})] = \nabla f^{(k)} (x),\; \forall~\xi^{(k)} \sim \mathcal{D}^{(k)},\; \forall~k\in[K].$$\\
    
    \vspace{-1cm}
    \noindent (ii) {Intra- and inter- node Variance Bound.} The following bounds hold:
    \begin{align*}
        \mathbb{E}\| \nabla f^{(k)} (x ; \xi^{(k)}) - \nabla f^{(k)} (x) \|^2 \leq \sigma^2, \; \| \nabla f^{(k)} (x) - \nabla f^{(\ell)} (x) \|^2 \leq \zeta^2,\; \forall~\xi^{(k)} \sim \mathcal{D}^{(k)},\; \forall k,\ell \in [K].
    \end{align*}
\end{assump}
Note that Assumption \ref{Ass: Lip_Smoothness} is stronger than directly assuming $f^{(k)}$'s are Lipschitz smooth (which we  will refer to as the {\it averaged} gradient Lipschitz smooth condition), but it is still a rather standard assumption in SGD analysis. For example it has been used in analyzing centralized SGD algorithms such as SPIDER \cite{Fang_NIPS_2018_spider}, SNVRG \cite{Zhou_NIPS_2018_SNVRG}, STORM \cite{Cutkosky_NIPS2019} (and many others) as well as in  FL algorithms such as  MIME  \cite{Karimireddy_Arxiv_2020mime} and Fed-GLOMO \cite{Sanghvi_FedSTEPH_2020}. 
The second relation in Assumption \ref{Ass: Unbiased_Var_Grad}-(ii) 
quantifies the data heterogeneity, and we call $\zeta>0$ as the {\it heterogeneity parameter}. 
This is a typical assumption required to evaluate the performance of FL algorithms. If data distributions across individual WNs are identical, i.e., $\mathcal{D}^{(k)} = \mathcal{D}^{(\ell)}$ for all $k, \ell \in [K]$ then we have $\zeta = 0$.

Next, we define the $\epsilon$-stationary solution for non-convex optimization problems, as well as quantify the computation and communication complexities to achieve an $\epsilon$-stationary point.

\begin{defn}[{$\epsilon$-Stationary Point}]
\label{Def: StationaryPt}
A point $x$ is called $\epsilon$-stationary if $\| \nabla f(x) \|^2 \leq \epsilon$. Moreover, a stochastic algorithm is said to achieve an $\epsilon$-stationary point in $t$ iterations if $\mathbb{E}[\| \nabla f(x_t) \|^2] \leq \epsilon$, where the expectation is over the stochasticity of the algorithm until time instant $t$. 
	\end{defn}
	
	\begin{defn}[{Sample complexity}]
\label{Def: ComputationComplexity}
We assume an Incremental First-order Oracle (IFO) framework \cite{Bottou_SIAM_2018_Review}, where, given a sample $\xi^{(k)} \sim \mathcal{D}^{(k)}$ at the $k^\text{th}$ node and iterate $x$, the oracle returns $(f^{(k)}(x; \xi^{(k)}), \nabla f^{(k)}(x ; \xi^{(k)}) )$. Each access to the oracle is counted as a single IFO operation. We measure the sample (and computational) complexity in terms of the total number of calls to the IFO by all WNs to achieve an $\epsilon$-stationary point given in Definition \ref{Def: StationaryPt}.
	\end{defn}
	
		\begin{defn}[{Communication complexity}]
\label{Def: CommunicationComplexity}
We define a communication round as a one back-and-forth sharing of parameters between the WNs and the SN. Then the communication complexity is defined to be the total number of communication rounds between any WN and the SN required to achieve an $\epsilon$-stationary point given in Definition \ref{Def: StationaryPt}. 
	\end{defn}


\section{The \aname~algorithm and the trade-off analysis}
\label{sec: Algo}
In this section, we discuss the proposed algorithm and present the main results. 
The key in the algorithm design is to carefully balance {\it all the four} design elements mentioned in Sec. \ref{sec: Intro}, so that sufficient and useful progress can be made between two rounds of communication. 

Let us discuss the key steps of \aname, listed in  Algorithm \ref{Algo_DR-STORM_batch}. In Step 10, each node locally updates its model parameters using the local direction $d_{t}^k$, computed by using $b$ stochastic gradients at two consecutive iterates $x_{t+1}^{(k)}$ and $x_{t}^{(k)}$. 
After every $I$ local steps, the WNs share their current local models $\{ x_{t+1}^{(k)} \}_{k = 1}^K$ and directions $\{ d_{t+1}^{(k)} \}_{k = 1}^K$ with the SN. The SN aggregates these quantities, and performs a server-side momentum step, before returning $\bar{x}_{t+1}$ and $\bar{d}_{t+1}$ to all the WNs. Because both the WNs and the SN perform momentum based updates, we call the algorithm a stochastic {\it two-sided} momentum algorithm. 
The key parameters are: $b$ the minibatch size, $I$ the local update steps between two communication rounds,  $\{\eta_t\}$ the stepsizes, and $\{a_t\}$ the momentum parameters. 

{One key technical innovation of our algorithm design is to identify the most suitable way to incorporate momentum based directions in FL algorithms. 
Although the momentum-based gradient estimator itself is not new and has been used in the literature before  (see e.g., in \cite{Cutkosky_NIPS2019, Dinh_Arxiv_2019} and \cite{Karimireddy_Arxiv_2020mime,Sanghvi_FedSTEPH_2020} to improve the sample complexities of centralized and decentralized stochastic optimization problems, respectively), it is by no means clear if and how it can contribute to improve the communication complexity of FL algorithms.} We show that in the FL setting, the local directions together with the local models have to be aggregated by the SN so to avoid being influenced too much by the local data.  More importantly, besides the WNs, the SN also needs to perform updates using the (aggregated) momentum directions. Finally, such {\it two-sided} momentum updates have to be done carefully with the correct choice of minibatch size $b$, and the local update frequency $I$. Overall, it is the judicious choice of all these design elements that results in the optimal sample and communication complexities.

Next, we present the convergence guarantees of the~\aname~algorithm.   

\begin{algorithm}[t]
\caption{The Stochastic Two-Sided Momentum (\aname) Algorithm}
\label{Algo_DR-STORM_batch}
\begin{algorithmic}[1]
	\State{\textbf{Input}: Parameters: 
	$c>0$,  the number of local updates $I$, batch size $b$, stepsizes $\{\eta_t\}$.}
	\State{\textbf{Initialize}:  Iterate $x_1^{(k)} = \bar{x}_1 = \frac{1}{K} \sum_{k = 1}^K x_1^{(k)}$, descent direction $d_1^{(k)} = \bar{d}_1 = \frac{1}{K} \sum_{k = 1}^K d_1^{(k)}$ with $d_1^{(k)} = \frac{1}{B} \sum_{\xi_1^{(k)} \in \mathcal{B}^{(k)}_1} \nabla f^{(k)}(x_1^{(k)} ; \xi_1^{(k)})$  and $|\mathcal{B}^{(k)}_1| = B$ for $k \in [K]$.} \\
	{Perform: $x^{(k)}_2 = x^{k}_1 - \eta_1  d^{(k)}_{1}, \; \forall~k\in[K]$}
	\For{$t = 1$ to $T$}
    	\For{$k = 1$ to $K$} \qquad\qquad\qquad\qquad\qquad\qquad \qquad \qquad ~~\quad \# at the WN
        	\State{$\!\!\! \displaystyle d_{t+1}^{(k)} = \frac{1}{b} \!\!\!\sum_{\xi_{t+1}^{(k)} \in \mathcal{B}_{t+1}^{(k)}} \!\!\!\!\!\! \nabla f^{(k)}(x_{t+1}^{(k)} ; \xi_{t+1}^{(k)}) +  (1 - a_{t+1})      \Big(  d_{t}^{(k)}   -   \frac{1}{b} \!\!\! \sum_{\xi_{t+1}^{(k)} \in \mathcal{B}_{t+1}^{(k)}}\!\!\! \!\!\!\nabla f^{(k)}(x_{t}^{(k)} ; \xi_{t+1}^{(k)}) \Big)$ \text{with}~$|\mathcal{B}_{t+1}^{(k)}| = b$, $a_{t+1}\! =\! c \eta_{t}^2$}
        	\State{{\bf if} $t~\text{mod}~I = 0$ {\bf then} ~~\quad\qquad\qquad\quad\quad\quad \qquad\quad\quad\quad\quad\quad\# at the SN}
        	        	\State{\quad $ d_{t + 1}^{(k)} = \bar{d}_{t + 1} \coloneqq \frac{1}{K} \sum_{k=1}^K d_{t + 1}^{(k)}$}
        			\State{\quad $ x_{t + 2}^{(k)}   \coloneqq  \bar{x}_{t + 1} - \eta_{t+1} \bar{d}_{t+1} =\frac{1}{K} \sum_{k=1}^K x_{t + 1}^{(k)} - \eta_{t+1} \bar{d}_{t+1}$~ \# server-side momentum step} 

        	\State{{\bf else} $ x_{t+2}^{(k)} =  x_{t+1}^{(k)} - \eta_{t+1} d_{t+1}^{(k)}$ \qquad\qquad\qquad~\quad\quad\quad\quad\quad   \# worker-side momentum step}
        	\State{\bf end if} 
        	
        	        	
	    \EndFor
	\EndFor
\State{{\bf Return:} $\bar{x}_a$ chosen uniformly randomly from $\{\bar{x}_t\}_{t=1}^T$}	
\end{algorithmic}
\end{algorithm}

\subsection{Main results: convergence guarantees for \aname}
\label{Sec: Convergence}
In this section, we analyze the performance of \aname. We first   present our main result, and then provide discussions about a few parameter choices. In the next subsection, we discuss a special case of \aname~ related to the classical FedAvg and minibatch SGD algorithms.  
\begin{theorem}
\label{Thm: PR_Convergence_Main}
Under the Assumptions \ref{Ass: Lip_Smoothness} and \ref{Ass: Unbiased_Var_Grad}, suppose the stepsize sequence is chosen as: 
\begin{align}
 \eta_{t} & = \frac{\bar{\kappa}}{(w_t + \sigma^2 t )^{1/3}},
 \label{eq: Step-Size}
\end{align}
  where we define : 
    \begin{align*}
 \bar{\kappa} & = \frac{(bK)^{2/3} \sigma^{2/3}}{L}, \quad  w_t  = \max \bigg\{2 \sigma^2,  4096 L^3 I^3\bar{\kappa}^3 - \sigma^2t,  \frac{c^3 \bar{\kappa}^3 }{4096 L^3I^3} \bigg\}. ~ 
\end{align*}
Further, let us set $c = \frac{64L^2}{bK} + \frac{\sigma^2}{24  \bar{\kappa}^3 LI} = L^2 \bigg(\frac{64}{bK} + \frac{1}{24 (bK)^{2}  I} \bigg)$, and set the initial batch size as $B = bI$; set the local updates $I$ and minibatch size $b$ as follows:
\begin{align}
\label{eq: I:b_SAMVR}
 I = \mathcal{O}\big(({T}/{K^2})^{\nu/3}\big), \quad b = \mathcal{O} \big(({T}/{K^2})^{{1}/{2} - {\nu}/{2}} \big)
\end{align}
where $\nu$ satisfies $\nu \in [0, 1]$. Then for ~\aname~ the following holds:
\begin{enumerate}[leftmargin = 0.6 cm, label = (\roman*)]
    \item For $\bar{x}_a$ chosen according to Algorithm \ref{Algo_DR-STORM_batch}, we have: 
    \begin{align}\label{eq:complexity}
 \mathbb{E}\| \nabla f(\bar{x}_a) \|^2 = \mathcal{O}\bigg( \frac{f(\bar{x}_1) - f^\ast}{K^{2\nu/3}T^{1 - \nu/3}} \bigg) + \tilde{\mathcal{O}}\bigg( \frac{\sigma^2 }{K^{2\nu/3}T^{1 - \nu/3}}\bigg) + \tilde{\mathcal{O}} \bigg(  \frac{\zeta^2}{K^{2\nu/3}T^{1 - \nu/3}} \bigg).
\end{align}
\item For any $\nu \in [0,1]$, we have\\
{\bf \em Sample Complexity:}  The sample complexity of \aname~is $\tilde{\mathcal{O}}(\epsilon^{-3/2})$. This implies that each WN requires at most $\tilde{\mathcal{O}}(K^{-1} \epsilon^{-3/2})$ gradient computations, thereby achieving linear speedup with the number of WNs present in the network. \\ 
{\bf \em Communication Complexity:} The communication complexity of \aname~is $\tilde{\mathcal{O}}(\epsilon^{-1})$.
\end{enumerate}
\end{theorem}
The proof of this result is relegated to the Supplemental Material. A few remarks are in order. 

\begin{rem}[Near-Optimal sample and communication complexities]\label{rmk:complexity} Theorem \ref{Thm: PR_Convergence_Main} suggests that when $I$ and $b$ are selected appropriately, then \aname~achieves $\tilde{\mathcal{O}}(\epsilon^{-3/2})$ and $\tilde{\mathcal{O}}(\epsilon^{-1})$ sample and communication complexities. Taking them separately, these complexity bounds are the best achievable by the existing FL algorithms (upto logarithmic factors regardless of sample or batch Lipschitz smooth assumption) \cite{Drori_ICML_2020complexity}; see Table \ref{Table: Comparison}. 
We note that the $\mathcal{O}(\epsilon^{-3/2})$ complexity is the best possible that can be achieved by centralized SGD with the sample Lipschitz gradient assumption; see \cite{Fang_NIPS_2018_spider}. On the other hand, the $\mathcal{O}(\epsilon^{-1})$ complexity bound is also {\it likely} to be the optimal, since in \cite{Zhang_FedPD_Arxiv_2020} the authors showed that even when the local steps use a class of (deterministic) first-order algorithms, $\mathcal{O}(\epsilon^{-1})$ is the best achievable communication complexity. The only difference is that \cite{Zhang_FedPD_Arxiv_2020} does not explicitly assume the intra-node variance bound (i.e., the second relation in Assumption \ref{Ass: Unbiased_Var_Grad}-(ii)). We leave the precise characterization of the communication lower bound with intra-node variance as future work.  \qed   
\end{rem}

\begin{rem}[The Optimal Batch Sizes and Local Updates Trade-off]
\label{Rem: Tradeoff}
The parameter $\nu \in [0,1]$ is used to balance the local minibatch sizes $b$, and the number of local updates $I$. Eqs. in \eqref{eq: I:b_SAMVR} suggest that when $\nu$ increases from 0 to 1, $b$ decreases and  $I$ increases. Specifically, if $\nu = 1$, then $b$ is only $\mathcal{O}(1)$ but $I = \mathcal{O}(T^{1/3}/K^{2/3})$. In this case, each WN chooses a small minibatch while executing multiple local updates, and \aname~resembles a FedAvg (a.k.a. Local SGD) algorithm but with double-sided momentum update directions, {and is referred to as Fed \aname.} 
In contrast, if $\nu = 0$, then $b = \mathcal{O}(T^{1/2}/K)$ but $I$ is  only $\mathcal{O}(1)$. In this case, each WN chooses a large batch size while executing only a few, or even one, local updates, and \aname~resembles the Minibatch SGD, but again with different update directions, {and is referred to as Minibatch \aname}. Such a trade-off can be seen in Fig. \ref{Fig: BvsI_SAMVR}. Due to space limitation, these two special cases will be precisely stated in the supplementary materials as corollaries of Theorem \ref{Thm: PR_Convergence_Main}. 
\qed
\end{rem}

\begin{rem}[The Sub-Optimal Batch Sizes and Local Updates Trade-off]
{From our proof (Theorem \ref{Thm: PR_Convergence_Main_App} included in the supplemental material), we can see that \aname~requires 
 $ \tilde{\mathcal{O}} \big( \max  \big\{
    (b\cdot I) \epsilon^{-1},  K^{-1} \epsilon^{-3/2} \big\} \big)$ samples and $\tilde{\mathcal{O}} \big( \max  \big\{
    \epsilon^{-1}, (b\cdot I)^{-1} K^{-1} \epsilon^{-3/2}\big\} \big)$ and communication rounds. According to the above expressions, if $b \cdot I$ increases beyond $\mathcal{O} (K^{-1} \epsilon^{-1/2})$, then
    the sample complexity will increase from the optimal $\tilde{\mathcal{O}}(\epsilon^{-3/2})$; otherwise, 
    the optimal sample complexity $\tilde{\mathcal{O}}(\epsilon^{-3/2})$ is maintained.
    On the other hand, if $b \cdot I$ decreases beyond $\mathcal{O}(K^{-1}\epsilon^{-1/2})$, the communication complexity increases from  $\tilde{\mathcal{O}}(\epsilon^{-1})$. 
    For instance, if we choose $b = \mathcal{O}(1)$ and $I = \mathcal{O}(1)$ the communication complexity becomes $\tilde{\mathcal{O}}(\epsilon^{-3/2})$.  This trade-off is illustrated in Figure \ref{Fig: Trade-off_SAMVR}, where we maintain the optimal sample complexity, while changing $b$ and $I$ to generate the trade-off surface. 
}
\qed
\end{rem}



\begin{rem}[Data Heterogeneity]
\label{Rem: Heterogeneity}
The term $\tilde{\mathcal{O}} \Big(  \frac{\zeta^2}{K^{2\nu/3}T^{1 - \nu/3}} \Big)$ in the gradient bound \eqref{eq:complexity} captures the effect of the heterogeneity of data across WNs, where $\zeta$ is the parameter characterizing the intra-node variance and has been defined in Assumption \ref{Ass: Unbiased_Var_Grad}-(ii). Highly heterogeneous data with large $\zeta^2$ can adversely impact the performance of \aname. Note that such a dependency on $\zeta$ also appears in other existing FL algorithms, such as \cite{Zhang_FedPD_Arxiv_2020, Yu_Jin_Arxiv_2019linear,  Sanghvi_FedSTEPH_2020}.  
However, there is one special case of \aname~ that does not depend on the parameter $\zeta$. This is the case where $I=1$, i.e., the minibatch SGD counterpart of \aname~where only a single local iteration is performed between two communication rounds. We have the following corollary.\qed
\end{rem}
 

\begin{cor}[Minibatch \aname]
\label{cor: Batches}
Under Assumptions \ref{Ass: Lip_Smoothness} and \ref{Ass: Unbiased_Var_Grad} , and choose the algorithm parameters as in Theorem \ref{Thm: PR_Convergence_Main}. At each WN, choose $I = 1$, $\displaystyle b = (T/K^2)^{1/2}$, and the initial batch size $B = b\cdot I$. Then \aname~satisfies:
\begin{enumerate}[leftmargin = 0.6 cm, label = (\roman*)]
    \item For $\bar{x}_a$ chosen according to Algorithm \ref{Algo_DR-STORM_batch}, we have 
   \begin{align*}
           \mathbb{E}\| \nabla f(\bar{x}_a) \|^2 =  \mathcal{O}\Big( \frac{f(\bar{x}_1) - f^\ast}{T} \Big) + \tilde{\mathcal{O}}\Big( \frac{\sigma^2}{T}\Big). 
   \end{align*}
\item Minibatch \aname~achieves $\mathcal{\tilde{O}}(\epsilon^{-3/2})$ sample and $\mathcal{\tilde{O}}(\epsilon^{-1})$ communication complexity. 
\end{enumerate}
\end{cor}
Next, we show that FedAvg also exhibits a trade-off similar to that of \aname~but with worse sample and communication complexities.

\subsection{Special cases: The FedAvg algorithm}
\label{Sec: Special Case}
We briefly discuss another interesting special case of \aname, where the local momentum update is replaced by the conventional SGD (i.e., $a_{t}=1,\;\forall~t$), while the server does not perform the momentum update (i.e., $\bar{d}_t =0, \forall~t$). {This is essentially the classical FedAvg algorithm, just that it balances the number of local updates $I$ and the minibatch size $b$.}  We show that this algorithm also exhibits a trade-off between $b$ and $I$ and on the trade-off curve it achieves $\mathcal{O}(\epsilon^{-2})$ sample complexity and $\mathcal{O}(\epsilon^{-3/2})$ communication complexity. 

\begin{algorithm}[t]
\caption{The FedAvg Algorithm}
\label{Algo_FedAvg}
\begin{algorithmic}[1]
	\State{\textbf{Input}: 
	$\{\eta_t\}_{t=0}^{T}$; $I$, the \# of local updates per communication rounds; $b$,  the minibatch sizes.}
	\For{$t = 1$ to $T$}
    	\For{$k = 1$ to $K$}
    	\State{$d_{t}^{(k)} = \frac{1}{b} \sum_{\xi_{t}^{(k)} \in \mathcal{B}_{t}^{(k)}}    \nabla f^{(k)}(x_{t}^{(k)} ; \xi_{t}^{(k)}) $  \text{with}~$|\mathcal{B}_{t}^{(k)}| = b$}
        	\State{$ x_{t+1}^{(k)} =  x_t^{(k)} - \eta_{t} d_{t}^{(k)}$}
        	\State{{\bf if} $t~\text{mod}~I = 0$ {\bf then}}
        			\State{\quad $ x_{t + 1}^{(k)} = \bar{x}_{t + 1} = \frac{1}{K} \sum_{k=1}^K x_{t + 1}^{(k)}$}
        	\State{{\bf end if}}
	    \EndFor
	\EndFor
\State{{\bf Return:} $\bar{x}_a$ chosen uniformly randomly from $\{\bar{x}_t\}_{t=1}^T$}	
\end{algorithmic}
\end{algorithm}

\begin{theorem}[The FedAvg Algorithm]
\label{Thm: Flexible_FedAvg}
Under Assumptions \ref{Ass: Lip_Smoothness} and \ref{Ass: Unbiased_Var_Grad}, suppose the stepsize is chosen as: $\eta = \sqrt{\frac{bK}{T}}$; Let us set:
\begin{align}\label{eq:I:b_FedAvg}
 I = \mathcal{O}\big( (T/K^3)^{\nu/4} \big), \quad b = \mathcal{O}\big( (T/K^3)^{1/3 - \nu/3} \big)
\end{align}
where $\nu \in [0, 1]$ is a constant. Then for FedAvg with $T \geq 81 L^2I^2 bK$, the following holds 


\begin{enumerate}[label = (\roman*), leftmargin= 0.6 cm]
\item For $\bar{x}_a$ chosen according to Algorithm \ref{Algo_FedAvg}, we have 
\begin{align*}
\mathbb{E}\| \nabla f(\bar{x}_a)\|^2 = \mathcal{O}\bigg( \frac{f(\bar{x}_1) - f^\ast}{K^{\nu/2} T^{2/3 - \nu/6}}\bigg) + \mathcal{O}\bigg( \frac{\sigma^2}{K^{\nu/2} T^{2/3 - \nu/6}} \bigg)  + \mathcal{O}\bigg( \frac{\zeta^2}{K^{\nu/2} T^{2/3 - \nu/6}} \bigg). 
\end{align*}
\item For any choice of $\nu \in [0,1]$ we have:\\
{\bf \em Sample Complexity:} The sample complexity of FedAvg  is $\mathcal{O}(\epsilon^{-2})$. This implies that each WN requires at most $\mathcal{O}(K^{-1}\epsilon^{-2})$ gradient computations, thereby achieving linear speedup with the number of WNs in the network.\\ 
{\bf \em Communication Complexity:} The communication complexity of FedAvg is $\mathcal{O}(\epsilon^{-3/2})$.  
\end{enumerate} 
\end{theorem}
{Note that the requirement on $T$ being lower bounded is only relevant for theoretical purposes, a similar requirement was also imposed in \cite{Yu_Jin_Arxiv_2019linear} to prove convergence.}
Again, the parameter $\nu \in [0,1]$ in the statement of Theorem \ref{Thm: Flexible_FedAvg} balances $I$ and $b$ at each WN while maintaining {state-of-the-art} sample and communication complexities; please see Table \ref{Table: Comparison} for a comparison of those bounds with existing FedAvg bounds. 
For $\nu = 1$, FedAvg (cf. Theorem \ref{Thm: Flexible_FedAvg}) reduces to FedAvg proposed in \cite{Yu_Zhu_2018parallel, Yu_Jin_Arxiv_2019linear} and for $\nu = 0$, the algorithm can be viewed as a large batch FedAvg with constant local updates \cite{Karimireddy_scaffold_2020,Yang_ICLR_2021}. Note that similar to \aname, it is known that for $I = 1$, the Minibatch SGD's performance is independent of the heterogeneity parameter, $\zeta$ \cite{Woodworth_Minibatch_Arxiv_2020}. We also point out that if Algorithm \ref{Algo_DR-STORM_batch} uses Nesterov's or Polyak's momentum \cite{Yu_Jin_Arxiv_2019linear} at local WNs instead of the recursive momentum estimator we get the same guarantees as in Theorem \ref{Thm: Flexible_FedAvg}.



{In summary, this section established that once the WN's and the SN's update directions (SGD in FedAvg and momentum based directions in \aname) are fixed, there exists a sequence of optimal choices of the number of local updates $I$, and the batch sizes $b$, which guarantees the best possible sample and communication complexities for the particular algorithm. 
The trade-off analysis presented in this section provides some useful guidelines for how to best select $b$ and $I$ in practice. Our subsequent numerical results will also verify that if $b$ or $I$ are not chosen judiciously, then the practical performance of the algorithms can degrade significantly. }

 \section{Numerical results}
 \label{sec: Experiments}
 In this section, we validate the proposed ~\aname~ algorithm and compare its performance with the de facto standard FedAvg \cite{Mcmahan_PMLR_2017} and recently proposed SCAFFOLD \cite{Karimireddy_scaffold_2020}. The goal of our experiments are three-fold: (1) To show that \aname~performs on par, if not better, compared to other algorithms in both moderate and high heterogeneity settings, (2) there are multiple ways to reach the desired solution accuracy, one can either choose a large batch size and perform only a few local updates or select a smaller batch size and perform multiple local updates, and finally, (3) if the local updates and the batch sizes are not chosen appropriately, the WNs might need to perform excessive computations to achieve the desired solution accuracy, thereby slowing down convergence. 

  \begin{figure}[t]
     \centering
		\begin{minipage}[t]{0.45\textwidth}
        \flushright
        \includegraphics[width=1\linewidth,height = 2.2 in]{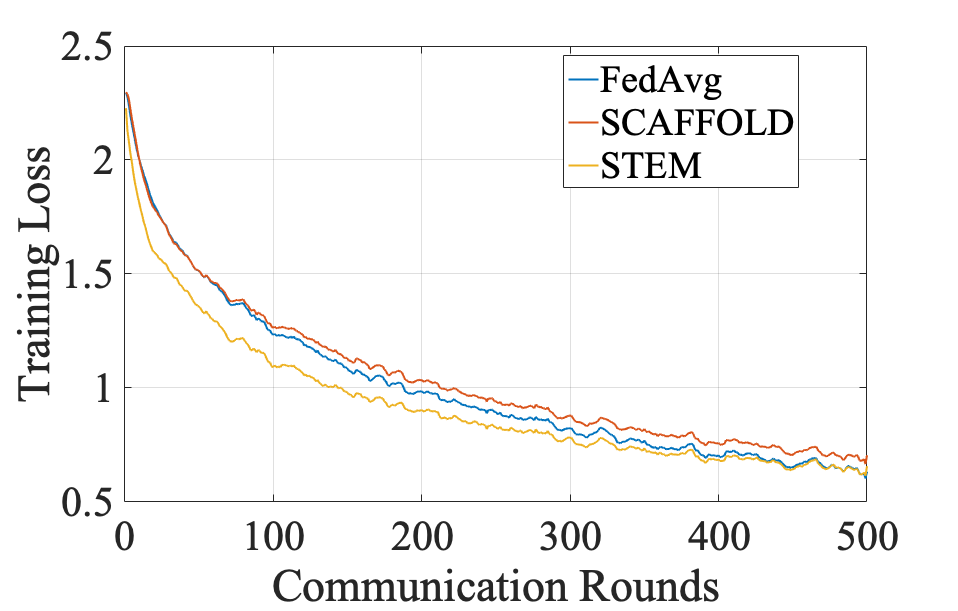}
        \captionsetup{labelformat=empty}
        \end{minipage}
        \begin{minipage}[t]{0.45\textwidth}
        \flushleft
        \includegraphics[width=1\linewidth,height = 2.2 in]{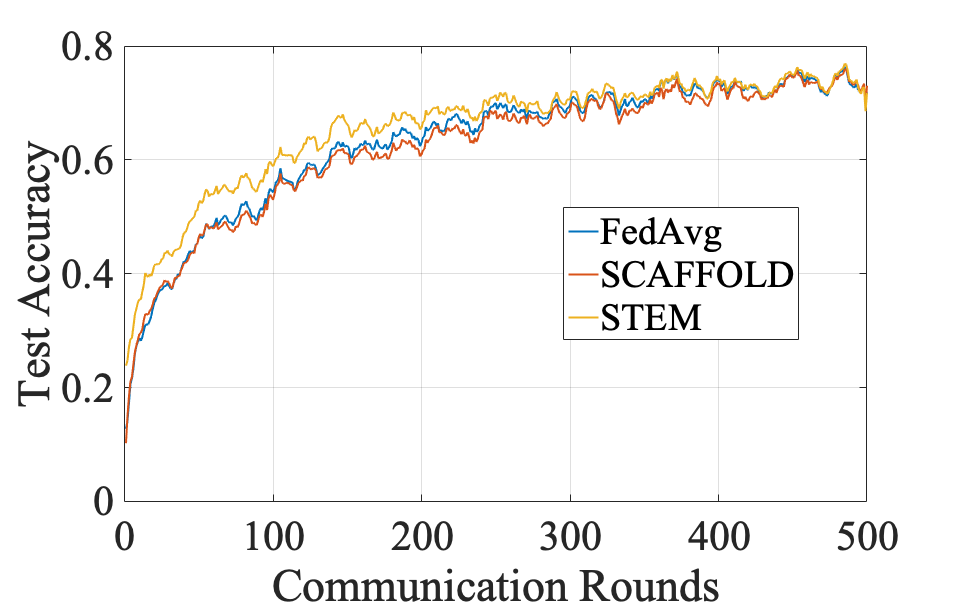}
        \captionsetup{labelformat=empty}
        \end{minipage}
     \caption{Training loss and testing accuracy for classification on CIFAR-10 dataset against the number of communication rounds for moderate heterogeneity setting with $b = 8$ and $I = 61$. }
     \label{Fig: CIFAR_b8e1}
 \end{figure}
 
 \begin{figure}[t]
		\centering
        \begin{minipage}[b]{0.45\textwidth}
        \flushright
        \includegraphics[width=1\linewidth, height = 2.2 in]{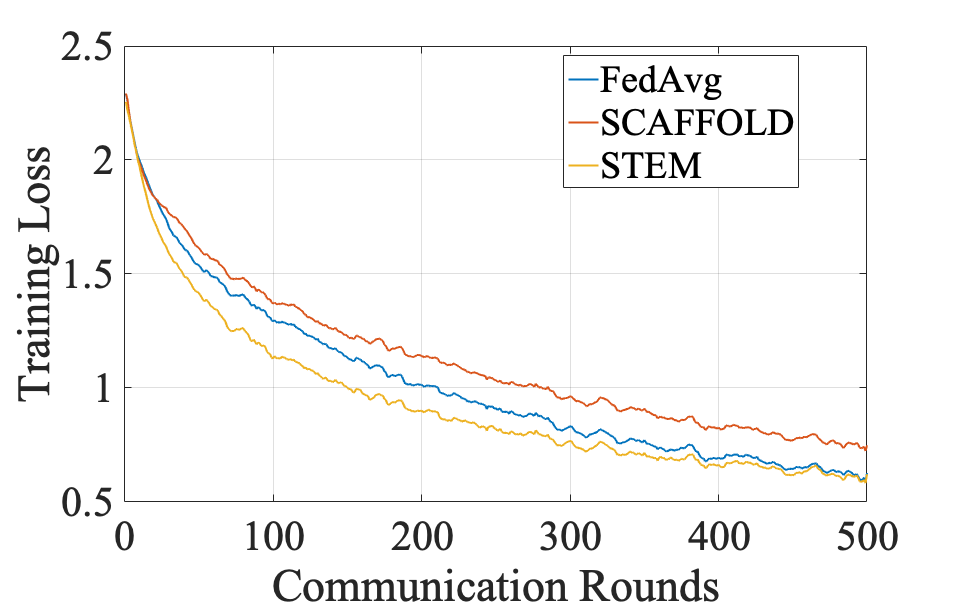}
        \captionsetup{labelformat=empty}
        \end{minipage}
        \begin{minipage}[t]{0.45\textwidth}
        \flushleft
        \includegraphics[width=1\linewidth,height = 2.2 in]{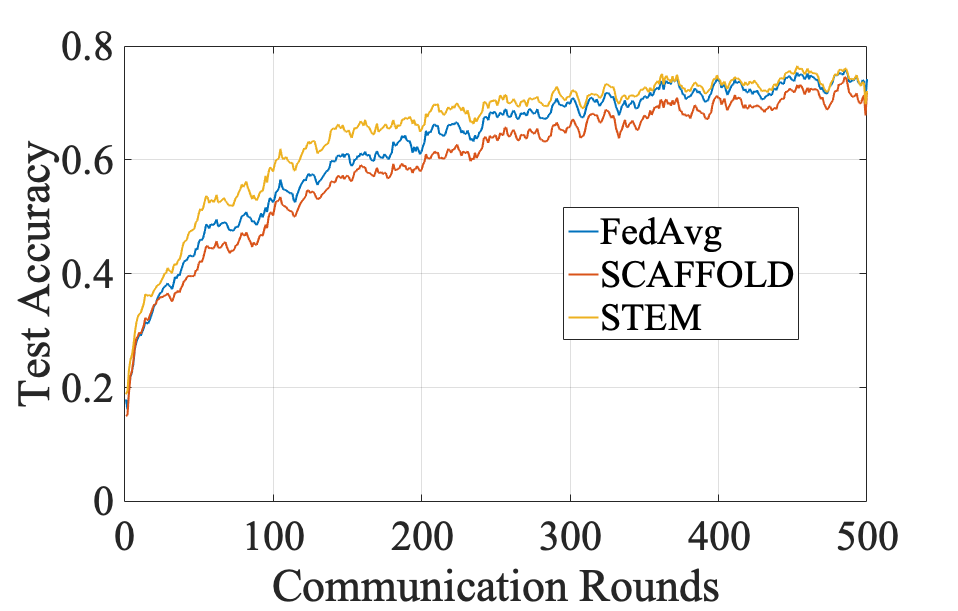}
        \captionsetup{labelformat=empty}
        \end{minipage}
        \caption{Training loss and testing accuracy for classification on CIFAR-10 dataset against the number of communication rounds for moderate heterogeneity setting with $b = 64$ and $I = 7$.}
        \label{Fig: CIFAR_b64e1}  
	\end{figure} 
 
\noindent{\bf Data and Parameter Settings:} We compare the algorithms for image classification tasks on CIFAR-10 and MNIST data sets with $100$ WNs in the network. For both CIFAR-10 and MNIST, each WN implements a two-hidden-layer convolutional neural network (CNN) architecture followed by three linear layers for CIFAR-10 and two for MNIST. All the experiments are implemented on a single NVIDIA Quadro RTX 5000 GPU. We consider two settings, one with moderate heterogeneity and the other with high heterogeneity. For both settings, the data is partitioned into disjoint sets among the WNs. In the moderate heterogeneity setting, the WNs have access to partitioned data from all the classes but for the high heterogeneity setting the data is partitioned such that each WN can access data from only a subset (5 out of 10 classes) of classes. For CIFAR-10 (resp. MNIST), each WN has access to 490 (resp. 540) samples for training and 90 (resp. 80) samples for testing purposes.

For \aname, we set $w_t = 1$, $c = \bar{c}/\bar{\kappa}^2$ and tune for $\bar{\kappa}$ and $\bar{c}$ in the range $\bar{\kappa} \in [0.01, 0.5]$ and $\bar{c} \in [1, 10]$, respectively (cf. Theorem \ref{Thm: PR_Convergence_Main}). We note that for small batch sizes  $\bar{\kappa} \in [0.01, 0.1]$, whereas for larger batch sizes $\bar{\kappa} \in [0.3, 0.5]$ perform well. We diminish $\eta_t$ according to \eqref{eq: Step-Size} in each epoch\footnote{We define epoch as a single pass over the whole data.}. For SCAFFOLD and FedAvg, the stepsize choices of $0.1$ and $0.01$ perform well for large and smaller batch sizes, respectively. We use cross entropy as the loss function and evaluate the algorithm performance under a few settings discussed next.  


\begin{figure}[t!]
		\centering
		\begin{minipage}[t]{0.45\textwidth}
        \flushright
        \includegraphics[width=1\linewidth]{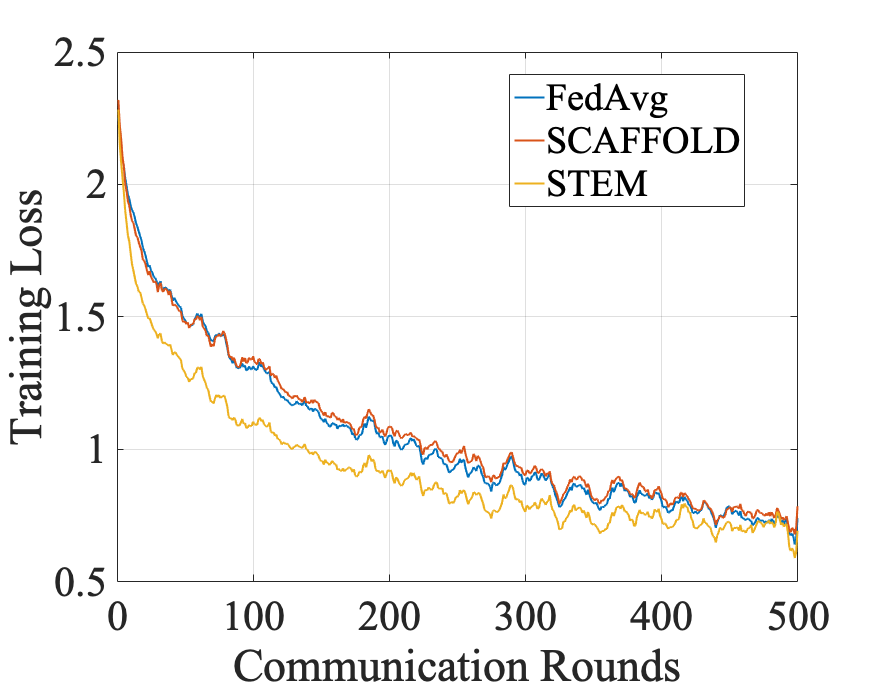}
        \captionsetup{labelformat=empty}
        \end{minipage}
        \begin{minipage}[t]{0.45\textwidth}
        \flushleft
        \includegraphics[width=1\linewidth]{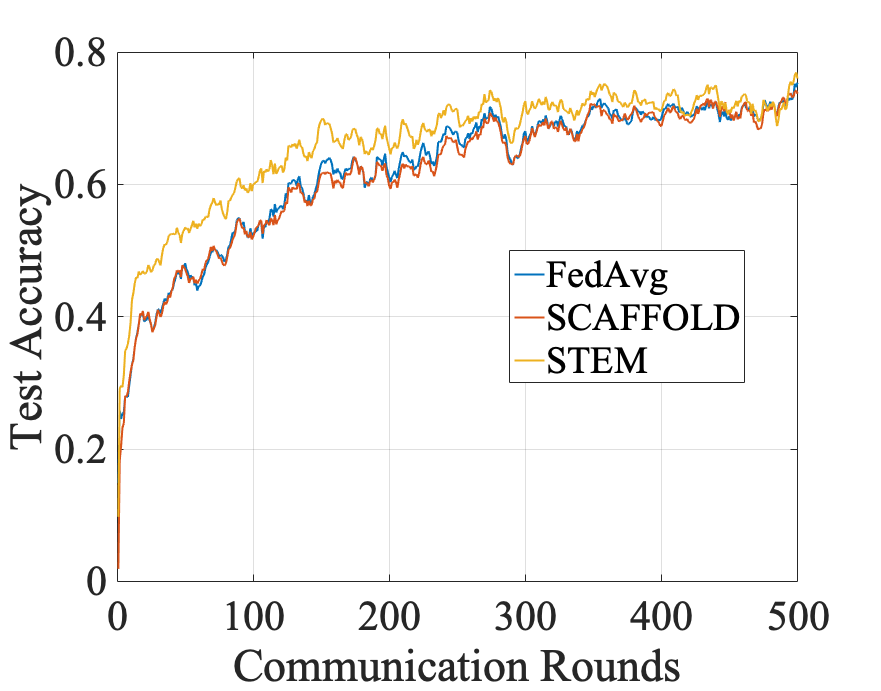}
        \captionsetup{labelformat=empty}
        \end{minipage}
        \caption{Training loss and testing accuracy for classification on CIFAR-10 dataset against the number of communication rounds for high heterogeneity setting with $b = 8$ and $I = 61$.}
        \label{Fig: CIFAR_b8e1_noniid}  
	\end{figure}

	\begin{figure}[t!]
		\centering
        \begin{minipage}[t]{0.45\textwidth}
        \flushright
        \includegraphics[width=1\linewidth]{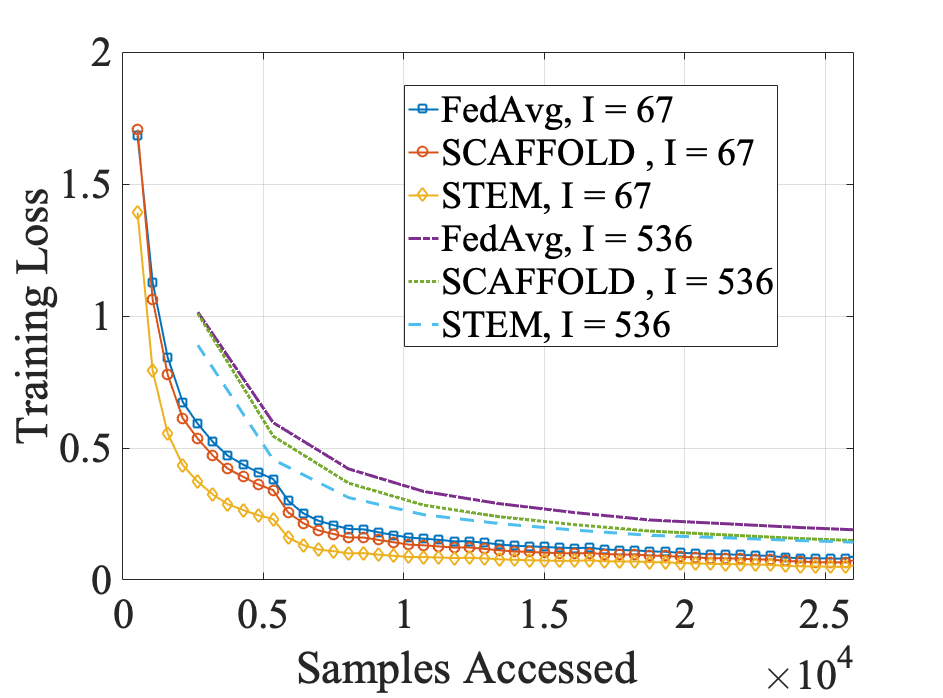}
        \captionsetup{labelformat=empty}
        \end{minipage}
        \hspace{0.5cm}
        \begin{minipage}[t]{0.45\textwidth}
        \flushleft
        \includegraphics[width=1\linewidth]{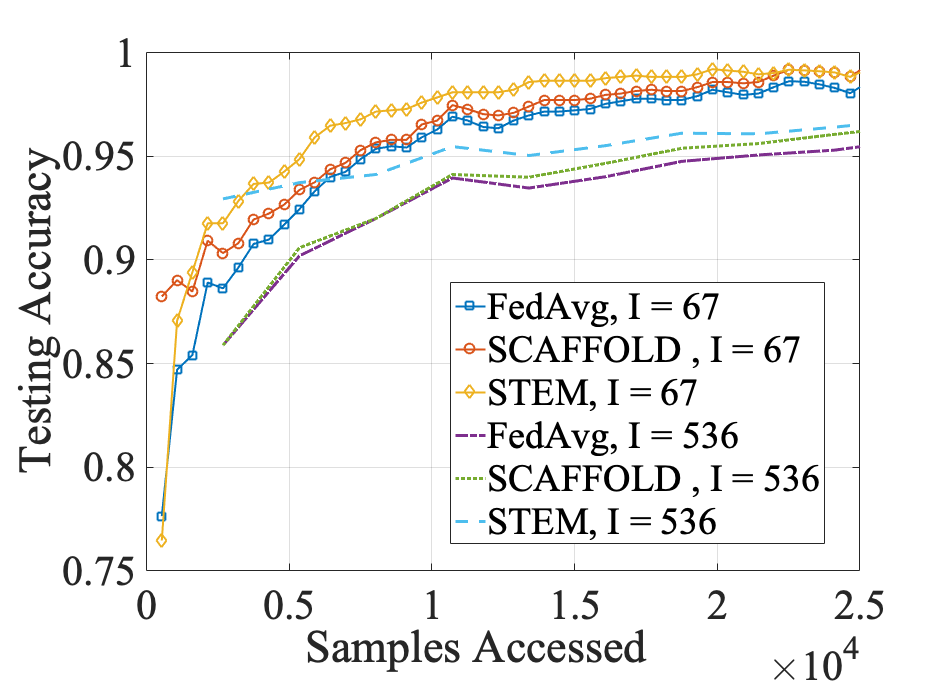}
        \captionsetup{labelformat=empty}  
        \end{minipage}
        \caption{Training loss and the testing accuracy for classification on MNIST data set against the number of samples accessed at each WN for high heterogeneity setting with $b = 8$. }
        \label{Fig: SampleComplexity_Comp}	
        \end{figure} 
\paragraph{Discussion:} In Figures \ref{Fig: CIFAR_b8e1} and \ref{Fig: CIFAR_b64e1}, we compare the training and testing performance of \aname~with FedAvg and SCAFFOLD for CIFAR-10 dataset under moderate heterogeneity setting. For Figure \ref{Fig: CIFAR_b8e1}, we choose $b = 8$ and $I = 61$, whereas for Figure \ref{Fig: CIFAR_b64e1}, we choose $b = 64$ and $I = 7$. We first note that for both cases \aname~performs better than FedAvg and SCAFFOLD. Moreover, observe that for both settings, small batches with multiple local updates (Figure \ref{Fig: CIFAR_b8e1}) and large batches with few local updates (Figure \ref{Fig: CIFAR_b64e1}), the algorithms converge with approximately similar performance, corroborating the theoretical analysis (see Discussion in Section \ref{sec: Intro}). Next, in Figure \ref{Fig: CIFAR_b8e1_noniid} we evaluate the performance of the proposed algorithms on CIFAR-10 with high heterogeneity setting for $b = 8$ and $I = 61$. We note that \aname~outperforms FedAvg  and SCAFFOLD in this setting as well. Finally, with the next set of experiments we emphasize the importance of choosing $b$ and $I$ carefully. In Figure \ref{Fig: SampleComplexity_Comp}, we compare the training and testing performance of the algorithms against the number of samples accessed at each WN for the classification task on MNIST dataset with high heterogeneity. We fix $b =8$ and conduct experiments under two settings, one with $I = 67$, and the other with $I = 536$ local updates at each WN. Note that although a large number of local updates might lead to fewer communication rounds but it can make the sample complexity extremely high as is demonstrated by Figure \ref{Fig: SampleComplexity_Comp}. For example, Figure \ref{Fig: SampleComplexity_Comp} shows that to reach testing accuracy of $96 - 97 \%$ with $I = 67$, \aname~requires approximately $5000-6000$ samples, in contrast with $I = 536$ it requires more than $25000$ samples at each WN. Similar behavior can be observed if we fix $I > 1$ and increase the local batch sizes. This implies not choosing the local updates and the batch sizes judiciously might lead to increased sample complexity. 

\section{Conclusion}
\label{sec: Conclusion}
  In this work, we proposed a novel algorithm \aname, for distributed stochastic non-convex optimization with applications to FL. We showed that \aname~reaches an $\epsilon$-stationary point with $\mathcal{\tilde{O}}(\epsilon^{-3/2})$ sample complexity while achieving linear speed-up with the number of WNs. Moreover, the algorithm achieves a communication complexity of $\mathcal{\tilde{O}}(\epsilon^{-1})$. We established a (optimal) trade-off that allows interpolation between varying choices of local updates and the batch sizes at each WN while maintaining (near optimal) sample and communication complexities. We showed that FedAvg (a.k.a LocalSGD) also exhibits a similar trade-off while achieving worse complexities. Our results provide guidelines to carefully choose the update frequency, directions, and minibatch sizes to achieve the best performance. The future directions of this work include developing lower bounds on communication complexity that establishes the tightness of the analysis conducted in this work. 
  
 
 \newpage
 
\newpage
\bibliographystyle{IEEEtran}
\bibliography{abrv,References}

\newpage
\appendix

\section*{Appendix}
The organization of the Appendix is given below.  In Appendix \ref{App: SAMVR}, we present the proof of the convergence guarantees for \aname~(Algorithm \ref{Algo_DR-STORM_batch}) stated in Section \ref{Sec: Convergence}. In Appendix \ref{App: FedAvg}, we present the proof of the convergence guarantees associated with FedAvg (Algorithm \ref{Algo_FedAvg}) stated in Section \ref{Sec: Special Case}. Finally, in Appendix \ref{App: Useful_Lemmas} we state some useful lemmas utilized throughout the proofs. 

\section{Proofs of Convergence Guarantees for \aname}
\label{App: SAMVR}
In this section we present the proofs for the convergence of \aname. First, we present some preliminary lemmas to be utilized throughout the proof. For reader's convenience here we restate the steps of the Algorithm \ref{Algo_DR-STORM_batch} in Algorithm \ref{Algo_DR-STORM_batch_App}.

 \subsection{Preliminary lemmas}
\begin{lem}
\label{Lem: InnerProduct_e_t_Grad}
Define $\bar{e}_{t} \coloneqq \bar{d}_{t} - \frac{1}{K} \sum_{k=1}^K \nabla f^{(k)}(x_{t}^{(k)})$, then the iterates generated according to Algorithm \ref{Algo_DR-STORM_batch_App} satisfy
\begin{align*}
   & \mathbb{E} \Bigg[   \bigg\langle   (1 - a_t) \bar{e}_{t-1}, \frac{1}{K}\sum_{k = 1}^K \frac{1}{b}  \sum_{\xi_t^{(k)} \in \mathcal{B}_t^{(k)}}   \bigg[\bigg( \nabla f^{(k)}(x^{(k)}_{t};\xi_t^{(k)})-  \nabla f^{(k)}(x^{(k)}_{t}) \bigg)  \\
   & \qquad \qquad \qquad \qquad  \qquad \qquad \qquad    \qquad      - (1 - a_t) \left( \nabla f^{(k)}(x^{(k)}_{t - 1}; \xi_t^{(k)})  - \nabla f^{(k)}(x^{(k)}_{t - 1})\right) \bigg]  \bigg\rangle   \Bigg]   = 0,
\end{align*}
where the expectation is w.r.t. the stochasticity of the algorithm.
 \end{lem}
 \begin{proof}
 Note that, given the filtration 
 $$\mathcal{F}_t = \sigma(x_1^{(k)}, x_2^{(k)}, \ldots, x_t^{(k)}~\text{for all}~k \in [K]),$$  
the gradient error term, $\bar{e}_{t-1}$, is fixed. The only randomness in the left hand side of the statement of the Lemma is with respect to $\xi_t^{(k)}$, for all $k \in [K]$. This implies that we have
\begin{align*}
& \mathbb{E} \Bigg[   \bigg\langle   (1 - a_t) \bar{e}_{t-1}, \frac{1}{K}\sum_{k = 1}^K \frac{1}{b}  \sum_{\xi_t^{(k)} \in \mathcal{B}_t^{(k)}}   \bigg[\bigg( \nabla f^{(k)}(x^{(k)}_{t};\xi_t^{(k)})-  \nabla f^{(k)}(x^{(k)}_{t}) \bigg)  \\
   & \qquad \qquad \qquad \qquad \qquad \qquad \qquad \qquad \qquad    \qquad      - (1 - a_t) \left( \nabla f^{(k)}(x^{(k)}_{t - 1}; \xi_t^{(k)})  - \nabla f^{(k)}(x^{(k)}_{t - 1})\right) \bigg]  \bigg\rangle   \Bigg] \\
   &  \qquad \qquad = \mathbb{E} \Bigg[ \bigg\langle  (1 - a_t) \bar{e}_{t-1},    \frac{1}{K}\sum_{k = 1}^K \mathbb{E} \bigg[ \frac{1}{b}\sum_{\xi_t^{(k)} \in \mathcal{B}_t^{(k)}}\bigg[\bigg( \nabla f^{(k)}(x^{(k)}_{t};\xi_t^{(k)})  -  \nabla f^{(k)}(x^{(k)}_{t}) \bigg)   \\
   & \qquad \qquad \qquad \qquad \qquad \qquad \qquad \qquad \qquad     \qquad  - (1 - a_t) \bigg( \nabla f^{(k)}(x^{(k)}_{t - 1}; \xi_t^{(k)})  - \nabla f^{(k)}(x^{(k)}_{t - 1})\bigg) \bigg] \bigg| \mathcal{F}_t \bigg]  \bigg\rangle \Bigg].
\end{align*}
The result then follows from the fact that $\xi_t^{(k)}$ is chosen uniformly randomly at each $k \in [K]$, and we have from (Assumption \ref{Ass: Unbiased_Var_Grad}) that: $\mathbb{E} \big[ \nabla f^{(k)}(x^{(k)}_{t}; \xi_t^{(k)}) \big] = \nabla f^{(k)}(x^{(k)}_{t})$. This implies we have
 $$\mathbb{E} \bigg[ \frac{1}{b}\sum_{\xi_t^{(k)} \in \mathcal{B}_t^{(k)}}\left[\left( \nabla f^{(k)}(x^{(k)}_{t};\xi_t^{(k)})  -  \nabla f^{(k)}(x^{(k)}_{t}) \right)    - (1 - a_t) \left( \nabla f^{(k)}(x^{(k)}_{t - 1}; \xi_t^{(k)})  - \nabla f^{(k)}(x^{(k)}_{t - 1})\right) \right] \bigg| \mathcal{F}_t \bigg] = 0$$
 for all $k \in [K]$. 
 
 Therefore, the lemma is proved. 
\end{proof}

\begin{lem}
\label{Lem: InnerProd_AcrossNodes}
For $k , \ell \in [K]$ with $k \neq \ell$, the iterates generated according to Algorithm \ref{Algo_DR-STORM_batch_App} satisfy
\begin{align*}
 & \mathbb{E} \Bigg[ \bigg\langle   \sum_{\xi_t^{(k)} \in \mathcal{B}_t^{(k)}} \Big[ \Big( \nabla f^{(k)}(x^{(k)}_{t};\xi_t^{(k)})  -  \nabla f^{(k)}(x^{(k)}_{t}) \Big)    - (1 - a_t) \Big( \nabla f^{(k)}(x^{(k)}_{t - 1}; \xi_t^{(k)}) - \nabla f^{(k)}(x^{(k)}_{t - 1})\Big) \Big],\\
   & \qquad \qquad    \sum_{\xi_t^{(\ell)} \in \mathcal{B}_t^{(\ell)}} \Big[ \Big( \nabla f^{(\ell)}(x^{(\ell)}_{t};\xi_t^{(\ell)})  -  \nabla f^{(\ell)}(x^{(\ell)}_{t}) \Big)    - (1 - a_t) \Big( \nabla f^{(\ell)}(x^{(\ell)}_{t - 1}; \xi_t^{(\ell)}) - \nabla f^{(\ell)}(x^{(\ell)}_{t - 1}) \Big) \Big] \bigg\rangle \Bigg] = 0
   \end{align*}
\end{lem}
\begin{proof}
Again note from the fact that conditioned on $\mathcal{F}_t$ the batches $\mathcal{B}_t^{(k)}$ and $\mathcal{B}_t^{(\ell)}$ for all $k,\ell \in [K]$ with $k \neq \ell$ across WNs are chosen independently of each other. Therefore, we have
\begin{align*}
& \mathbb{E} \Bigg[ \bigg\langle   \sum_{\xi_t^{(k)} \in \mathcal{B}_t^{(k)}} \Big[ \Big( \nabla f^{(k)}(x^{(k)}_{t};\xi_t^{(k)})  -  \nabla f^{(k)}(x^{(k)}_{t}) \Big)    - (1 - a_t) \Big( \nabla f^{(k)}(x^{(k)}_{t - 1}; \xi_t^{(k)}) - \nabla f^{(k)}(x^{(k)}_{t - 1})\Big) \Big],\\
   & \qquad  \qquad   \sum_{\xi_t^{(\ell)} \in \mathcal{B}_t^{(\ell)}} \Big[ \Big( \nabla f^{(\ell)}(x^{(\ell)}_{t};\xi_t^{(\ell)})  -  \nabla f^{(\ell)}(x^{(\ell)}_{t}) \Big)    - (1 - a_t) \Big( \nabla f^{(\ell)}(x^{(\ell)}_{t - 1}; \xi_t^{(\ell)}) - \nabla f^{(\ell)}(x^{(\ell)}_{t - 1}) \Big) \Big] \bigg\rangle \Bigg]\\
     & =  \mathbb{E} \Bigg[ \bigg\langle  \mathbb{E} \bigg[ \sum_{\xi_t^{(k)} \in \mathcal{B}_t^{(k)}} \Big[ \Big( \nabla f^{(k)}(x^{(k)}_{t};\xi_t^{(k)})  -  \nabla f^{(k)}(x^{(k)}_{t}) \Big)    - (1 - a_t) \Big( \nabla f^{(k)}(x^{(k)}_{t - 1}; \xi_t^{(k)}) - \nabla f^{(k)}(x^{(k)}_{t - 1})\Big) \Big] \bigg| \mathcal{F}_t \bigg],\\
   & \qquad      \mathbb{E} \bigg[ \sum_{\xi_t^{(\ell)} \in \mathcal{B}_t^{(\ell)}} \Big[ \Big( \nabla f^{(\ell)}(x^{(\ell)}_{t};\xi_t^{(\ell)})  -  \nabla f^{(\ell)}(x^{(\ell)}_{t}) \Big)    - (1 - a_t) \Big( \nabla f^{(\ell)}(x^{(\ell)}_{t - 1}; \xi_t^{(\ell)}) - \nabla f^{(\ell)}(x^{(\ell)}_{t - 1}) \Big) \Big] \bigg| \mathcal{F}_t\bigg] \bigg\rangle \Bigg].
\end{align*}
The result then follows from the fact that $\xi_t^{(k)}$ is chosen uniformly randomly across $k \in [K]$ and we have from the unbiased gradient Assumption \ref{Ass: Unbiased_Var_Grad} that: $\mathbb{E} \big[ \nabla f^{(k)}(x^{(k)}_{t}; \xi_t^{(k)}) \big] = \nabla f^{(k)}(x^{(k)}_{t})$. This implies we have
\begin{align*}
    \mathbb{E} \bigg[ \sum_{\xi_t^{(k)} \in \mathcal{B}_t^{(k)}} \Big[ \Big( \nabla f^{(k)}(x^{(k)}_{t};\xi_t^{(k)})  -  \nabla f^{(k)}(x^{(k)}_{t}) \Big)    - (1 - a_t) \Big( \nabla f^{(k)}(x^{(k)}_{t - 1}; \xi_t^{(k)}) - \nabla f^{(k)}(x^{(k)}_{t - 1})\Big) \Big] \bigg| \mathcal{F}_t \bigg] = 0
\end{align*}
for all $k \in [K]$.

Therefore, the lemma is proved. 
\end{proof}

\begin{lem}
\label{Lem: e_bar_bound_Batch}
For $\bar{e}_1 \coloneqq \bar{d}_1 - \frac{1}{K} \sum_{k=1}^K   \nabla f^{(k)}(x_1^{(k)})$ where $\bar{d}_1$ chosen according to Algorithm \ref{Algo_DR-STORM_batch_App}, we have:
\begin{align*}
    \mathbb{E}\| \bar{e}_1 \|^2 \leq \frac{\sigma^2}{KB}.
\end{align*}
\end{lem}
\begin{proof}

Using the definition of $\bar{e}_1$ we have:
\begin{align*}
\mathbb{E} \|    \bar{e}_1 \|^2  & = \mathbb{E} \bigg\| \bar{d}_1 - \frac{1}{K} \sum_{k=1}^K   \nabla f^{(k)}(x_1^{(k)}) \bigg\|^2 \\
 & \overset{(a)}{=} \mathbb{E} \bigg\| \frac{1}{K} \sum_{k=1}^K \frac{1}{B} \sum_{\xi_1^{(k)} \in \mathcal{B}_1^{(k)}} \nabla f^{(k)}(x_1^{(k)}; \xi_1^{(k)}) - \frac{1}{K} \sum_{k=1}^K \nabla f^{(k)}(x_1^{(k)})  \bigg\|^2 \\
 & \overset{(b)}{=}  \frac{1}{K^2B^2} \sum_{k=1}^K \mathbb{E} \bigg\| \sum_{\xi_1^{(k)} \in \mathcal{B}_1^{(k)}} \big( \nabla f^{(k)}(x_1^{(k)}; \xi_1^{(k)}) - \nabla f^{(k)}(x_1^{(k)}) \big) \bigg\|^2  \\
 & \overset{(c)}{=}  \frac{1}{K^2B^2} \sum_{k=1}^K \sum_{\xi_1^{(k)} \in \mathcal{B}_1^{(k)}} \mathbb{E} \big\|  \nabla f^{(k)}(x_1^{(k)}; \xi_1^{(k)}) - \nabla f^{(k)}(x_1^{(k)})  \big\|^2 \\
& \overset{(d)}{\leq} \frac{\sigma^2}{KB}.
\end{align*}
where $(a)$ follows from the definition of $\bar{d}_1$ in Algorithm \ref{Algo_DR-STORM_batch} and $(b)$ follows from the following: From Assumption \ref{Ass: Unbiased_Var_Grad}, given $\mathcal{F}_t$ we have: $\mathbb{E} \big[  \nabla f^{(k)}({x}_1^{(k)};\xi_1^{(k)}) \big]  =  \nabla f^{(k)}({x}_1^{(k)})$, for all $k \in [K]$. Moreover, given $\mathcal{F}_t$ the samples $\xi_1^{(k)}$ and $\xi_1^{(\ell)}$ at the $k^\text{th}$ and the $\ell^\text{th}$ WNs are chosen uniformly randomly, and independent of each other for all $k,\ell \in [K]$ and $k \neq \ell$.
\begin{align*}
  & \mathbb{E} \bigg[ \bigg\langle \sum_{\xi_1^{(k)} \in \mathcal{B}_1^{(k)}} \big( \nabla f^{(k)}({x}_1^{(k)}; \xi_1^{(k)}) - \nabla f^{(k)}(x_1^{(k)}) \big), \sum_{\xi_1^{(\ell)} \in \mathcal{B}_1^{(\ell)}} \big( \nabla f^{(\ell)}(x_1^{(\ell)}; \xi_1^{(\ell)}) - \nabla f^{(\ell)}(\bar{x}_1) \big) \bigg\rangle \Bigg] \\
  &    =  \mathbb{E} \bigg[\bigg\langle \sum_{\xi_1^{(k)} \in \mathcal{B}_1^{(k)}} \underbrace{\mathbb{E} \big[  \nabla f^{(k)}({x}_1^{(k)}; \xi_1^{(k)}) - \nabla f^{(k)}({x}_1^{(k)}) \Big| \mathcal{F}_t \big]}_{=0} , \sum_{\xi_1^{(\ell)} \in \mathcal{B}_1^{(\ell)}}\underbrace{\mathbb{E} \big[ \nabla f^{(\ell)}({x}_1^{(\ell)}; \xi_1^{(\ell)}) - \nabla f^{(\ell)}(x_1^{(\ell)})  \Big| \mathcal{F}_t \big]}_{=0}\bigg\rangle \bigg] \\
  &  = 0. 
\end{align*}
The equality $(c)$ follows from the fact that the samples $\xi_1^{(k)} \in \mathcal{B}_1^{(k)}$ for all $k \in [K]$ are chosen independently of each other. Then we conclude $(c)$ from an argument similar to that of $(b)$. Finally, $(d)$ results from the intra-node variance bound given in Assumption \ref{Ass: Unbiased_Var_Grad}.

Hence, the lemma is proved. 
\end{proof}

\begin{algorithm}[t]
\caption{The Stochastic Two-Sided Momentum (\aname) Algorithm}
\label{Algo_DR-STORM_batch_App}
\begin{algorithmic}[1]
	\State{\textbf{Input}: Parameters: 
	$c>0$,  the number of local updates $I$, batch size $b$, stepsizes $\{\eta_t\}$.}
	\State{\textbf{Initialize}:  Iterate $x_1^{(k)} = \bar{x}_1 = \frac{1}{K} \sum_{k = 1}^K x_1^{(k)}$, descent direction $d_1^{(k)} = \bar{d}_1 = \frac{1}{K} \sum_{k = 1}^K d_1^{(k)}$ with $d_1^{(k)} = \frac{1}{B} \sum_{\xi_1^{(k)} \in \mathcal{B}^{(k)}_1} \nabla f^{(k)}(x_1^{(k)} ; \xi_1^{(k)})$  and $|\mathcal{B}^{(k)}_1| = B$ for $k \in [K]$.} \\
	{Perform: $x^{(k)}_2 = x^{k}_1 - \eta_1  d^{(k)}_{1}, \; \forall~k\in[K]$}
	\For{$t = 1$ to $T$}
    	\For{$k = 1$ to $K$} \qquad\qquad\qquad\qquad\qquad\qquad \qquad \qquad ~~\quad \# at the WN
        	\State{$\!\!\! \displaystyle d_{t+1}^{(k)} = \frac{1}{b} \!\!\!\sum_{\xi_{t+1}^{(k)} \in \mathcal{B}_{t+1}^{(k)}} \!\!\!\!\!\! \nabla f^{(k)}(x_{t+1}^{(k)} ; \xi_{t+1}^{(k)}) +  (1 - a_{t+1})      \Big(  d_{t}^{(k)}   -   \frac{1}{b} \!\!\! \sum_{\xi_{t+1}^{(k)} \in \mathcal{B}_{t+1}^{(k)}}\!\!\! \!\!\!\nabla f^{(k)}(x_{t}^{(k)} ; \xi_{t+1}^{(k)}) \Big)$ \text{with}~$|\mathcal{B}_{t+1}^{(k)}| = b$, $a_{t+1}\! =\! c \eta_{t}^2$}
        	\State{{\bf if} $t~\text{mod}~I = 0$ {\bf then} ~~\quad\qquad\qquad\quad\quad\quad \qquad\quad\quad\quad\quad\quad\# at the SN}
        	        	\State{\quad $ d_{t + 1}^{(k)} = \bar{d}_{t + 1} \coloneqq \frac{1}{K} \sum_{k=1}^K d_{t + 1}^{(k)}$}
        			\State{\quad $ x_{t + 2}^{(k)}   \coloneqq  \bar{x}_{t + 1} - \eta_{t+1} \bar{d}_{t+1} =\frac{1}{K} \sum_{k=1}^K x_{t + 1}^{(k)} - \eta_{t+1} \bar{d}_{t+1}$~ \# server-side momentum step} 

        	\State{{\bf else} $ x_{t+2}^{(k)} =  x_{t+1}^{(k)} - \eta_{t+1} d_{t+1}^{(k)}$ \qquad\qquad\qquad~\quad\quad\quad\quad\quad   \# worker-side momentum step}
        	\State{\bf end if} 
        	
        	        	
	    \EndFor
	\EndFor
\State{{\bf Return:} $\bar{x}_a$ chosen uniformly randomly from $\{\bar{x}_t\}_{t=1}^T$}	
\end{algorithmic}
\end{algorithm}

Next, using the preliminary lemmas developed in this section we prove the main results of the work.

\subsection{Proof of Main Results: \aname}	
\label{App: Proof_MainResults}
In this section, we utilize the results developed in earlier sections to derive the main result of the paper presented in Section \ref{Sec: Convergence}. Throughout the section we assume Assumptions \ref{Ass: Lip_Smoothness} and \ref{Ass: Unbiased_Var_Grad} to hold. Before proceeding, we first define some notations. 

We define $\bar{t}_s \coloneqq sI + 1$ with $s \in [S]$. Note from Algorithm \ref{Algo_DR-STORM_batch_App} that at $(s \times I)^\text{th}$ iteration, i.e., when $t~ \text{mod}~ I = 0$, the descent directions, $\{d_t^{(k)}\}_{k = 1}^K$, corresponding to $t = (\bar{t}_s)^\text{th}$ time instant are shared with the SN. At the same time instant, the iterates, $\{x_t^{(k)}\}_{k = 1}^K$ are also shared and the SN performs the ``server side momentum step'' (cf. Step 9 of Algorithm \ref{Algo_DR-STORM_batch_App}).

\subsubsection{Proof of Descent Lemma}
In the first step, we bound the error accumulation via the iterates generated by Algorithm \ref{Algo_DR-STORM_batch_App}. 

\begin{lem}[Error Accumulation from Iterates]
\label{lem: ErrorAccumulation_Batch_App}
For each $t \in [\bar{t}_{s-1},  \bar{t}_s - 1]$ and $s \in [S]$, the iterates $x_t^{(k)}$ for each $k \in [K]$ generated from Algorithm \ref{Algo_DR-STORM_batch_App} satisfy:
\begin{align*}
\sum_{k = 1}^K \mathbb{E}\| x_t^{(k)}-  \bar{x}_t \|^2 \leq {(I - 1)} \sum_{\ell = \bar{t}_{s-1}}^{t} \eta_\ell^2 \sum_{k = 1}^K \mathbb{E}\|  d_\ell^{(k)} - \bar{d}_\ell  \|^2, 
\end{align*}
where the expectation is w.r.t the stochasticity of the algorithm.
\end{lem}
\begin{proof}
Note from Algorithm \ref{Algo_DR-STORM_batch_App} and the definition of $\bar{t}_s$ that at $t = \bar{t}_{s - 1}$ with $s \in [S]$, $x_{t}^{(k)} = \bar{x}_{t}$, for all $k$. 
This implies 
$$\sum_{k = 1}^K \| x_{\bar{t}_{s-1}}^{(k)} - \bar{x}_{\bar{t}_{s-1}} \|^2 = 0.$$ 
Therefore, the statement of the lemma holds trivially. 
Moreover, for $t \in [\bar{t}_{s-1} + 1,  \bar{t}_s - 1]$, with $s \in [S]$, we have from Algorithm \ref{Algo_DR-STORM_batch_App}: $x_{t}^{(k)} = x_{t-1}^{(k)} - \eta_{t-1} d_{t-1}^{(k)}$, this implies that:
\begin{align*}
    x_t^{(k)} = x_{\bar{t}_{s-1}}^{(k)} - \sum_{\ell = \bar{t}_{s-1}}^{t-1} \eta_\ell d_\ell^{(k)} \quad \text{and} \quad \bar{x}_{t}  = \bar{x}_{\bar{t}_{s-1}}  - \sum_{\ell = \bar{t}_{s-1}}^{t-1} \eta_\ell \bar{d}_\ell.
\end{align*}
This implies that for $t \in [\bar{t}_{s-1} + 1,  \bar{t}_s - 1]$, with $s \in [S]$ we have
\begin{align*}
  \sum_{k = 1}^K  \| x_t^{(k)}-  \bar{x}_t \|^2 & =  \sum_{k = 1}^K \Big\| x_{\bar{t}_{s-1}}^{(k)} - \bar{x}_{\bar{t}_{s-1}}  - \Big( \sum_{\ell = \bar{t}_{s-1}}^{t-1} \eta_\ell d_\ell^{(k)} -   \sum_{\ell =  \bar{t}_{s-1}}^{t-1} \eta_\ell \bar{d}_\ell  \Big) \Big\|^2 \\
  & \overset{(a)}{=} \sum_{k = 1}^K \Big\|  \sum_{\ell = \bar{t}_{s-1}}^{t-1} \big( \eta_\ell d_\ell^{(k)} -     \eta_\ell \bar{d}_\ell  \big) \Big\|^2  \\
  & \overset{(b)}{\leq} {(I - 1)} \sum_{\ell = \bar{t}_{s-1}}^{t-1} \eta_\ell^2 \sum_{k = 1}^K \|   d_\ell^{(k)} -       \bar{d}_\ell  \|^2\\
  & \leq {(I - 1)} \sum_{\ell = \bar{t}_{s-1}}^{t} \eta_\ell^2 \sum_{k = 1}^K \|   d_\ell^{(k)} -       \bar{d}_\ell  \|^2,
\end{align*}
where the equality $(a)$ follows from the fact that $x_{\bar{t}_{s-1}}^{(k)} = \bar{x}_{\bar{t}_{s-1}}$ and inequality $(b)$ uses the Lemma \ref{Lem: Norm_Ineq} {along with the fact that we have $d_t^{(k)} = \bar{d}_t$ for $t = \bar{t}_{s - 1}$}. 

Taking expectation on both sides yields the statement of the lemma.
\end{proof}

Next, we utilize Lemma \ref{lem: ErrorAccumulation_Batch_App} along with the smoothness of the function $f(\cdot)$ (Assumption \ref{Ass: Lip_Smoothness}) to  show descent in the objective function value at consecutive iterates. 
\begin{lem}[Descent Lemma]
\label{lem: Descent_Batch_App}
With $\bar{e}_t \coloneqq \bar{d}_t - \frac{1}{K} \sum_{k = 1}^K \nabla f^{(k)}(x_t^{(k)})$, for all $t \in [\bar{t}_{s-1}, \bar{t}_s - 1]$ and $s \in [S]$, the iterates generated by Algorithm \ref{Algo_DR-STORM_batch_App} satisfy:
\begin{align*}
    \mathbb{E} f(\bar{x}_{t + 1}) & \leq      \mathbb{E} f(\bar{x}_{t })   -  \left( \frac{\eta_t}{2} - \frac{\eta_t^2 L}{2} \right)     \mathbb{E} \| \bar{d}_{t}  \|^2 - \frac{\eta_t}{2}     \mathbb{E}\|\nabla f(\bar{x}_t) \|^2  + \eta_t      \mathbb{E} \|\bar{e}_t \|^2 \\
    &  \qquad \qquad \qquad \qquad \qquad    \qquad \qquad \qquad +   \frac{\eta_t L^2 {(I - 1)}}{K} \sum_{\ell = \bar{t}_{s-1}}^{t} \eta_\ell^2 \sum_{k=1}^K     \mathbb{E} \|d_\ell^{(k)} - \bar{d}_\ell\|^2,
\end{align*}
where the expectation is w.r.t the stochasticity of the algorithm.
\end{lem}
\begin{proof}
Using the smoothness of $f$ (Assumption \ref{Ass: Lip_Smoothness}) we have:
\begin{align}
    f(\bar{x}_{t + 1}) & \leq f(\bar{x}_{t }) + \langle \nabla f(\bar{x}_{t}),  \bar{x}_{t + 1} - \bar{x}_{t}\rangle + \frac{L}{2} \| \bar{x}_{t + 1} - \bar{x}_{t } \|^2 \nonumber\\
    &  \overset{(a)}{=} f(\bar{x}_{t}) - \eta_t \langle \nabla f(\bar{x}_{t}),  \bar{d}_t \rangle + \frac{\eta_t^2 L}{2} \| \bar{d}_{t}  \|^2 \nonumber\\
     & \overset{(b)}{=} f(\bar{x}_{t}) - \eta_t  \| \bar{d}_{t}  \|^2 + \eta_t \langle \bar{d}_t - \nabla f(\bar{x}_{t}),  \bar{d}_t \rangle + \frac{\eta_t^2 L}{2} \| \bar{d}_{t}  \|^2 \nonumber\\
       & \overset{(c)}{=}     f(\bar{x}_{t }) -  \left( \frac{\eta_t}{2} - \frac{\eta_t^2 L}{2} \right)  \| \bar{d}_{t}  \|^2 - \frac{\eta_t}{2} \|\nabla f(\bar{x}_t) \|^2 + \frac{\eta_t}{2}  \| \bar{d}_t - \nabla f(\bar{x}_{t}) \|^2  \nonumber \\  
        & \overset{(d)}{\leq}     f(\bar{x}_{t }) -  \left( \frac{\eta_t}{2} - \frac{\eta_t^2 L}{2} \right)  \| \bar{d}_{t}  \|^2 - \frac{\eta_t}{2} \|\nabla f(\bar{x}_t) \|^2  + \eta_t  \bigg\| \bar{d}_t - \frac{1}{K}\sum_{k=1}^K  \nabla f^{(k)}(x^{(k)}_{t}) \bigg\|^2  \nonumber \\
        & \qquad \qquad \qquad \quad \qquad    \qquad \qquad  \quad   + \eta_t  \bigg\| \frac{1}{K}\sum_{k=1}^K \big(\nabla f^{(k)}(x^{(k)}_{t}) - \nabla f^{(k)}(\bar{x}_t) \big) \bigg\|^2,    \label{eq: DR_Smoothness_1st_Batch}
\end{align}
where equality $(a)$ follows from the iterate update given in Step 10 of Algorithm \ref{Algo_DR-STORM_batch_App}, $(b)$ results by adding and subtracting $\bar{d}_t$ to $\nabla f(\bar{x}_t)$ in the inner product term and using the linearity of the inner product, $(c)$ follows from the relation $\langle x , y \rangle = \frac{1}{2} \|x\|^2 + \frac{1}{2} \|y\|^2 - \frac{1}{2} \|x - y\|^2$, finally inequality $(d)$ results from adding and subtracting $\frac{1}{K}\sum_{k=1}^K  \nabla f^{(k)}(x^{(k)}_{t})$ in the last term of $(c)$ and using Lemma \ref{Lem: Norm_Ineq}.

Taking expectation on both sides and considering the last term of \eqref{eq: DR_Smoothness_1st_Batch}, we have 
\begin{align}
      \mathbb{E}  \bigg\| \frac{1}{K}\sum_{k=1}^K \big(\nabla f^{(k)}(x^{(k)}_{t}) - \nabla f^{(k)}(\bar{x}_t) \big) \bigg\|^2 & \leq \frac{1}{K} \sum_{k=1}^K     \mathbb{E}\big\| \nabla f^{(k)}(x^{(k)}_{t}) - \nabla f^{(k)}(\bar{x}_t)  \big\|^2 \nonumber\\
    & \leq \frac{L^2}{K} \sum_{k=1}^K     \mathbb{E} \|x_t^{(k)} - \bar{x}_t\|^2,
    \label{eq: DR_Gradient_Ineq_Batch}
\end{align}
where the first inequality follows from Lemma \ref{Lem: Norm_Ineq}, and the second follows from the $L$-smoothness of $f^{(k)} (\cdot)$ (Assumption \ref{Ass: Lip_Smoothness}). 

Substituting \eqref{eq: DR_Gradient_Ineq_Batch} in \eqref{eq: DR_Smoothness_1st_Batch} and using the definition $\displaystyle \bar{e}_t \coloneqq  \bar{d}_t - \frac{1}{K} \sum_{k=1}^K \nabla f^{(k)}(x_t^{(k)}) $ we get:
\begin{align}
     \mathbb{E} f(\bar{x}_{t + 1})  & \leq      \mathbb{E}  f(\bar{x}_{t }) -  \left( \frac{\eta_t}{2} - \frac{\eta_t^2 L}{2} \right)     \mathbb{E} \| \bar{d}_{t}  \|^2 - \frac{\eta_t}{2}     \mathbb{E} \|\nabla f(\bar{x}_t) \|^2  + \eta_t      \mathbb{E} \|\bar{e}_t \|^2   +   \frac{\eta_t L^2}{K} \sum_{k=1}^K     \mathbb{E} \|x_t^{(k)} - \bar{x}_t\|^2.
  \label{eq: DR_Smoothness_2nd_batch}
\end{align}
Finally, using Lemma \ref{lem: ErrorAccumulation_Batch_App} to bound the last term of \eqref{eq: DR_Smoothness_2nd_batch}, we get:
\begin{align*}
         \mathbb{E} f(\bar{x}_{t + 1}) & \leq       \mathbb{E} f(\bar{x}_{t }) -  \left( \frac{\eta_t}{2} - \frac{\eta_t^2 L}{2} \right)      \mathbb{E} \| \bar{d}_{t}  \|^2 - \frac{\eta_t}{2}     \mathbb{E} \|\nabla f(\bar{x}_t) \|^2  + \eta_t     \mathbb{E} \|\bar{e}_t \|^2  
      +   \frac{\eta_t L^2{(I - 1)}}{K} \sum_{\ell = \bar{t}_{s-1}}^{t} \eta_\ell^2 \sum_{k=1}^K     \mathbb{E} \|d_\ell^{(k)} - \bar{d}_\ell\|^2.
\end{align*}
Hence, the lemma is proved.
\end{proof}

Lemma \ref{lem: Descent_Batch_App} shows that the expected descent in the function $f$ depends on the magnitude of the expected gradient error term $\bar{e}_t$, and the expected gradient drift across WNs, i.e., $\mathbb{E}\|d_\ell^{(k)} - \bar{d}_\ell \|^2$. This implies that to ensure sufficient descent we need to control the gradient error and the gradient drift across WNs. We achieve this by carefully designing the number of local updates, $I$, at each WN, and the batch-sizes $b$ (and initial batch size $B$), that each WN uses to compute the descent direction. 

Next, we present the error contraction lemma which analyzes how the term $\mathbb{E} \| \bar{e}_t \|^2$ contracts across time. 
\subsubsection{Proof of Gradient Error Contraction}

\begin{lem}[{Gradient Error Contraction}]
\label{lem: ErrorContraction_BatchGradients_App}
Define $\bar{e}_t \coloneqq \bar{d}_t - \frac{1}{K} \sum_{k = 1}^K \nabla f^{(k)}(x_t^{(k)})$, then for every $t \in [T]$ the iterates generated by Algorithm \ref{Algo_DR-STORM_batch_App} satisfy
\begin{align*}
    \mathbb{E} \| \bar{e}_{t+1} \|^2  & \leq (1 - a_{t+1})^2 \mathbb{E}\| \bar{e}_t\|^2 +  {\frac{8 (1 - a_{t+1})^2 L^2 }{bK^2}  \frac{(I - 1)}{I}  \eta_{t}^2}  \sum_{k=1}^K \mathbb{E} \big\| d_t^{(k)} - \bar{d}_t \big\|^2  \\
    & \qquad \qquad \qquad \qquad \qquad \qquad \qquad \qquad  \qquad +   \frac{4 (1 - a_{t+1})^2 L^2 \eta_t^2}{bK}\mathbb{E}\|  \bar{d}_t \|^2 + \frac{2a_{t+1}^2 \sigma^2}{bK},
\end{align*}
where the expectation is w.r.t the stochasticity of the algorithm.
\end{lem}
\begin{proof}
Consider the error term $\|\bar{e}_t\|^2$ as
\begin{align}
   & \mathbb{E} \| \bar{e}_t \|^2 = \mathbb{E} \bigg\| \bar{d}_t - \frac{1}{K} \sum_{k = 1}^K \nabla f^{(k)}(x^{(k)}_{t})  \bigg\|^2 \nonumber \\
    &   \overset{(a)}{=} \mathbb{E} \bigg\| \frac{1}{K} \sum_{k = 1}^K  \frac{1}{b} \!\! \sum_{\xi_t^{(k)} \in \mathcal{B}_t^{(k)}} \!\! \nabla f^{(k)}(x^{(k)}_{t};\xi_t^{(k)})     + (1 - a_t)\bigg( \bar{d}_{t-1} - \frac{1}{K} \sum_{k = 1}^K  \frac{1}{b} \sum_{\xi_t^{(k)} \in \mathcal{B}_t^{(k)}} \nabla f^{(k)}(x^{(k)}_{t - 1}; \xi_t^{(k)})\bigg) \nonumber\\
    & \qquad \qquad \qquad \qquad \qquad \qquad \qquad \qquad  \qquad \qquad \qquad \qquad \quad \qquad \qquad  - \frac{1}{K} \sum_{k = 1}^K \nabla f^{(k)}(x^{(k)}_{t})  \bigg\|^2 \nonumber \\
    & \overset{(b)}{=} \mathbb{E} \bigg\| \frac{1}{K} \sum_{k = 1}^K  \frac{1}{b} \sum_{\xi_t^{(k)} \in \mathcal{B}_t^{(k)}} \bigg[\left( \nabla f^{(k)}(x^{(k)}_{t};\xi_t^{(k)})  -  \nabla f^{(k)}(x^{(k)}_{t}) \right) \nonumber \\
    & \qquad \qquad \qquad \qquad \qquad \qquad   - (1 - a_t) \Big( \nabla f^{(k)}(x^{(k)}_{t - 1}; \xi_t^{(k)}) - \nabla f^{(k)}(x^{(k)}_{t - 1})\Big) \bigg] + (1 - a_t) \bar{e}_{t-1}   \bigg\|^2, \nonumber
    \end{align}
    where $(a)$ follows from the definition of descent direction given in Step 6 of Algorithm \ref{Algo_DR-STORM_batch_App}; $(b)$ follows by adding and subtracting $(1 - a_t) \frac{1}{K} \sum_{k=1}^K \nabla f^{(k)}(x_{t-1}^{(k)})$ and using the definition of $\bar{e}_{t-1}$. Further simplifying the above expression, we get
    \begin{align}
     \mathbb{E} \| \bar{e}_t \|^2   & \overset{(c)}{=} (1 - a_t)^2 \mathbb{E}\| \bar{e}_{t-1}\|^2    + \frac{1}{b^2 K^2 }\mathbb{E}\bigg\|\sum_{k = 1}^K \sum_{\xi_t^{(k)} \in \mathcal{B}_t^{(k)}}\Big[\left( \nabla f^{(k)}(x^{(k)}_{t};\xi_t^{(k)})  -  \nabla f^{(k)}(x^{(k)}_{t}) \right)  \nonumber\\
   & \qquad \qquad \qquad \qquad \qquad \qquad \qquad \qquad   - (1 - a_t) \left( \nabla f^{(k)}(x^{(k)}_{t - 1}; \xi_t^{(k)}) - \nabla f^{(k)}(x^{(k)}_{t - 1})\right) \Big] \bigg\|^2 \nonumber \\
    & \overset{(d)}{=} (1 - a_t)^2 \mathbb{E}\| \bar{e}_{t-1}\|^2    + \frac{1}{b^2 K^2 } \sum_{k = 1}^K  \mathbb{E} \bigg\|  \sum_{\xi_t^{(k)} \in \mathcal{B}_t^{(k)}} \Big[ \Big( \nabla f^{(k)}(x^{(k)}_{t};\xi_t^{(k)})  -  \nabla f^{(k)}(x^{(k)}_{t}) \Big)  \nonumber\\
    & \qquad \qquad \qquad \qquad \qquad \qquad \qquad \qquad   - (1 - a_t) \Big( \nabla f^{(k)}(x^{(k)}_{t - 1}; \xi_t^{(k)}) - \nabla f^{(k)}(x^{(k)}_{t - 1})\Big) \Big]  \bigg\|^2, \nonumber \\
    & \overset{(e)}{=} (1 - a_t)^2 \mathbb{E}\| \bar{e}_{t-1}\|^2    + \frac{1}{b^2 K^2 } \sum_{k = 1}^K \sum_{\xi_t^{(k)} \in \mathcal{B}_t^{(k)}} \mathbb{E}\Big\| \Big( \nabla f^{(k)}(x^{(k)}_{t};\xi_t^{(k)})  -  \nabla f^{(k)}(x^{(k)}_{t}) \Big)  \nonumber\\
    &\qquad \qquad \qquad \qquad\qquad \qquad \qquad \qquad   - (1 - a_t) \Big( \nabla f^{(k)}(x^{(k)}_{t - 1}; \xi_t^{(k)}) - \nabla f^{(k)}(x^{(k)}_{t - 1})\Big)  \Big\|^2, \label{eq_bound_grad_error_norm}
   \end{align}
   where $(c)$ results from expanding the norm using inner product and noting that the cross terms are zero in expectation from Lemma \ref{Lem: InnerProduct_e_t_Grad}; $(d)$ follows from expanding the norm using the inner products across $k \in [K]$ and noting that the cross term is zero in expectation from Lemma \ref{Lem: InnerProd_AcrossNodes}; finally, $(e)$ results from expanding the norm using the inner product across samples used to compute the minibatch gradients and the inner product is zero since at each node $k \in [K]$, the samples in the minibatch, $\xi_t^{(k)} \in \mathcal{B}_t^{(k)}$, are sampled independently of each other. 

Now considering the 2nd term of \eqref{eq_bound_grad_error_norm} above, we have
\begin{align}
  &    \mathbb{E}\big\| \big( \nabla f^{(k)}(x^{(k)}_{t};\xi_t^{(k)})  -  \nabla f^{(k)}(x^{(k)}_{t}) \big)    - (1 - a_t) \big( \nabla f^{(k)}(x^{(k)}_{t - 1}; \xi_t^{(k)}) - \nabla f^{(k)}(x^{(k)}_{t - 1})\big)  \big\|^2 \nonumber \\
    &   \qquad    =   \mathbb{E}\big\| (1 -a_t) \big[ \big( \nabla f^{(k)}(x^{(k)}_{t};\xi_t^{(k)})  -  \nabla f^{(k)}(x^{(k)}_{t}) \big)    - \big( \nabla f^{(k)}(x^{(k)}_{t - 1}; \xi_t^{(k)}) - \nabla f^{(k)}(x^{(k)}_{t - 1})\big) \big] \nonumber \\
    & \qquad \qquad \qquad \qquad \qquad \qquad  \qquad \qquad \qquad \qquad \qquad \qquad  + a_t   \big( \nabla f^{(k)}(x^{(k)}_{t}; \xi_t^{(k)}) - \nabla f^{(k)}(x^{(k)}_{t})\big) \big\|^2 \nonumber \\
    &  \qquad   \overset{(a)}{\leq} 2 (1 - a_t)^2  \mathbb{E} \big\| \big( \nabla f^{(k)}(x^{(k)}_{t};\xi_t^{(k)})  -  \nabla f^{(k)}(x^{(k)}_{t - 1}; \xi_t^{(k)}) \big)    -  \big(\nabla f^{(k)}(x^{(k)}_{t}) - \nabla f^{(k)}(x^{(k)}_{t - 1}) \big) \big\|^2  \nonumber \\
    & \qquad \qquad \qquad \qquad \qquad \qquad \quad  \qquad \qquad \qquad \qquad  \qquad   + 2 a_t^2  \mathbb{E}\big\|    \nabla f^{(k)}(x^{(k)}_{t}; \xi_t^{(k)}) - \nabla f^{(k)}(x^{(k)}_{t})    \big\|^2 \nonumber \\
    &   \qquad   \overset{(b)}{\leq}  2 (1 - a_t)^2   \mathbb{E} \big\|   \nabla f^{(k)}(x^{(k)}_{t};\xi_t^{(k)})  -  \nabla f^{(k)}(x^{(k)}_{t - 1}; \xi_t^{(k)}) \big\|^2 +  2a_t^2 \sigma^2 \nonumber \\
    & \qquad   \overset{(c)}{\leq}  2 (1 - a_t)^2 L^2    \mathbb{E}\| x_t^{(k)} - x_{t-1}^{(k)} \|^2 +  2a_t^2 \sigma^2 \nonumber \\
    &  \qquad   \overset{(d)}{\leq}  2 (1 - a_t)^2 L^2 \eta_{t-1}^2   \mathbb{E}\| d_{t-1}^{(k)} \|^2 +  2a_t^2 \sigma^2 \nonumber \\
      &  \qquad \overset{(e)}{\leq} { 8 (1 - a_t)^2 L^2 \frac{(I - 1)}{I} \eta_{t-1}^2   } \mathbb{E}\| d_{t-1}^{(k)} - \bar{d}_{t-1} \|^2 +    4 (1 - a_t)^2 L^2 \eta_{t-1}^2    \mathbb{E}\|  \bar{d}_{t-1} \|^2 +  2a_t^2 \sigma^2, \label{eq_bound_grad_error_norm_2}
\end{align}
where $(a)$ follows from Lemma \ref{Lem: Norm_Ineq}; $(b)$ results from use of Assumption \ref{Ass: Unbiased_Var_Grad} and mean variance inequality: For a random variable $Z$ we have $\mathbb{E} \|Z - \mathbb{E}[Z]\|^2 \leq \mathbb{E}\|Z\|^2$; $(c)$ follows from the Lipschitz continuity of the gradient given in Assumption \ref{Ass: Lip_Smoothness}; $(d)$ results from the iterate update equation given in Step 10 of Algorithm \ref{Algo_DR-STORM_batch_App}; {finally, $(e)$ uses the fact that: $(i)$ for $I = 1$ we have $d_t^{(k)} = \bar{d}_t$ for all $t \in [T]$ and $(ii)$ for $I \geq 2$ we use Lemma \ref{Lem: Norm_Ineq} and the fact that $(I - 1)/I \geq 1/2$.}

Substituting \eqref{eq_bound_grad_error_norm_2} in \eqref{eq_bound_grad_error_norm} we get:
\begin{align*}
    \mathbb{E} \| \bar{e}_t \|^2  & \leq (1 - a_t)^2 \mathbb{E}\| \bar{e}_{t-1}\|^2 +  {\frac{8 (1 - a_t)^2 L^2 }{b K^2} \frac{(I - 1)}{I} \eta_{t-1}^2} \sum_{k=1}^K \mathbb{E}\| d_{t-1}^{(k)} - \bar{d}_{t-1} \|^2 \\
    & \qquad \qquad \qquad \qquad \qquad \qquad \qquad \qquad \qquad \qquad  +   \frac{4 (1 - a_t)^2 L^2 \eta_{t-1}^2}{b K}\mathbb{E}\|  \bar{d}_{t-1} \|^2 + \frac{2a_t^2 \sigma^2}{b K}.
\end{align*}
Finally, the lemma is proved by replacing $t$ by $t + 1$. 
\end{proof}
Lemma \ref{lem: ErrorContraction_BatchGradients_App} shows that the gradient error contracts in each iteration. Next, we first define a potential function and then utilize Lemmas \ref{lem: Descent_Batch_App} and \ref{lem: ErrorContraction_BatchGradients_App} to show descent in the potential function.

\subsubsection{Descent in Potential Function}
We define the potential function as a linear combination of the objective function and the gradient estimation error: $\bar{e}_t \coloneqq \bar{d}_t - \frac{1}{K} \sum_{k = 1}^K \nabla f^{(k)}(x_t^{(k)})$ 
\begin{align}
\label{eq: Potential_Fn_Batch_App}
    \Phi_t \coloneqq f(\bar{x}_t) + \frac{b K}{64 L^2} \frac{\|\bar{e}_t \|^2}{\eta_{t-1}}.
\end{align}
Next, we characterize the descent in the potential function.
\begin{lem}[Potential Function Descent]
\label{Lem: Potential_Descent_App}
For $\bar{t} \in [\bar{t}_{s - 1} , \bar{t}_s - 1]$ and for {$\eta_t \leq \frac{1}{16LI}$} we have 
\begin{align*}
    \mathbb{E}[ \Phi_{\bar{t} + 1} - \Phi_{\bar{t}_{s-1}}] & \leq    - \sum_{t = \bar{t}_{s - 1}}^{\bar{t}} \left( \frac{7 \eta_t}{16} - \frac{\eta_t^2 L}{2} \right)  \mathbb{E} \| \bar{d}_{t}  \|^2   - \sum_{t = \bar{t}_{s - 1}}^{\bar{t}} \frac{\eta_t}{2} \mathbb{E} \|\nabla f(\bar{x}_t) \|^2 +  \frac{\sigma^2 c^2 }{32L^2 } \sum_{t = \bar{t}_{s - 1}}^{\bar{t}} \eta_{t}^3 \\
 &\qquad \qquad \qquad \qquad \qquad \qquad \qquad   + {\frac{33}{256 K} \frac{(I - 1)}{I}} \sum_{t = \bar{t}_{s - 1}}^{\bar{t}}  \eta_t  \sum_{k=1}^K \mathbb{E}\| d_{t}^{(k)} - \bar{d}_{t} \|^2
\end{align*}
where the expectation is w.r.t the stochasticity of the algorithm.
\end{lem}
\begin{proof}
To get the descent on the the potential function, we first consider the term: $\displaystyle  \frac{\mathbb{E} \|\bar{e}_{t+1} \|^2}{\eta_t} - \frac{\mathbb{E}\|\bar{e}_t \|^2}{\eta_{t-1}}$. \vspace{0.1 in}\\
Using Lemma \ref{lem: ErrorContraction_BatchGradients_App} we get
\begin{align}
    \frac{\mathbb{E} \|\bar{e}_{t+1} \|^2}{\eta_t} - \frac{\mathbb{E}\|\bar{e}_t \|^2}{\eta_{t-1}} & \leq \bigg[ \frac{(1 - a_{t+1})^2}{\eta_t} - \frac{1}{\eta_{t-1}} \bigg] \mathbb{E}\| \bar{e}_t\|^2 +  {\frac{8 (1 - a_{t+1})^2 L^2 }{b K^2} \frac{(I - 1)}{I} \eta_t} \sum_{k=1}^K \mathbb{E}\| d_{t}^{(k)} - \bar{d}_{t} \|^2 \nonumber\\
    & \qquad \qquad \qquad \qquad   \qquad \qquad \qquad \qquad  +   \frac{4 (1 - a_{t+1})^2 L^2 \eta_t}{bK}\mathbb{E}\|  \bar{d}_t \|^2 + \frac{2a_{t+1}^2 \sigma^2}{ \eta_t bK} \nonumber\\
    & \overset{(a)}{\leq}  \big( \eta_t^{-1} - \eta_{t-1}^{-1}  - c \eta_t \big)  \mathbb{E}\| \bar{e}_t\|^2 +  {\frac{8   L^2 }{bK^2} \frac{(I - 1)}{I} \eta_t} \sum_{k=1}^K \mathbb{E}\| d_{t}^{(k)} - \bar{d}_{t} \|^2 \nonumber\\
   & \qquad \qquad \qquad \qquad \qquad  \qquad \qquad \qquad \qquad +   \frac{4   L^2 \eta_t}{bK}\mathbb{E}\|  \bar{d}_t \|^2 + \frac{2\sigma^2 c^2 \eta_{t}^3 }{bK},
    \label{Eq: Error_Descent_1}
\end{align}
where inequality $(a)$ utilizes the fact that $(1 - a_t)^2 \leq 1 - a_t \leq 1$ for all $t \in [T]$.

Let us consider $\eta_t^{-1} - \eta_{t-1}^{-1}$ in the first term of the inequality in \eqref{Eq: Error_Descent_1} and using the definition of the stepsize $\eta_t$ from Theorem \ref{Thm: PR_Convergence_Main}, we have
\begin{align}
    \eta_t^{-1} - \eta_{t-1}^{-1} & =  \frac{(w_t + \sigma^2 t)^{1/3}}{\bar{\kappa}} -  \frac{(w_{t-1} + \sigma^2 (t-1))^{1/3}}{\bar{\kappa}} \nonumber \\
    & \overset{(a)}{\leq}  \frac{(w_t + \sigma^2 t)^{1/3}}{\bar{\kappa}} -   \frac{(w_{t} + \sigma^2 (t-1))^{1/3}}{\bar{\kappa}} \nonumber \\
      & \overset{(b)}{\leq}  \frac{\sigma^2}{3 \bar{\kappa} (w_{t} + \sigma^2 (t-1))^{2/3}} \nonumber \\
      & \overset{(c)}{\leq} \frac{2^{2/3} \sigma^2 \bar{\kappa}^2}{3 \bar{\kappa}^3 (w_t + \sigma^2 t)^{2/3}} \nonumber \\
      & \overset{(d)}{=} \frac{2^{2/3} \sigma^2}{3 \bar{\kappa}^3 } \eta_t^2 \nonumber \\
      &{\overset{(e)}{\leq} \frac{ \sigma^2 }{24 \bar{\kappa}^3 LI} \eta_t,}
      \label{Eq: Step_Difference}
\end{align}
where inequality $(a)$ follows from the fact that we choose $w_t \leq w_{t-1}$ (see definition of $w_t$ in Theorem \ref{Thm: PR_Convergence_Main}), $(b)$ results from the concavity of $x^{1/3}$ as:
$$(x + y)^{1/3} - x^{1/3} \leq \frac{y}{3x^{2/3}}.$$
In inequality $(c)$, we have used the fact that $w_t \geq 2\sigma^2$, finally, $(d)$ and $(e)$ utilize the definition of $\eta_t$ and the fact that $\eta_t \leq \frac{1}{16LI}$ for all $t \in [T]$, respectively.

Now combining the first term of inequality in \eqref{Eq: Error_Descent_1} with \eqref{Eq: Step_Difference} and choosing {$\displaystyle c = \frac{64L^2}{b K} + \frac{ \sigma^2 }{24 \bar{\kappa}^3 LI}$} we get:
\begin{align*}
    \eta_t^{-1} - \eta_{t-1}^{-1}  - c \eta_t  \leq  - \frac{64L^2}{b K} \eta_t .
\end{align*}
Therefore, we have from \eqref{Eq: Error_Descent_1}:
\begin{align*}
   \frac{\mathbb{E}  \|\bar{e}_{t+1} \|^2}{\eta_t} - \frac{\mathbb{E} \|\bar{e}_t \|^2}{\eta_{t-1}} & \leq - \frac{64L^2 \eta_t}{b K}  \mathbb{E}\| \bar{e}_t\|^2 +  {\frac{8   L^2 }{bK^2} \frac{(I - 1)}{I} \eta_t} \sum_{k=1}^K \mathbb{E}\| d_{t}^{(k)} - \bar{d}_{t} \|^2  \\
   & \qquad \qquad \qquad \qquad\qquad \qquad \qquad \qquad  +   \frac{4   L^2 \eta_t}{b K}\mathbb{E}\|  \bar{d}_t \|^2 + \frac{2\sigma^2 c^2 \eta_{t}^3 }{b K}\\
    \frac{b K}{64L^2} \bigg(  \frac{\mathbb{E}  \|\bar{e}_{t+1} \|^2}{\eta_t} - \frac{\mathbb{E} \|\bar{e}_t \|^2}{\eta_{t-1}} \bigg) & \leq -   \eta_t  \mathbb{E}\| \bar{e}_t\|^2 +  {\frac{1}{8 K} \frac{(I - 1)}{I} \eta_t} \sum_{k=1}^K \mathbb{E}\| d_{t}^{(k)} - \bar{d}_{t} \|^2   +   \frac{\eta_t}{16}\mathbb{E}\|  \bar{d}_t \|^2 + \frac{\sigma^2 c^2 \eta_{t}^3 }{32L^2}.
\end{align*}
Finally, using Lemma \ref{lem: Descent_Batch_App} and the definition of potential function given in \eqref{eq: Potential_Fn_Batch_App}, using the above we get the descent in the potential function for any $t \in [\bar{t}_{s-1}, \bar{t}_s - 1]$ with $s \in [S]$ as:
\begin{align*}
   \mathbb{E}[  \Phi_{t+1}  -   \Phi_{t}] & \leq  -  \left( \frac{7 \eta_t}{16} - \frac{\eta_t^2 L}{2} \right)  \mathbb{E} \| \bar{d}_{t}  \|^2 - \frac{\eta_t}{2} \mathbb{E} \|\nabla f(\bar{x}_t) \|^2    +   {\frac{\eta_t L^2 (I - 1)}{K} }\sum_{\ell = \bar{t}_{s-1}}^{t} \eta_\ell^2 \sum_{k=1}^K \mathbb{E} \|d_\ell^{(k)} - \bar{d}_\ell\|^2 \\
     & \qquad \qquad \qquad \qquad \qquad \qquad   \qquad \qquad  +  {\frac{1}{8 K} \frac{(I - 1)}{I} \eta_t} \sum_{k=1}^K \mathbb{E}\| d_{t}^{(k)} - \bar{d}_{t} \|^2   +   \frac{\sigma^2 c^2 \eta_{t}^3 }{32L^2}.
\end{align*}
Summing the above over $t= \bar{t}_{s-1}$ to $\bar{t}$ for $\bar{t} \in [\bar{t}_{s-1}, \bar{t}_s - 1]$, we get:
\begin{align*}
    \mathbb{E}[ \Phi_{\bar{t} + 1} - \Phi_{\bar{t}_{s-1}}] & \leq  - \sum_{t = \bar{t}_{s-1}}^{\bar{t}} \left( \frac{7 \eta_t}{16} - \frac{\eta_t^2 L}{2} \right)  \mathbb{E}\| \bar{d}_{t}  \|^2   - \sum_{t = \bar{t}_{s-1}}^{\bar{t}} \frac{\eta_t}{2} \mathbb{E}\|\nabla f(\bar{x}_t) \|^2 +  \frac{\sigma^2 c^2   }{32L^2 } \sum_{t = \bar{t}_{s-1}}^{\bar{t}}   \eta_{t}^3 \\
    & \!\!\!\!\!\!  +  {\frac{ L^2 (I - 1)}{K}}\sum_{t = \bar{t}_{s-1}}^{\bar{t}} \eta_t \sum_{\ell = \bar{t}_{s-1}}^{t} \eta_\ell^2 \sum_{k=1}^K \mathbb{E}\|d_\ell^{(k)} - \bar{d}_\ell\|^2 + {\frac{1}{8 K} \frac{(I - 1)}{I}} \sum_{t = \bar{t}_{s-1}}^{\bar{t}}  \eta_t  \sum_{k=1}^K \mathbb{E}\| d_{t}^{(k)} - \bar{d}_{t} \|^2  \\
     & \leq  - \sum_{t = \bar{t}_{s-1}}^{\bar{t}} \left( \frac{7 \eta_t}{16} - \frac{\eta_t^2 L}{2} \right)  \mathbb{E}\| \bar{d}_{t}  \|^2   - \sum_{t = \bar{t}_{s-1}}^{\bar{t}} \frac{\eta_t}{2} \mathbb{E}\|\nabla f(\bar{x}_t) \|^2 + \frac{\sigma^2 c^2   }{32L^2 } \sum_{t = \bar{t}_{s-1}}^{\bar{t}}   \eta_{t}^3  \\
    &  \qquad \qquad\qquad  +  {\frac{ L^2(I - 1)}{K}} \bigg( \sum_{t = \bar{t}_{s-1}}^{\bar{t}} \eta_t \bigg) \bigg( \sum_{\ell = \bar{t}_{s-1}}^{\bar{t}} \eta_\ell^2 \sum_{k=1}^K \mathbb{E}\|d_\ell^{(k)} - \bar{d}_\ell\|^2 \bigg) \\
    & \qquad \qquad \qquad\qquad \qquad\qquad \qquad\qquad  + {\frac{1}{8 K} \frac{(I - 1)}{I}} \sum_{t = \bar{t}_{s-1}}^{\bar{t}}  \eta_t  \sum_{k=1}^K \mathbb{E}\| d_{t}^{(k)} - \bar{d}_{t} \|^2.
\end{align*}
Finally, using the fact that we have: $\eta_t \leq \frac{1}{16LI}$ for all $t \in [T]$, we get:
\begin{align*}
    \mathbb{E}[ \Phi_{\bar{t} + 1} - \Phi_{\bar{t}_{s-1}}] & \leq  - \sum_{t = \bar{t}_{s-1}}^{\bar{t}} \left( \frac{7 \eta_t}{16} - \frac{\eta_t^2 L}{2} \right)  \mathbb{E}\| \bar{d}_{t}  \|^2   - \sum_{t = \bar{t}_{s - 1}}^{\bar{t}} \frac{\eta_t}{2} \mathbb{E}\|\nabla f(\bar{x}_t) \|^2 + \frac{\sigma^2 c^2 }{32L^2 } \sum_{t = \bar{t}_{s - 1}}^{\bar{t}} \eta_{t}^3 \\
    &   \qquad \qquad     +  {\frac{ L^2 (I - 1)}{K}} {\bigg( I \times \frac{1}{16LI} \times \frac{1}{16LI} \bigg)} \sum_{t = \bar{t}_{s - 1}}^{\bar{t}} \eta_t \sum_{k=1}^K \mathbb{E}\|d_t^{(k)} - \bar{d}_t\|^2  \\
    & \qquad \qquad \qquad \qquad \qquad \qquad \qquad    + {\frac{1}{8 K} \frac{(I - 1)}{I}} \sum_{t = \bar{t}_{s - 1}}^{\bar{t}}  \eta_t  \sum_{k=1}^K \mathbb{E}\| d_{t}^{(k)} - \bar{d}_{t} \|^2  \\
 & =   - \sum_{t = \bar{t}_{s - 1}}^{\bar{t}} \left( \frac{7 \eta_t}{16} - \frac{\eta_t^2 L}{2} \right)  \mathbb{E} \| \bar{d}_{t}  \|^2   - \sum_{t = \bar{t}_{s - 1}}^{\bar{t}} \frac{\eta_t}{2} \mathbb{E} \|\nabla f(\bar{x}_t) \|^2 +  \frac{\sigma^2 c^2 }{32L^2 } \sum_{t = \bar{t}_{s - 1}}^{\bar{t}} \eta_{t}^3 \\
 &\qquad \qquad \qquad \qquad \qquad \qquad \qquad    + { \frac{33}{256 K} \frac{(I - 1)}{I}} \sum_{t = \bar{t}_{s - 1}}^{\bar{t}}  \eta_t  \sum_{k=1}^K \mathbb{E}\| d_{t}^{(k)} - \bar{d}_{t} \|^2.
\end{align*}
Therefore, the lemma is proved. 
\end{proof}

Multiple local updates at each WN on heterogeneous data can cause the local descent directions to drift away from each other.
Next, we bound this error accumulated via gradient drift across WNs. 

\subsubsection{Accumulated Gradient Consensus Error}
We first upper bound the gradient consensus error given by term $\sum_{k=1}^K \mathbb{E}\| d_{t}^{(k)} - \bar{d}_{t} \|^2$. 
\begin{lem}[Gradient Consensus Error]
\label{lem: GradientErrorContraction_BatchGradients_App}
For every $t \in [T]$ and some $\beta > 0$ we have
\begin{align*}
  \sum_{k=1}^K \mathbb{E} \| d_{t}^{(k)} - \bar{d}_{t} \|^2    & \leq   \bigg[ (1 - a_t)^2 (1 + \beta)   + 4 L^2 \bigg( 1 + \frac{1}{\beta}\bigg) \eta_{t-1}^2 \bigg] \sum_{k=1}^K \mathbb{E}\| d_{t-1}^{(k)} - \bar{d}_{t-1}  \|^2 \\
     & \quad \quad \quad  + 4 K L^2  \bigg( 1 + \frac{1}{\beta}\bigg) \eta_{t-1}^2  \mathbb{E}\|  \bar{d}_{t-1}  \|^2 +   \frac{4 K \sigma^2}{b} \bigg( 1 + \frac{1}{\beta}\bigg) a_t^2  + 8 K \zeta^2\bigg( 1 + \frac{1}{\beta}\bigg) a_t^2 \\
    & \qquad \qquad \qquad \qquad  {+ 32L^2 \bigg( 1 + \frac{1}{\beta} \bigg) (I - 1) a_t^2    \sum_{\bar{\ell} = \bar{t}_{s-1}}^{t - 1}    \eta_{\bar{\ell}}^2  \sum_{k = 1}^K \mathbb{E}\|d_{\bar{\ell}}^{(k)} - \bar{d}_{\bar{\ell}} \|^2.}
\end{align*}
where the expectation is w.r.t. the stochasticity of the algorithm. 
\end{lem}
\begin{proof}
Using the definition of the descent direction $d_t^{(k)}$ from Algorithm \ref{Algo_DR-STORM_batch_App} we have
\begin{align}
& \sum_{k=1}^K \mathbb{E} \| d_{t}^{(k)} - \bar{d}_{t} \|^2 \nonumber\\
& = \sum_{k=1}^K \mathbb{E} \bigg\| \frac{1}{b} \sum_{\xi_t^{(k)} \in \mathcal{B}_t^{(k)}} \nabla f^{(k)}(x_t^{(k)}; \xi_t^{(k)}) + (1 - a_t) \big( d_{t-1}^{(k)} -   \frac{1}{b} \sum_{\xi_t^{(k)} \in \mathcal{B}_t^{(k)}} \nabla f^{(k)}(x_{t-1}^{(k)}; \xi_t^{(k)})\big)  \nonumber\\
&      - \bigg( \frac{1}{K} \sum_{j=1}^K  \frac{1}{b} \sum_{\xi_t^{(j)} \in \mathcal{B}_t^{(j)}} \nabla f^{(j)}(x_t^{(j)}; \xi_t^{(j)}) + (1 - a_t) \big( \bar{d}_{t-1} - \frac{1}{K} \sum_{j=1}^K  \frac{1}{b} \sum_{\xi_t^{(j)} \in \mathcal{B}_t^{(j)}} \nabla f^{(j)}(x_{t-1}^{(j)}; \xi_t^{(j)})\big) \bigg)   \bigg\|^2 \nonumber \\
& =  \sum_{k=1}^K \mathbb{E} \bigg\| (1 - a_t) \big( d_{t-1}^{(k)}  - \bar{d}_{t-1} \big) +  \frac{1}{b} \sum_{\xi_t^{(k)} \in \mathcal{B}_t^{(k)}} \nabla f^{(k)}(x_t^{(k)}; \xi_t^{(k)}) - \frac{1}{K} \sum_{j=1}^K  \frac{1}{b} \sum_{\xi_t^{(j)} \in \mathcal{B}_t^{(j)}} \nabla f^{(j)}(x_t^{(j)}; \xi_t^{(j)}) \nonumber\\
& \qquad \qquad \qquad  - (1- a_t) \bigg(  \frac{1}{b} \sum_{\xi_t^{(k)} \in \mathcal{B}_t^{(k)}}  \nabla f^{(k)}(x_{t-1}^{(k)}; \xi_t^{(k)}) - \frac{1}{K} \sum_{j=1}^K  \frac{1}{b} \sum_{\xi_t^{(j)} \in \mathcal{B}_t^{(j)}} \nabla f^{(j)}(x_{t-1}^{(j)}; \xi_t^{(j)}) \bigg)  \bigg\|^2 \nonumber \\
& \overset{(a)}{\leq} (1 + \beta) (1 - a_t)^2 \sum_{k=1}^K \mathbb{E} \|  d_{t-1}^{(k)}  - \bar{d}_{t-1} \|^2 \nonumber\\
&   + \Big( 1 + \frac{1}{\beta} \bigg) \sum_{k=1}^K \mathbb{E} \bigg\|  \frac{1}{b} \sum_{\xi_t^{(k)} \in \mathcal{B}_t^{(k)}} \nabla f^{(k)}(x_t^{(k)}; \xi_t^{(k)}) - \frac{1}{K} \sum_{j=1}^K  \frac{1}{b} \sum_{\xi_t^{(j)} \in \mathcal{B}_t^{(j)}} \nabla f^{(j)}(x_t^{(j)}; \xi_t^{(j)}) \nonumber\\
&  \qquad \qquad \qquad   - (1- a_t) \bigg(  \frac{1}{b} \sum_{\xi_t^{(k)} \in \mathcal{B}_t^{(k)}} \nabla f^{(k)}(x_{t-1}^{(k)}; \xi_t^{(k)}) - \frac{1}{K} \sum_{j=1}^K  \frac{1}{b} \sum_{\xi_t^{(j)} \in \mathcal{B}_t^{(j)}} \nabla f^{(j)}(x_{t-1}^{(j)}; \xi_t^{(j)}) \bigg)  \bigg\|^2
    \label{eq: DR_Descent_d1_Batch}
\end{align}
where inequality $(a)$ follows from the Young's inequality for some $\beta>0$. Now considering the second term in \eqref{eq: DR_Descent_d1_Batch}, we get
\begin{align}
   & \sum_{k=1}^K \mathbb{E}\bigg\|  \frac{1}{b} \sum_{\xi_t^{(k)} \in \mathcal{B}_t^{(k)}}  \nabla f^{(k)}(x_t^{(k)}; \xi_t^{(k)}) - \frac{1}{K} \sum_{j=1}^K  \frac{1}{b} \sum_{\xi_t^{(j)} \in \mathcal{B}_t^{(j)}} \nabla f^{(j)}(x_t^{(j)}; \xi_t^{(j)}) \nonumber\\
   & \qquad \qquad \qquad   - (1- a_t) \bigg(  \frac{1}{b} \sum_{\xi_t^{(k)} \in \mathcal{B}_t^{(k)}} \nabla f^{(k)}(x_{t-1}^{(k)}; \xi_t^{(k)}) - \frac{1}{K} \sum_{j=1}^K  \frac{1}{b} \sum_{\xi_t^{(j)} \in \mathcal{B}_t^{(j)}} \nabla f^{(j)}(x_{t-1}^{(j)}; \xi_t^{(j)}) \bigg)  \bigg\|^2 \nonumber\\
    & = \sum_{k=1}^K \mathbb{E}\bigg\| \frac{1}{b} \sum_{\xi_t^{(k)} \in \mathcal{B}_t^{(k)}} \nabla f^{(k)}(x_t^{(k)}; \xi_t^{(k)}) - \frac{1}{K} \sum_{j=1}^K  \frac{1}{b} \sum_{\xi_t^{(j)} \in \mathcal{B}_t^{(j)}} \nabla f^{(j)}(x_t^{(j)}; \xi_t^{(j)})  \nonumber\\
    & \qquad \qquad \qquad \qquad  -  \bigg( \frac{1}{b} \sum_{\xi_t^{(k)} \in \mathcal{B}_t^{(k)}} \nabla f^{(k)}(x_{t-1}^{(k)}; \xi_t^{(k)}) - \frac{1}{K} \sum_{j=1}^K \frac{1}{b} \sum_{\xi_t^{(j)} \in \mathcal{B}_t^{(j)}} \nabla f^{(j)}(x_{t-1}^{(j)}; \xi_t^{(j)}) \bigg) \nonumber\\
    & \qquad \qquad \qquad   \qquad \qquad   \qquad  + a_t   \bigg(\frac{1}{b} \sum_{\xi_t^{(k)} \in \mathcal{B}_t^{(k)}} \nabla  f^{(k)}(x_{t-1}^{(k)}; \xi_t^{(k)}) - \frac{1}{K} \sum_{j=1}^K \frac{1}{b} \sum_{\xi_t^{(j)} \in \mathcal{B}_t^{(j)}} \nabla f^{(j)}(x_{t-1}^{(j)}; \xi_t^{(j)}) \bigg)   \bigg\|^2 \nonumber\\
    & \overset{(a)}{\leq}  2 \sum_{k=1}^K \mathbb{E}\bigg\| \frac{1}{b} \sum_{\xi_t^{(k)} \in \mathcal{B}_t^{(k)}}  \nabla f^{(k)}(x_t^{(k)}; \xi_t^{(k)}) - \frac{1}{K} \sum_{j=1}^K \frac{1}{b} \sum_{\xi_t^{(j)} \in \mathcal{B}_t^{(j)}} \nabla f^{(j)}(x_t^{(j)}; \xi_t^{(j)}) \nonumber\\
    & \qquad \qquad \qquad \qquad  -  \bigg( \frac{1}{b} \sum_{\xi_t^{(k)} \in \mathcal{B}_t^{(k)}} \nabla f^{(k)}(x_{t-1}^{(k)}; \xi_t^{(k)}) - \frac{1}{K} \sum_{j=1}^K \frac{1}{b} \sum_{\xi_t^{(j)} \in \mathcal{B}_t^{(j)}} \nabla f^{(j)}(x_{t-1}^{(j)}; \xi_t^{(j)}) \bigg) \bigg\|^2 \nonumber\\
    & \qquad \qquad \qquad \qquad\qquad \qquad       + 2 a_t^2 \sum_{k=1}^K \mathbb{E} \bigg\|  \frac{1}{b} \sum_{\xi_t^{(k)} \in \mathcal{B}_t^{(k)}} \nabla f^{(k)}(x_{t-1}^{(k)}; \xi_t^{(k)}) - \frac{1}{K} \sum_{j=1}^K \frac{1}{b} \sum_{\xi_t^{(j)} \in \mathcal{B}_t^{(j)}} \nabla f^{(j)}(x_{t-1}^{(j)}; \xi_t^{(j)})   \bigg\|^2 \nonumber \\
   & \overset{(b)}{\leq}  2 \sum_{k=1}^K \mathbb{E} \bigg\| \frac{1}{b} \sum_{\xi_t^{(k)} \in \mathcal{B}_t^{(k)}} \big( \nabla f^{(k)}(x_t^{(k)}; \xi_t^{(k)}) - \nabla f^{(k)}(x_{t-1}^{(k)}; \xi_t^{(k)}) \big) \bigg\|^2 \nonumber\\
   & \qquad \qquad \qquad \qquad \qquad     + 2 a_t^2 \sum_{k=1}^K \mathbb{E}\bigg\| \frac{1}{b} \sum_{\xi_t^{(k)} \in \mathcal{B}_t^{(k)}} \nabla f^{(k)}(x_{t-1}^{(k)}; \xi_t^{(k)}) - \frac{1}{K} \sum_{j=1}^K \frac{1}{b} \sum_{\xi_t^{(j)} \in \mathcal{B}_t^{(j)}} \nabla f^{(j)}(x_{t-1}^{(j)}; \xi_t^{(j)})   \bigg\|^2 \nonumber \\
  & \overset{(c)}{\leq}  2 \sum_{k=1}^K \frac{1}{b}     \sum_{\xi_t^{(k)} \in \mathcal{B}_t^{(k)}}   \mathbb{E}\big\|  \nabla f^{(k)}(x_t^{(k)}; \xi_t^{(k)}) - \nabla f^{(k)}(x_{t-1}^{(k)}; \xi_t^{(k)})  \big\|^2   \nonumber\\
   & \qquad \qquad \qquad \qquad \qquad     + 2 a_t^2 \sum_{k=1}^K \mathbb{E}\bigg\| \frac{1}{b} \sum_{\xi_t^{(k)} \in \mathcal{B}_t^{(k)}} \nabla f^{(k)}(x_{t-1}^{(k)}; \xi_t^{(k)}) - \frac{1}{K} \sum_{j=1}^K \frac{1}{b} \sum_{\xi_t^{(j)} \in \mathcal{B}_t^{(j)}} \nabla f^{(j)}(x_{t-1}^{(j)}; \xi_t^{(j)})   \bigg\|^2 \nonumber \\
   & \overset{(d)}{\leq}  2 L^2 \sum_{k=1}^K \mathbb{E}\| x_t^{(k)} - x_{t-1}^{(k)} \|^2  + 2 a_t^2 \sum_{k=1}^K \mathbb{E}\bigg\| \frac{1}{b} \sum_{\xi_t^{(k)} \in \mathcal{B}_t^{(k)}} \nabla  f^{(k)}(x_{t-1}^{(k)}; \xi_t^{(k)}) - \frac{1}{K} \sum_{j=1}^K \frac{1}{b} \sum_{\xi_t^{(j)} \in \mathcal{B}_t^{(j)}} \nabla f^{(j)}(x_{t-1}^{(j)}; \xi_t^{(j)})   \bigg\|^2,
   \label{eq: DR_Descent_d_Intermediate1_Batch}
\end{align}
where inequality $(a)$ above follows from Lemma \ref{Lem: Norm_Ineq}, $(b)$ follows from Lemma \ref{Lem: Sum_Mean_Kron}, inequality $(c)$ again uses Lemma \ref{Lem: Norm_Ineq} and $(d)$ follows from the Lipschitz-smoothness of the individual functions $f^{(k)}$ given in Assumption \ref{Ass: Lip_Smoothness}.

Next, we consider the second term in \eqref{eq: DR_Descent_d_Intermediate1_Batch} above, we have
\begin{align}
  &  \sum_{k=1}^K \mathbb{E}\bigg\| \frac{1}{b} \sum_{\xi_t^{(k)} \in \mathcal{B}_t^{(k)}} \nabla   f^{(k)}(x_{t-1}^{(k)}; \xi_t^{(k)}) - \frac{1}{K} \sum_{j=1}^K \frac{1}{b} \sum_{\xi_t^{(j)} \in \mathcal{B}_t^{(j)}} \nabla f^{(j)}(x_{t-1}^{(j)}; \xi_t^{(j)})    \bigg\|^2 \nonumber \\ 
   & \overset{(a)}{=}  \sum_{k=1}^K \mathbb{E}\bigg\|  \frac{1}{b} \sum_{\xi_t^{(k)} \in \mathcal{B}_t^{(k)}} \big( \nabla f^{(k)}(x_{t-1}^{(k)}; \xi_t^{(k)}) - \nabla   f^{(k)}(x_{t-1}^{(k)}) \big) \nonumber\\
   & \qquad \qquad   - \frac{1}{K} \sum_{j=1}^K  \frac{1}{b} \sum_{\xi_t^{(j)} \in \mathcal{B}_t^{(j)}} \big(  \nabla   f^{(j)}(x_{t-1}^{(j)}; \xi_t^{(j)}) - \nabla   f^{(j)}(x_{t-1}^{(j)})   \big)   + \nabla f^{(k)}(x_{t-1}^{(k)}) - \frac{1}{K} \sum_{j=1}^K \nabla f^{(j)}(x_{t-1}^{(j)})   \bigg\|^2 
    \nonumber \\
    & \overset{(b)}{\leq}  2 \sum_{k=1}^K \mathbb{E} \bigg\|  \frac{1}{b} \sum_{\xi_t^{(k)} \in \mathcal{B}_t^{(k)}} \big( \nabla f^{(k)}(x_{t-1}^{(k)}; \xi_t^{(k)}) - \nabla   f^{(k)}(x_{t-1}^{(k)}) \big) \nonumber\\
  & \qquad \qquad \qquad \qquad \qquad  - \frac{1}{K} \sum_{j=1}^K \frac{1}{b} \sum_{\xi_t^{(j)} \in \mathcal{B}_t^{(j)}} \big(  \nabla   f^{(j)}(x_{t-1}^{(j)}; \xi_t^{(j)}) - \nabla   f^{(j)}(x_{t-1}^{(j)})   \big) \bigg\|^2 \nonumber\\
    & \qquad \qquad \qquad \qquad \qquad \qquad \qquad   \qquad \qquad \qquad \qquad     + 2 \sum_{k=1}^K \mathbb{E} \bigg\| \nabla f^{(k)}(x_{t-1}^{(k)}) - \frac{1}{K} \sum_{j=1}^K \nabla f^{(j)}(x_{t-1}^{(j)})   \bigg\|^2 \nonumber \\
  & \overset{(c)}{\leq}   2 \sum_{k=1}^K \mathbb{E} \bigg\|  \frac{1}{b} \sum_{\xi_t^{(k)} \in \mathcal{B}_t^{(k)}}  \big( \nabla   f^{(k)}(x_{t-1}^{(k)}; \xi_t^{(k)}) - \nabla   f^{(k)}(x_{t-1}^{(k)}) \big) \bigg\|^2 \nonumber\\
  & \qquad \qquad \qquad \qquad \qquad \qquad \qquad \qquad + 2 \sum_{k=1}^K \mathbb{E} \bigg\| \nabla f^{(k)}(x_{t-1}^{(k)}) - \frac{1}{K} \sum_{j=1}^K \nabla f^{(j)}(x_{t-1}^{(j)})   \bigg\|^2 \nonumber\\
   & \overset{(d)}{\leq}   2 \sum_{k=1}^K \frac{1}{b^2} \sum_{\xi_t^{(k)} \in \mathcal{B}_t^{(k)}} \mathbb{E} \big\|    \big( \nabla   f^{(k)}(x_{t-1}^{(k)}; \xi_t^{(k)}) - \nabla   f^{(k)}(x_{t-1}^{(k)}) \big) \big\|^2 + 4 \sum_{k=1}^K \mathbb{E}\big\|\nabla f^{(k)}(\bar{x}_{t-1}) - \nabla f(\bar{x}_{t-1}) \big\|^2
   \nonumber\\
   & \qquad \qquad \qquad     +  
    8 \sum_{k=1}^K \mathbb{E} \big\| \nabla f^{(k)}(x_{t-1}^{(k)}) - \nabla f^{(k)}(\bar{x}_{t-1})   \big\|^2  + 8  \sum_{k=1}^K \mathbb{E} \bigg\| \nabla f(\bar{x}_{t-1}) -  \frac{1}{K} \sum_{j=1}^K \nabla f^{(j)}(x_{t-1}^{(j)})  \bigg\|^2  \nonumber\\
 & \overset{(e)}{\leq} \frac{2 K \sigma^2}{b}      + 4 \sum_{k=1}^K \frac{1}{K} \sum_{j=1}^K  \mathbb{E}  \|  \nabla f^{(k)}(\bar{x}_{t-1}) - \nabla f^{(j)}(\bar{x}_{t-1})  \|^2 +  16 L^2 \sum_{k = 1}^K \mathbb{E}\| x_{t - 1}^{(k)} - \bar{x}_{t-1}\|^2 \nonumber\\
 & \overset{(g)}{\leq} \frac{2 K \sigma^2}{b} + 4 K \zeta^2   +  16 L^2 \sum_{k = 1}^K \mathbb{E}\| x_{t - 1}^{(k)} - \bar{x}_{t-1}\|^2,
  \label{eq: DR_Descent_d_Intermediate2_Batch}
\end{align}
where equality $(a)$ follows from adding and subtracting $\nabla f^{(k)}(x_{t-1}^{(k)})$ and $\frac{1}{K} \sum_{j =1}^K \nabla f^{(j)}(x_{t-1}^{(j)})$  inside the norm; inequality $(b)$ uses Lemma \ref{Lem: Norm_Ineq}; inequality $(c)$ results from the use of Lemma \ref{Lem: Sum_Mean_Kron}; inequality $(d)$ expands the sum of the first term using inner products and utilizes the fact that the cross product terms are zero in expectation. This follows from the fact that conditioned on $\mathcal{F}_t$ we have $\mathbb{E}[ \nabla f^{(k)}(x_t^{(k)} ; \xi_t^{(k)}) ] = \nabla f^{(k)}(x_t^{(k)})$ for all $k \in [K]$ and $t \in [T]$; inequality $(e)$ utilizes intra-node variance Bound given in Assumption \ref{Ass: Unbiased_Var_Grad}, Lemma \ref{Lem: Norm_Ineq} and Lipschitz smoothness Assumption \ref{Ass: Lip_Smoothness}; finally, $(g)$ results from the inter-node variance bound stated in Assumption \ref{Ass: Unbiased_Var_Grad}.

Finally, substituting \eqref{eq: DR_Descent_d_Intermediate2_Batch} and \eqref{eq: DR_Descent_d_Intermediate1_Batch} in \eqref{eq: DR_Descent_d1_Batch}, we get
\begin{align*}
    \sum_{k=1}^K \mathbb{E} \| d_{t}^{(k)} - \bar{d}_{t} \|^2 & \leq (1 - a_t)^2 (1 + \beta) \sum_{k=1}^K \mathbb{E} \big\|   d_{t-1}^{(k)}  - \bar{d}_{t-1} \big\|^2  + 2 L^2 \bigg( 1 + \frac{1}{\beta}\bigg) \sum_{k=1}^K \mathbb{E}\| x_t^{(k)} - x_{t-1}^{(k)} \|^2 \\
    &       + \frac{4 K \sigma^2}{b} \bigg( 1 + \frac{1}{\beta}\bigg) a_t^2  + 8 K \zeta^2\bigg( 1 + \frac{1}{\beta}\bigg) a_t^2 +  32L^2 \bigg( 1 + \frac{1}{\beta} \bigg) a_t^2 \sum_{k = 1}^K \mathbb{E}\|x_{t - 1}^{(k)} - \bar{x}_{t - 1}\|^2 \\
    & \overset{(a)}{\leq} (1 - a_t)^2 (1 + \beta) \sum_{k=1}^K \mathbb{E} \big\|   d_{t-1}^{(k)}  - \bar{d}_{t-1} \big\|^2  + 2 L^2 \bigg( 1 + \frac{1}{\beta}\bigg) \eta_{t-1}^2 \sum_{k=1}^K \mathbb{E}\| d_{t-1}^{(k)} \|^2 \\
    &  \qquad \qquad \qquad \qquad   + \frac{4 K \sigma^2}{b} \bigg( 1 + \frac{1}{\beta}\bigg) a_t^2  + 8 K \zeta^2\bigg( 1 + \frac{1}{\beta}\bigg) a_t^2  \\
    & \qquad \qquad\qquad \qquad\qquad \qquad  +  32L^2 \bigg( 1 + \frac{1}{\beta} \bigg)  (I - 1) a_t^2   \sum_{\bar{\ell} = \bar{t}_{s-1}}^{t - 1}    \eta_{\bar{\ell}}^2  \sum_{k = 1}^K \mathbb{E}\|d_{\bar{\ell}}^{(k)} - \bar{d}_{\bar{\ell}} \|^2  \\
     & \overset{(b)}{\leq} (1 - a_t)^2 (1 + \beta) \sum_{k=1}^K \mathbb{E} \big\|   d_{t-1}^{(k)}  - \bar{d}_{t-1} \big\|^2  + 4 L^2 \bigg( 1 + \frac{1}{\beta}\bigg) \eta_{t-1}^2 \sum_{k=1}^K \mathbb{E}\| d_{t-1}^{(k)} - \bar{d}_{t-1}  \|^2 \\
     &  \quad    + 4 L^2 \bigg( 1 + \frac{1}{\beta}\bigg) \eta_{t-1}^2 \sum_{k=1}^K \mathbb{E}\|  \bar{d}_{t-1}  \|^2 +  \frac{4 K \sigma^2}{b} \bigg( 1 + \frac{1}{\beta}\bigg) a_t^2  + 8 K \zeta^2\bigg( 1 + \frac{1}{\beta}\bigg) a_t^2\\ 
     & \qquad \qquad \qquad \qquad \qquad   + 32L^2 \bigg( 1 + \frac{1}{\beta} \bigg) (I - 1) a_t^2    \sum_{\bar{\ell} = \bar{t}_{s-1}}^{t - 1}    \eta_{\bar{\ell}}^2  \sum_{k = 1}^K \mathbb{E}\|d_{\bar{\ell}}^{(k)} - \bar{d}_{\bar{\ell}} \|^2 \\
     & = \bigg[ (1 - a_t)^2 (1 + \beta)   + 4 L^2 \bigg( 1 + \frac{1}{\beta}\bigg) \eta_{t-1}^2 \bigg] \sum_{k=1}^K \mathbb{E}\| d_{t-1}^{(k)} - \bar{d}_{t-1}  \|^2 \\
     & \quad     + 4 K L^2  \bigg( 1 + \frac{1}{\beta}\bigg) \eta_{t-1}^2  \mathbb{E}\|  \bar{d}_{t-1}  \|^2 +   \frac{4 K \sigma^2}{b} \bigg( 1 + \frac{1}{\beta}\bigg) a_t^2  + 8 K \zeta^2\bigg( 1 + \frac{1}{\beta}\bigg) a_t^2\\
     & \qquad \quad \quad  \qquad \qquad \qquad  { +~ 32L^2 \bigg( 1 + \frac{1}{\beta} \bigg)  (I - 1)a_t^2   \sum_{\bar{\ell} = \bar{t}_{s-1}}^{t - 1}    \eta_{\bar{\ell}}^2  \sum_{k = 1}^K \mathbb{E}\|d_{\bar{\ell}}^{(k)} - \bar{d}_{\bar{\ell}} \|^2,}
\end{align*}
where inequality $(a)$ follows from the iterate update given in Step 10 of Algorithm \ref{Algo_DR-STORM_batch_App} and the application Lemma \ref{lem: ErrorAccumulation_Batch_App}; and inequality $(b)$ results from the use of Lemma \ref{Lem: Norm_Ineq}. 
\end{proof}
Next, we utilize the above Lemma \ref{lem: GradientErrorContraction_BatchGradients_App} to bound the accumulated gradient consensus error in the potential function's descent derived in Lemma \ref{Lem: Potential_Descent_App}.

\begin{lem}[Accumulated Gradient Consensus Error]
\label{lem: GradientError_PotentialFn_App}
For $\bar{t} \in [\bar{t}_{s - 1}, \bar{t}_s - 1]$ with $s \in [S]$ we have
\begin{align*}
   \frac{33}{256 K} \frac{(I - 1)}{I}  \sum_{t = \bar{t}_{s - 1}}^{\bar{t}}  \eta_t  \sum_{k=1}^K \mathbb{E}\| d_{t}^{(k)} - \bar{d}_{t} \|^2 & \leq   \sum_{t = \bar{t}_{s - 1}}^{\bar{t}} \frac{ \eta_{t}}{64}  \mathbb{E}\|  \bar{d}_{t}  \|^2  +   {\frac{ \sigma^2 c^2}{64 b L^2}}     \sum_{t = \bar{t}_{s - 1}}^{\bar{t}}  \eta_t^3   +   {\frac{ \zeta^2 c^2}{ 32 L^2}} \frac{(I - 1)}{I}     \sum_{t = \bar{t}_{s - 1}}^{\bar{t}}   \eta_t^3  .
\end{align*}
\end{lem}
\begin{proof}
First, from the statement of Lemma \ref{lem: GradientErrorContraction_BatchGradients_App}, considering the coefficient of first term on the right hand side of the expression, we have:
\begin{align*}
    (1 - a_t)^2 (1 + \beta)   + 4 L^2 \bigg( 1 + \frac{1}{\beta}\bigg) \eta_{t - 1}^2 & \overset{(a)}{\leq}   1 + \beta   + 4 L^2 \bigg( 1 + \frac{1}{\beta}\bigg) \eta_{t - 1}^2 \\
    & \overset{(b)}{\leq} { 1 + \frac{1}{I} + 4 L^2 (I + 1) \eta_{t - 1}^2} \\
    & \overset{(c)}{\leq} {1 + \frac{1}{I} + \frac{I + 1}{64 I^2}} \\
    & {\overset{(d)}{\leq} 1 + \frac{33}{32I},}
\end{align*}
where inequality $(a)$ uses the fact that $(1 - a_t)^2 \leq 1$; { the second inequality $(b)$ follows from taking $\beta = {1}/{I}$, inequality $(c)$ uses the bound {$\eta_t \leq 1 / 16 LI$} for all $t \in [T]$. Finally, the last inequality $(d)$ results by using the fact that we have $I + 1 \leq 2I$.} Substituting in the statement of Lemma \ref{lem: GradientErrorContraction_BatchGradients_App} above, we get
\begin{align}
\label{Eq: GradientError_t_App}
  \sum_{k=1}^K \mathbb{E} \| d_t^{(k)} - \bar{d}_t \|^2    & \leq   { \bigg( 1 + \frac{33}{32 I} \bigg)} \sum_{k=1}^K \mathbb{E}\| d_{t - 1}^{(k)} - \bar{d}_{t - 1}  \|^2 + 4 K L^2  \bigg( 1 + \frac{1}{\beta}\bigg) \eta_{t-1}^2  \mathbb{E}\|  \bar{d}_{t-1}  \|^2  +   \frac{4 K \sigma^2}{b} \bigg( 1 + \frac{1}{\beta}\bigg) a_t^2  \nonumber \\
   & \qquad \qquad  + 8 K \zeta^2\bigg( 1 + \frac{1}{\beta}\bigg) a_t^2   {+ 32L^2 \bigg( 1 + \frac{1}{\beta} \bigg)  (I - 1)a_t^2   \sum_{\bar{\ell} = \bar{t}_{s-1}}^{t - 1}    \eta_{\bar{\ell}}^2  \sum_{k = 1}^K \mathbb{E}\|d_{\bar{\ell}}^{(k)} - \bar{d}_{\bar{\ell}} \|^2.} \nonumber\\
   & \overset{(a)}{\leq}   { \bigg( 1 + \frac{33}{32 I} \bigg)} \sum_{k=1}^K \mathbb{E}\| d_{t - 1}^{(k)} - \bar{d}_{t - 1}  \|^2 + 8 K L^2 I  \eta_{t-1}^2  \mathbb{E}\|  \bar{d}_{t-1}  \|^2  +   \frac{8 K I \sigma^2}{b} c^2 \eta_{t-1}^4  \nonumber \\
   & \qquad \qquad \qquad \qquad  + 16 K I \zeta^2 c^2 \eta_{t-1}^4   { + 64 L^2 I^2 c^2 \eta_{t-1}^4   \sum_{\bar{\ell} = \bar{t}_{s-1}}^{t - 1}    \eta_{\bar{\ell}}^2  \sum_{k = 1}^K \mathbb{E}\|d_{\bar{\ell}}^{(k)} - \bar{d}_{\bar{\ell}} \|^2} \nonumber\\
  & \overset{(b)}{\leq}   { \bigg( 1 + \frac{33}{32 I} \bigg)} \sum_{k=1}^K \mathbb{E}\| d_{t - 1}^{(k)} - \bar{d}_{t - 1}  \|^2 + \frac{K L}{2}   \eta_{t-1}  \mathbb{E}\|  \bar{d}_{t-1}  \|^2  +   \frac{ K \sigma^2 c^2}{2 b L}  \eta_{t-1}^3  \nonumber \\
   & \qquad \qquad \qquad\qquad \qquad  + \frac{ K  \zeta^2 c^2}{L} \eta_{t-1}^3   { +~ 64 L^2 I^2 c^2 \eta_{t-1}^4   \sum_{\bar{\ell} = \bar{t}_{s-1}}^{t - 1}    \eta_{\bar{\ell}}^2  \sum_{k = 1}^K \mathbb{E}\|d_{\bar{\ell}}^{(k)} - \bar{d}_{\bar{\ell}} \|^2}
\end{align}
where $(a)$ follows from using $\beta = 1/I$, the fact that $I + 1 \leq 2I$ and the definition of $a_t$ from Algorithm \ref{Algo_DR-STORM_batch_App}.
 
Note form Algorithm \ref{Algo_DR-STORM_batch_App} that we have $d_{t}^{(k)} = \bar{d}_t$ for $t = \bar{t}_{s- 1}$ with $s \in [S]$. This implies that for ${t} = \bar{t}_{s - 1}$ with $s \in [S]$, we have, $\sum_{k = 1}^K \| d_{t}^{(k)} - \bar{d}_t \|^2 = 0$. Applying \eqref{Eq: GradientError_t_App} above recursively for $t \in [\bar{t}_{s-1}+1, \bar{t}_s - 1]$ we get:
\begin{align}
  & \sum_{k=1}^K \mathbb{E} \| d_{t}^{(k)} - \bar{d}_{t} \|^2     \leq    \frac{K L}{2} \sum_{\ell = \bar{t}_{s - 1}}^{t-1} {\bigg( 1 + \frac{33}{32I} \bigg)^{t - 1 -\ell}} \eta_{\ell}  \mathbb{E}\|  \bar{d}_{\ell}  \|^2  + {\frac{K \sigma^2 c^2}{2 b L}}  \sum_{\ell = \bar{t}_{s - 1}}^{t-1} { \bigg( 1 + \frac{33}{32I} \bigg)^{t - 1 -\ell}} \eta_\ell^3  \nonumber\\
  &       + {\frac{K \zeta^2 c^2}{ L}}  \sum_{\ell = \bar{t}_{s - 1}}^{t-1} {\bigg( 1 + \frac{3}{2I} \bigg)^{t - 1 -\ell}} \eta_\ell^3  + 64 L^2 I^2 c^2 \sum_{\ell = {\bar{t}_{s-1}}}^{t-1} \bigg( 1 + \frac{33}{32I} \bigg)^{t - 1- \ell}   \eta_{\ell}^4   \sum_{\bar{\ell} = \bar{t}_{s-1}}^{\ell}    \eta_{\bar{\ell}}^2  \sum_{k = 1}^K \mathbb{E}\|d_{\bar{\ell}}^{(k)} - \bar{d}_{\bar{\ell}} \|^2\nonumber\\
  & \qquad \overset{(a)}{\leq}     \frac{K L}{2} \bigg( 1 + \frac{33}{32I} \bigg)^I  \sum_{\ell = \bar{t}_{s - 1}}^t  \eta_{\ell}  \mathbb{E}\|  \bar{d}_{\ell}  \|^2  + {\frac{K \sigma^2 c^2}{2 b L}} \bigg( 1 + \frac{33}{32I} \bigg)^I   \sum_{\ell = \bar{t}_{s - 1}}^t  \eta_\ell^3  \nonumber\\
  &  \qquad     + {\frac{K \zeta^2 c^2}{ L}}  \bigg( 1 + \frac{33}{32I} \bigg)^I  \sum_{\ell = \bar{t}_{s - 1}}^t   \eta_\ell^3  + 64 L^2 I^3 c^2 \bigg(\frac{1}{16LI} \bigg)^5  \bigg( 1 + \frac{33}{32I} \bigg)^I    \sum_{\bar{\ell} = \bar{t}_{s-1}}^{t}    \eta_{\bar{\ell}}  \sum_{k = 1}^K \mathbb{E}\|d_{\bar{\ell}}^{(k)} - \bar{d}_{\bar{\ell}} \|^2 \nonumber\\
  & \qquad \overset{(b)}{\leq} \frac{3 K L}{2}   \sum_{\ell = \bar{t}_{s - 1}}^t  \eta_{\ell}  \mathbb{E}\|  \bar{d}_{\ell}  \|^2  + {\frac{3 K \sigma^2 c^2}{2 b L}}     \sum_{\ell = \bar{t}_{s - 1}}^t  \eta_\ell^3     + {\frac{3 K \zeta^2 c^2}{ L}}     \sum_{\ell = \bar{t}_{s - 1}}^t   \eta_\ell^3  \nonumber\\
  &   \qquad \qquad \qquad \qquad \qquad \qquad \quad + 192 L^2 I^3 c^2 \bigg(\frac{1}{16LI} \bigg)^5      \sum_{\ell = \bar{t}_{s-1}}^{t}    \eta_{\ell}  \sum_{k = 1}^K \mathbb{E}\|d_\ell^{(k)} - \bar{d}_\ell \|^2,
  \label{Eq: GradientError_Recursive}
\end{align}
where inequality $(a)$ follows from the fact that $1 + 33/32I > 1$ and $t - 1 - \ell \leq I$ for $t \in [\bar{t}_{s - 1} , \bar{t}_s - 1]$ and $\ell \in [\bar{t}_{s - 1}, t]$ and inequality $(b)$ follows from the fact that $(1 + 33/32I)^I \leq \mathrm{e}^{33/32} < 3$ and $\eta_t \leq \frac{1}{16LI}$ for all $t \in [T]$.

Next, multiplying both sides of \eqref{Eq: GradientError_Recursive} by $\eta_t$ and summing over $t = \bar{t}_{s - 1}$ to $\bar{t}$ for $\bar{t} \in [\bar{t}_{s - 1}, \bar{t}_s - 1]$ with $s \in [S]$
\begin{align*}
 &   \sum_{t = \bar{t}_{s - 1}}^{\bar{t}}  \eta_t  \sum_{k=1}^K \mathbb{E}\| d_{t}^{(k)} - \bar{d}_{t} \|^2  \leq   \frac{3 K L}{2}  \sum_{t = \bar{t}_{s - 1}}^{\bar{t}}  \eta_t  \sum_{\ell = \bar{t}_{s - 1}}^t  \eta_{\ell}  \mathbb{E}\|  \bar{d}_{\ell}  \|^2  + {\frac{3 K \sigma^2 c^2}{2 b L}}   \sum_{t = \bar{t}_{s - 1}}^{\bar{t}}  \eta_t   \sum_{\ell = \bar{t}_{s - 1}}^t  \eta_\ell^3   \nonumber\\
  &  \qquad \qquad \qquad  + {\frac{3 K \zeta^2 c^2}{ L}}      \sum_{t = \bar{t}_{s - 1}}^{\bar{t}}  \eta_t \sum_{\ell = \bar{t}_{s - 1}}^t   \eta_\ell^3    + 192 L^2 I^3 c^2 \bigg(\frac{1}{16LI} \bigg)^5    \sum_{t = \bar{t}_{s - 1}}^{\bar{t}}  \eta_t   \sum_{\ell = \bar{t}_{s-1}}^{t}    \eta_{\ell}  \sum_{k = 1}^K \mathbb{E}\|d_\ell^{(k)} - \bar{d}_\ell \|^2 \\
 &   \overset{(a)}{\leq}   \frac{3 K L}{2}  \bigg( \sum_{t = \bar{t}_{s - 1}}^{\bar{t}}  \eta_t  \bigg) \sum_{\ell = \bar{t}_{s - 1}}^{\bar{t}}  \eta_{\ell}  \mathbb{E}\|  \bar{d}_{\ell}  \|^2  + {\frac{3 K \sigma^2 c^2}{2 b L}}   \bigg( \sum_{t = \bar{t}_{s - 1}}^{\bar{t}}  \eta_t  \bigg) \sum_{\ell = \bar{t}_{s - 1}}^{\bar{t}}  \eta_\ell^3   \nonumber\\
  &  \qquad \qquad   + {\frac{3 K \zeta^2 c^2}{ L}}    \bigg(  \sum_{t = \bar{t}_{s - 1}}^{\bar{t}}  \eta_t \bigg) \sum_{\ell = \bar{t}_{s - 1}}^{\bar{t}}   \eta_\ell^3    + 192 L^2 I^3 c^2 \bigg(\frac{1}{16LI} \bigg)^5  \bigg(  \sum_{t = \bar{t}_{s - 1}}^{\bar{t}}  \eta_t  \bigg)  \sum_{\ell = \bar{t}_{s-1}}^{\bar{t}}    \eta_{\ell}  \sum_{k = 1}^K \mathbb{E}\|d_\ell^{(k)} - \bar{d}_\ell \|^2\\
 &    \overset{(b)}{\leq}   \frac{3 K}{32}    \sum_{t = \bar{t}_{s - 1}}^{\bar{t}}  \eta_{t}  \mathbb{E}\|  \bar{d}_{t}  \|^2  + {\frac{3 K \sigma^2 c^2}{32 b L^2}}     \sum_{t = \bar{t}_{s - 1}}^{\bar{t}}  \eta_t^3   + {\frac{3 K \zeta^2 c^2}{ 16 L^2}}      \sum_{t = \bar{t}_{s - 1}}^{\bar{t}}   \eta_t^3   + 192 L^2 I^4 c^2 \bigg(\frac{1}{16LI} \bigg)^6    \sum_{t = \bar{t}_{s-1}}^{\bar{t}}    \eta_{t}  \sum_{k = 1}^K \mathbb{E}\|d_t^{(k)} - \bar{d}_t \|^2
\end{align*}
where inequality $(a)$ uses the fact that $t \in [\bar{t}_{s - 1} , \bar{t}]$ and $(b)$ follows from the fact that we have $\eta_t \leq 1/16 LI$ for all $t \in [T]$. Rearranging the terms we get
\begin{align*}
\bigg[ 1 -  192 L^2 I^4 c^2 \bigg(\frac{1}{16LI} \bigg)^6  \bigg]      \sum_{t = \bar{t}_{s - 1}}^{\bar{t}}  \eta_t  \sum_{k=1}^K \mathbb{E}\| d_{t}^{(k)} - \bar{d}_{t} \|^2 & \leq \frac{3 K}{32}    \sum_{t = \bar{t}_{s - 1}}^{\bar{t}}  \eta_{t}  \mathbb{E}\|  \bar{d}_{t}  \|^2 \\
&  \qquad  + {\frac{3 K \sigma^2 c^2}{32 b L^2}}     \sum_{t = \bar{t}_{s - 1}}^{\bar{t}}  \eta_t^3   + {\frac{3 K \zeta^2 c^2}{ 16 L^2}}      \sum_{t = \bar{t}_{s - 1}}^{\bar{t}}   \eta_t^3  
\end{align*}
using the fact that $c \leq 128 L^2/bK$, $b \geq 1$, $K \geq 1$ and $I \geq 1$, we have $\Big[ 1 -  192 L^2 I^4 c^2 \Big(\frac{1}{16LI} \Big)^6  \Big] \geq \frac{4}{5}$, therefore, we get
\begin{align*}
\frac{33}{256 K} \frac{(I - 1)}{I}  \sum_{t = \bar{t}_{s - 1}}^{\bar{t}}  \eta_t  \sum_{k=1}^K \mathbb{E}\| d_{t}^{(k)} - \bar{d}_{t} \|^2 & \leq       \sum_{t = \bar{t}_{s - 1}}^{\bar{t}} \frac{ \eta_{t}}{64}  \mathbb{E}\|  \bar{d}_{t}  \|^2  +   {\frac{ \sigma^2 c^2}{64 b L^2}}     \sum_{t = \bar{t}_{s - 1}}^{\bar{t}}  \eta_t^3   +   {\frac{ \zeta^2 c^2}{ 32 L^2}} \frac{(I - 1)}{I}     \sum_{t = \bar{t}_{s - 1}}^{\bar{t}}   \eta_t^3 .
\end{align*}
Hence, the lemma is proved. 
\end{proof}

 \subsubsection{Proof of Theorem \ref{Thm: PR_Convergence_Main}}
Next, to prove Theorem \ref{Thm: PR_Convergence_Main} we first prove an intermediate theorem by utilizing Lemmas \ref{lem: GradientError_PotentialFn_App} and \ref{Lem: Potential_Descent_App} derived above. 
\begin{theorem} 
\label{Thm: PR_Convergence_Main_App}
Choosing the parameters as
\begin{enumerate}[label = (\roman*)]
    \item $\displaystyle \bar{\kappa} = \frac{(bK)^{2/3} \sigma^{2/3}}{L}$, \vspace{0.1 in}
    \item {$\displaystyle c = \frac{64L^2}{bK} + \frac{\sigma^2}{24  \bar{\kappa}^3 LI}  \overset{(i)}{=} L^2 \bigg(\frac{64}{bK} + \frac{1}{24 (bK)^{2}  I} \bigg) {\leq} \frac{128 L^2}{bK}$,} \vspace{0.1 in}
    \item We choose $\{w_t\}_{t=0}^T$ as
    \begin{align*}
{ w_t = \max \bigg\{2 \sigma^2,  4096 L^3 I^3\bar{\kappa}^3 - \sigma^2t,  \frac{c^3 \bar{\kappa}^3 }{4096 L^3I^3} \bigg\}  \overset{(i)(ii)}{\leq}  \sigma^2 \max\bigg\{2 ,~ 4096 I^3 (bK)^{2} - t,  ~\frac{512}{bK  I^3} \bigg\}.}
\end{align*}
\end{enumerate}
Moreover, for any number of local updates, $I \geq 1$, batch sizes, $b \geq 1$, and initial batch size, $B \geq 1$, computed at individual WNs, \aname~satisfies:
\begin{align*}
 & \mathbb{E} \|\nabla f(\bar{x}_a) \|^2
     \leq    
   \bigg[\frac{32 LI}{T} + \frac{2 L}{(b K)^{2/3} T^{2/3}}\bigg] (f(\bar{x}_1) -   f^\ast )    
   + \bigg[ \frac{8 b I^2}{B T} + \frac{b I}{2  (b K)^{2/3} B T^{2/3}} \bigg] \sigma^2 \\
 & \qquad   +  \bigg[  \frac{256^2 I}{T} +  \frac{64^2}{(bK)^{2/3} T^{2/3}} \bigg] \sigma^2 \log(T+1)  + \bigg[  \frac{256^2 I}{T} +  \frac{64^2}{(bK)^{2/3} T^{2/3}} \bigg] \zeta^2 \frac{(I - 1)}{I} \log(T+1).
\end{align*}
\end{theorem}
\begin{proof}
Substituting the gradient consensus error derived in Lemma \ref{lem: GradientError_PotentialFn_App} into the Potential function descent derived in Lemma \ref{Lem: Potential_Descent_App}, we can write the descent of potential function for $\bar{t} \in [\bar{t}_{s-1}, \bar{t}_s - 1]$ with $s \in [S]$ as:
\begin{align*}
    \mathbb{E}[ \Phi_{\bar{t} + 1} - \Phi_{\bar{t}_{s - 1}}] & \leq   {- \sum_{t = \bar{t}_{s-1}}^{\bar{t}} \left( \frac{27 \eta_t}{64} - \frac{\eta_t^2 L}{2} \right)  \mathbb{E} \| \bar{d}_{t}  \|^2 }  - \sum_{t = \bar{t}_{s - 1}}^{\bar{t}} \frac{\eta_t}{2} \mathbb{E} \|\nabla f(\bar{x}_t) \|^2      \\
    & \qquad \qquad \qquad       { +  \frac{c^2 \sigma^2   }{32L^2 } \sum_{t = \bar{t}_{s - 1}}^{\bar{t}}   \eta_{t}^3 + {\frac{ c^2   \sigma^2}{64 b L^2} } \sum_{t = \bar{t}_{s-1}}^{\bar{t}}  \eta_t^3 + {\frac{ c^2   \zeta^2}{32 L^2} \frac{(I - 1)}{I}} \sum_{t = \bar{t}_{s-1}}^{\bar{t}}  \eta_t^3 }\\
    & \overset{(a)}{\leq}   - \sum_{t = \bar{t}_{s - 1}}^{\bar{t}} \frac{\eta_t}{2} \mathbb{E} \|\nabla f(\bar{x}_t) \|^2       +  {\frac{3 c^2 \sigma^2    }{64 L^2 }} \sum_{t = \bar{t}_{s - 1}}^{\bar{t}}   \eta_{t}^3 + {\frac{ c^2   \zeta^2}{32 L^2} \frac{(I - 1)}{I}}  \sum_{t = \bar{t}_{s - 1}}^{\bar{t}}  \eta_t^3. 
\end{align*}
where $(a)$ follows from the fact that $\eta_t \leq \frac{1}{16LI}$ for all $t \in [T]$ and $b \geq 1$. Taking $\bar{t} = \bar{t}_s - 1= sI$, the above expression can be written as:
\begin{align*}
    \mathbb{E}[ \Phi_{\bar{t}_s} - \Phi_{\bar{t}_{s - 1}}] & \leq     - \sum_{t = \bar{t}_{s - 1}}^{\bar{t}_s - 1} \frac{\eta_t}{2} \mathbb{E} \|\nabla f(\bar{x}_t) \|^2       +  {\frac{3 c^2 \sigma^2    }{64 L^2 }} \sum_{t = \bar{t}_{s - 1}}^{\bar{t}_s - 1}   \eta_{t}^3 + {\frac{ c^2   \zeta^2}{32 L^2} \frac{(I - 1)}{I}}  \sum_{t = \bar{t}_{s - 1}}^{\bar{t}_s - 1}  \eta_t^3. 
\end{align*}
Summing over all the restarts, i.e, $s \in [S]$, we get:
\begin{align*}
    \mathbb{E}[ \Phi_{\bar{t}_S} - \Phi_{\bar{t}_{0}}] & {\leq}     - \sum_{t = \bar{t}_{0}}^{\bar{t}_S - 1} \frac{\eta_t}{2} \mathbb{E} \|\nabla f(\bar{x}_t) \|^2       +  {\frac{3 c^2 \sigma^2    }{64 L^2 }} \sum_{t = \bar{t}_{0}}^{\bar{t}_S - 1}   \eta_{t}^3 + {\frac{ c^2   \zeta^2}{32 L^2}  \frac{(I - 1)}{I}}  \sum_{t = \bar{t}_{0}}^{\bar{t}_S - 1}  \eta_t^3. 
\end{align*}
Assuming that $T = SI$, then from the definition of $\bar{t}_s$ that $\bar{t}_0 = 1$ and $\bar{t}_S = SI + 1 = T + 1$, we get
\begin{align}
   \sum_{t = 1}^T \frac{\eta_t}{2} \mathbb{E} \|\nabla f(\bar{x}_t) \|^2    &  \leq   \mathbb{E} [\Phi_{1} -   \Phi_{T+1}  ]    +  {\frac{3 c^2 \sigma^2    }{64 L^2 }} \sum_{t = 1}^{T}   \eta_{t}^3 + {\frac{ c^2   \zeta^2}{32 L^2} \frac{(I - 1)}{I}}  \sum_{t = 1}^{T}  \eta_t^3 \nonumber\\
   &  \overset{(a)}{\leq}    {f(\bar{x}_1) -   f^\ast + \frac{bK}{64L^2} \frac{\mathbb{E}\|\bar{e}_1 \|^2}{\eta_0}  +
   {\frac{3 c^2 \sigma^2    }{64 L^2 }} \sum_{t = 1}^{T}   \eta_{t}^3 + {\frac{ c^2   \zeta^2}{32 L^2} \frac{(I - 1)}{I}}  \sum_{t = 1}^{T}  \eta_t^3 } \nonumber\\
  &  \overset{(b)}{\leq}    f(\bar{x}_1) -   f^\ast + \frac{\sigma^2}{64 L^2} \frac{b}{B\eta_0} +  {\frac{3 c^2 \sigma^2    }{64 L^2 }} \sum_{t = 1}^{T}   \eta_{t}^3 + {\frac{ c^2   \zeta^2}{32 L^2} \frac{(I - 1)}{I}}  \sum_{t = 1}^{T}  \eta_t^3 .
  \label{Eq: Grad_Norm_Sum}
\end{align}
where $(a)$ follows from the fact that $f^\ast \leq \Phi_{T+1}$ and $(b)$ results from application of Lemma \ref{Lem: e_bar_bound_Batch}.

First, let us consider the last term of the \eqref{Eq: Grad_Norm_Sum} above, we have from the definition of the stepsize $\eta_t$ 
\begin{align}
     \sum_{t=1}^T \eta_t^3 & =    \sum_{t = 1}^{T} \frac{\bar{\kappa}^3 }{w_t + \sigma^2 t} \nonumber \\
     & \overset{(a)}{\leq}  \sum_{t = 1}^{T} \frac{\bar{\kappa}^3  }{\sigma^2 + \sigma^2 t}\nonumber \\
    & = \frac{ \bar{\kappa}^3}{\sigma^2}   \sum_{t = 1}^{T} \frac{1}{1 +   t} \nonumber\\
     & \overset{(b)}{\leq} \frac{  \bar{\kappa}^3 }{\sigma^2}   \ln(T+1).
     \label{Eq: Sum_OverT_LastTerm}
\end{align}
where inequality $(a)$ above follows from the fact that we have 
$w_t \geq 2 \sigma^2 > \sigma^2$ and inequality $(b)$ follows from the application of Lemma \ref{Lem: AD_Sum_1overT}.

Substituting \eqref{Eq: Sum_OverT_LastTerm} in \eqref{Eq: Grad_Norm_Sum}, dividing both sides by $T$ and using the fact that $\eta_t$ is non-increasing in $t$ we have
\begin{align}
  \frac{1}{T} \sum_{t = 1}^T  \mathbb{E} \|\nabla f(\bar{x}_t) \|^2
   &  \leq  \frac{2 (f(\bar{x}_1) -   f^\ast )}{\eta_T T}  + \frac{1}{\eta_T T} \frac{\sigma^2 }{32 L^2} \frac{b}{B\eta_0 }   +   {\frac{1}{\eta_T T} \frac{3 c^2 \bar{\kappa}^3}{32 L^2}  \log(T+1)}\nonumber \\
   & \qquad \qquad \qquad \qquad \qquad \qquad \quad
   + \frac{1}{\eta_T T}  \frac{c^2 \bar{\kappa}^3}{16L^2} \frac{ \zeta^2}{ \sigma^2 }  \frac{(I - 1)}{I} \log(T+1) \nonumber \\
  & \overset{(a)}{\leq}  \frac{2 (f(\bar{x}_1) -   f^\ast )}{\eta_T T}  + \frac{1}{\eta_T T} \frac{\sigma^2 }{32 L^2} \frac{b}{B\eta_0 }   +   {\frac{1}{\eta_T T} \frac{ c^2 \bar{\kappa}^3}{4 L^2}  \log(T+1)}\nonumber \\
   & \qquad \qquad \qquad \qquad \qquad \qquad \quad
   + \frac{1}{\eta_T T}  \frac{c^2 \bar{\kappa}^3}{4 L^2} \frac{ \zeta^2}{ \sigma^2 }  \frac{(I - 1)}{I} \log(T+1).
   \label{Eq: Grad_Norm_Sum1}
\end{align}
where $(a)$ above utilizes the fact that $1/16 < 3/32 < 1/4$.

Now considering each term of \eqref{Eq: Grad_Norm_Sum1} above separately and using the definition of $\displaystyle \eta_t = \frac{\bar{\kappa}}{(w_t + \sigma^2 t)^{1/3}}$ we get from the coefficient of the first term:
\begin{align}
    \frac{1}{{\eta_T T}} = \frac{(w_T + \sigma^2T)^{1/3}}{\bar{\kappa} T} \overset{(a)}{\leq} \frac{w_T^{1/3}}{\bar{\kappa} T} + \frac{ \sigma^{2/3}}{\bar{\kappa} T^{2/3}} \overset{(b)}{\leq}{ \frac{16 LI}{T} + \frac{L}{(b K)^{2/3} T^{2/3}}. }
    \label{Eq: GradNorm_1stTerm}
\end{align}
where inequality $(a)$ follows from identity $(x + y)^{1/3} \leq x^{1/3} + y^{1/3}$ and inequality $(b)$ follows from the definition of $\bar{\kappa}$ and $w_T$
$${ w_T = \max \bigg\{2 \sigma^2,  4096 L^3 I^3\bar{\kappa}^3 - \sigma^2 T,  \frac{c^3 \bar{\kappa}^3 }{4096 L^3I^3} \bigg\} \leq  \sigma^2 \max\bigg\{2 ,~ 4096 I^3 (bK)^{2} - T,  ~\frac{  512}{bK I^3} \bigg\},}$$
where we used $\displaystyle 4096
 L^3 I^3\bar{\kappa}^3  >  4096
 L^3 I^3\bar{\kappa}^3 - \sigma^2 T \geq \max \bigg\{2 \sigma^2,  \frac{c^3 \bar{\kappa}^3 }{4096
 L^3I^3} \bigg\}$. Note that this choice of $w_T$ captures the worst case guarantees for \aname.

Now, let us consider the second term of \eqref{Eq: Grad_Norm_Sum1}, we have from the definition of $\eta_0$ and $\eta_T$ 
\begin{align}
\frac{1}{\eta_T T} \frac{\sigma^2 }{32 L^2} \frac{b}{B\eta_0 }    & \leq  
\bigg( \frac{16 LI}{T} + \frac{L}{(b K)^{2/3} T^{2/3}} \bigg) \times \frac{\sigma^2}{32 L^2} \times \frac{b w_0^{1/3}}{ B \bar{\kappa} }
\nonumber\\
& \overset{(a)}{\leq} \bigg( \frac{16 LI}{T} + \frac{L}{(b K)^{2/3} T^{2/3}} \bigg) \times \frac{\sigma^2}{32 L^2} \times \frac{16 L I b }{ B }\nonumber  \\
& \overset{(b)}{\leq}     \frac{8 b I^2}{ B T} \sigma^2   + \frac{b I}{(bK)^{2/3} B T^{2/3}} \frac{\sigma^2}{2} .   
\label{Eq: GradNorm_2ndTerm}
\end{align}
where inequality $(a)$ follows from the identity $(x + y)^{1/3} \leq x^{1/3} + y^{1/3}$ and $(b)$ follows from the definition of $\bar{\kappa}$ and using $w_0 \leq 4096
 L^3 I^3 \bar{\kappa}^3$ and $w_T \leq 4096 L^3 I^3 \bar{\kappa}^3$ (Similar to the approach in \eqref{Eq: GradNorm_1stTerm} this choice of $w_0$ and $w_T$ capture the worst case convergence guarantees for \aname.)

{Finally, considering the term $\frac{1}{\eta_T T} \frac{ c^2 \bar{\kappa}^3}{4 L^2}$ common to the last two terms in \eqref{Eq: Grad_Norm_Sum1} above, we have from the definition of the stepsize, $\eta_t$,}
{
\begin{align}
 \frac{1}{\eta_T T}   \frac{ c^2 \bar{\kappa}^3}{4 L^2}  & \leq   \bigg( \frac{16 LI}{T} + \frac{L}{(b K)^{2/3} T^{2/3}} \bigg) \times   \bigg(\frac{128 L^2}{bK} \bigg)^2 \times \frac{(bK)^2 \sigma^2}{L^3} \times \frac{1}{4 L^2} \nonumber \\
 & \overset{(a)}{\leq} 256^2 \sigma^2 \frac{  I}{T} + 64^2 \sigma^2 \frac{ 1}{(bK)^{2/3} T^{2/3}}.
    \label{Eq: GradNorm_3rdTerm}
\end{align}
where inequality $(a)$ follows from the identity $(x + y)^{1/3} \leq x^{1/3} + y^{1/3}$ and $(b)$ again uses $w_T \leq 4096
 L^3I^3\bar{\kappa}^3$ along with the definition of $\bar{\kappa}$ and $c$.}

{Finally, substituting the bounds obtained in \eqref{Eq: GradNorm_1stTerm}, \eqref{Eq: GradNorm_2ndTerm} and \eqref{Eq: GradNorm_3rdTerm} into \eqref{Eq: Grad_Norm_Sum1}, we get
\begin{align*}
& \mathbb{E} \|\nabla f(\bar{x}_a) \|^2
     \leq    
   \bigg[\frac{32 LI}{T} + \frac{2 L}{(b K)^{2/3} T^{2/3}}\bigg] (f(\bar{x}_1) -   f^\ast )    
   + \bigg[ \frac{8 b I^2}{B T} + \frac{b I}{2  (b K)^{2/3} B T^{2/3}} \bigg] \sigma^2 \\
 &  \quad  +  \bigg[  \frac{256^2 I}{T} +  \frac{64^2}{(bK)^{2/3} T^{2/3}} \bigg] \sigma^2 \log(T+1)  + \bigg[  \frac{256^2 I}{T} +  \frac{64^2}{(bK)^{2/3} T^{2/3}} \bigg] \zeta^2 \frac{(I - 1)}{I} \log(T+1).
\end{align*}}
Hence, the theorem is proved. 
\end{proof}
Next, using Theorem \ref{Thm: PR_Convergence_Main_App} we prove Theorem \ref{Thm: PR_Convergence_Main}.
 
{
\begin{theorem}[Theorem \ref{Thm: PR_Convergence_Main}: Trade-off: Local Updates vs Batch Sizes]
\label{cor: Trade-off_App}
With the parameters chosen according to Theorem \ref{Thm: PR_Convergence_Main_App} and 
for any $\nu \in [0, 1]$ at each WN we set the total number of local updates as $I = \mathcal{O}\big(({T}/{K^2})^{\nu/3}\big)$, batch size, $b = \mathcal{O} \big(({T}/{K^2})^{{1}/{2} - {\nu}/{2}} \big)$, and the initial batch size, $B = bI$. Then \aname~satisfies:
\begin{enumerate}[leftmargin = 0.6 cm, label = (\roman*)]
     \item We have: 
    \begin{align*}
 \mathbb{E}\| \nabla f(\bar{x}_a) \|^2 = \mathcal{O}\bigg( \frac{f(\bar{x}_1) - f^\ast}{K^{2\nu/3}T^{1 - \nu/3}} \bigg) + \tilde{\mathcal{O}}\bigg( \frac{\sigma^2 }{K^{2\nu/3}T^{1 - \nu/3}}\bigg) + \tilde{\mathcal{O}} \bigg(\frac{(I - 1)}{I} \times \frac{\zeta^2}{K^{2\nu/3}T^{1 - \nu/3}} \bigg).
\end{align*}
\item Sample Complexity: To achieve an $\epsilon$-stationary point \aname~requires at most $\mathcal{O}(\epsilon^{-3/2})$ gradient computations. This implies that each WN requires at most $\mathcal{O}(K^{-1} \epsilon^{-3/2})$ gradient computations, thereby achieving linear speedup with the number of WNs present in the network. 
\item Communication Complexity: To achieve an $\epsilon$-stationary point \aname~requires at most $\mathcal{O}(\epsilon^{-1})$ communication rounds. 
\end{enumerate}
\end{theorem}}
\begin{proof}
{The proof of statement (i) follows from the statement of Theorem \ref{Thm: PR_Convergence_Main_App} and substituting the values of parameters $B$, $I$ and $b$ in the expression. First, replacing $B = b I$ in the statement of Theorem \ref{Thm: PR_Convergence_Main_App} yields
\begin{align*}
& \mathbb{E} \|\nabla f(\bar{x}_a) \|^2
     \leq    
   \bigg[\frac{32 LI}{T} + \frac{2 L}{(b K)^{2/3} T^{2/3}}\bigg] (f(\bar{x}_1) -   f^\ast )    
   + \bigg[ \frac{8 I}{  T} + \frac{1}{2  (b K)^{2/3}  T^{2/3}} \bigg] \sigma^2 \\
 & \quad  +  \bigg[  \frac{256^2 I}{T} +  \frac{64^2}{(bK)^{2/3} T^{2/3}} \bigg] \sigma^2 \log(T+1) + \bigg[  \frac{256^2 I}{T} +  \frac{64^2}{(bK)^{2/3} T^{2/3}} \bigg] \zeta^2 \frac{(I - 1)}{I} \log(T+1).
\end{align*}
Then using the fact that $I = \mathcal{O}\big((T/K^2)^{\nu/3}\big)$ and $b = \mathcal{O}\big((T/K^2)^{1/2 - \nu/2}\big)$ yields the expression of statement $(i)$.
}

Next, we compute the computation and communication complexity of the algorithm.
\begin{itemize}[leftmargin = 0.6 cm]
    \item {\em Sample Complexity} [Theorem \ref{cor: Trade-off_App}(ii)]: From the statement of Theorem \ref{cor: Trade-off_App}(i), total iterations required to achieve an $\epsilon$-stationary point are: 
    \begin{align}
    \label{Eq: Total_Iterations_App}
    \tilde{\mathcal{O}} \bigg( \frac{1}{K^{2\nu/3}T^{1 - \nu/3}} \bigg) = \epsilon \quad \Rightarrow \quad T = \mathcal{\tilde{O}} \bigg( \frac{1}{K^{2\nu/(3 - \nu)}\epsilon^{3/(3- \nu)}} \bigg).
     \end{align}
In each iteration, each WN computes $2b$ stochastic gradients, therefore, the total gradient computations at each WN are $2bT$. Using $b = \mathcal{O}\big(  ( {T}/{K^2} )^{{1}/{2} - {\nu}/{2}} \big)$, we get the total gradient computations required at each WN as:
\begin{align*}
    b T = \tilde{\mathcal{O}} \bigg( \frac{T^{3/2 - \nu/2}}{K^{1 - \nu}} \bigg) \overset{\eqref{Eq: Total_Iterations_App}}{=} \tilde{\mathcal{O}} \bigg( \frac{1}{K \epsilon^{3/2}} \bigg) 
\end{align*}
This implies that the sample complexity is $\tilde{\mathcal{O}}(\epsilon^{-3/2})$.
\item {\em Communication Complexity} [Theorem \ref{cor: Trade-off_App}(iii)]: The total rounds of communication to achieve an $\epsilon$-stationary point are $T/I$, with $I = \mathcal{O}\big( ( {T}/{K^2} )^{\nu/3} \big)$ and $T$ given in \eqref{Eq: Total_Iterations_App}, therefore, we have the communication complexity as:
\begin{align*}
    \frac{T}{I} =  \mathcal{\tilde{O}} \big( T^{1 - \nu/3} K^{2\nu/3}  \big) \overset{\eqref{Eq: Total_Iterations_App}}{=} \mathcal{\tilde{O}} \bigg(\frac{1}{\epsilon} \bigg).
\end{align*}
 \end{itemize} 
Hence, the theorem is proved.
\end{proof}
\begin{cor}[Fed\aname: Local Updates]
\label{cor: LocalComputation_App}
With the choice of parameters given in Theorem \ref{Thm: PR_Convergence_Main_App}. At each WN, setting constant batch size, $b \geq 1$, number of local updates, $I = ({T}/{b^2 K^2})^{1/3}$, and the initial batch size, $B = bI$. Then \aname~satisfies the following: 
\begin{enumerate}[leftmargin = 0.6 cm, label = (\roman*)]
    \item We have: 
    $$\displaystyle \mathbb{E}\| \nabla f(\bar{x}_a) \|^2 = \mathcal{O}\bigg( \frac{f(\bar{x}_1) - f^\ast}{(bK)^{2/3} T^{2/3}} \bigg) + \tilde{\mathcal{O}}\bigg( \frac{\sigma^2}{(bK)^{2/3} T^{2/3}}\bigg) + \tilde{\mathcal{O}} \bigg( \frac{\zeta^2}{(bK)^{2/3} T^{2/3}} \bigg).$$
\item Sample Complexity: To achieve an $\epsilon$-stationary point Fed\aname~requires at most $\mathcal{\tilde{O}}(\epsilon^{-3/2})$ gradient computations while achieving linear speedup with the number of WNs. 
\item Communication Complexity: To achieve an $\epsilon$-stationary point Fed\aname~requires at most $\mathcal{\tilde{O}}(\epsilon^{-1})$ communication rounds. 
\end{enumerate}
\end{cor}

\begin{proof}
The proof of statement (i) follows from substituting the values of the parameters $b$, $I$ and $B$ as defined in the statement of the Corollary in the statement of Theorem \ref{Thm: PR_Convergence_Main_App}.

Next, we compute the sample and communication complexity of the algorithm.
\begin{itemize}[leftmargin = 0.6 cm]
    \item {\em Sample Complexity:} From the statement of Corollary \ref{cor: LocalComputation_App}(i), total iterations, $T$, required to achieve an $\epsilon$-stationary point are:
\begin{align}
\label{Eq: Total_Iterations_App_Local}
   \tilde{\mathcal{O}} \bigg( \frac{1}{(bK)^{2/3} T^{2/3}} \bigg) = \epsilon \qquad \Rightarrow \qquad T = \mathcal{\tilde{O}} \bigg( \frac{1}{bK \epsilon^{3/2}} \bigg).
\end{align}
At each iteration the algorithm computes $2b$ stochastic gradients. Therefore, the total number of gradient computations required at each WN are of the order of $2bT$, which is $\tilde{\mathcal{O}}(K^{-1} \epsilon^{-3/2})$. Therefore, the sample complexity of the algorithm is $\tilde{\mathcal{O}}(\epsilon^{-3/2})$.
\item {\em Communication Complexity:} Total rounds of communication to achieve an $\epsilon$-stationary point is $T/I$, therefore we have from the choice of $I$ that
\begin{align*}
    \frac{T}{I} = \mathcal{\tilde{O}} \big( (bK)^{2/3} T^{2/3} \big) \overset{\eqref{Eq: Total_Iterations_App_Local}}{=} \mathcal{\tilde{O}} \bigg(\frac{1}{\epsilon}\bigg).
\end{align*}
\end{itemize}
Hence, the corollary is proved. 
\end{proof}
 An alternate design choice for the algorithm is to design large batch-size gradients and communicate more often. The next corollary captures this idea.
\begin{cor}[Corollary \ref{cor: Batches}: Minibatch \aname]
\label{cor: Batches_App}
With the choice of parameters given in Theorem \ref{Thm: PR_Convergence_Main_App}. At each WN, choosing the number of local updates, $I = 1$, the batch size, $\displaystyle b = T^{1/2}/K$, and the initial batch size, $B = bI$. Then \aname~satisfies:
\begin{enumerate}[leftmargin = 0.6 cm, label = (\roman*)]
    \item We have: 
    $$\displaystyle \mathbb{E}\| \nabla f(\bar{x}_a) \|^2 = \mathcal{O}\bigg( \frac{f(\bar{x}_1) - f^\ast}{T} \bigg) + \tilde{\mathcal{O}}\bigg( \frac{\sigma^2}{T}\bigg).$$
\item Sample Complexity: To achieve an $\epsilon$-stationary point Minibatch \aname~requires at most $\mathcal{\tilde{O}}(\epsilon^{-3/2})$ gradient computations while achieving linear speedup with the number of WNs. 
\item Communication Complexity: To achieve an $\epsilon$-stationary point Minibatch \aname~requires at most $\mathcal{\tilde{O}}(\epsilon^{-1})$ communication rounds.
\end{enumerate}
\end{cor} 
 \begin{proof}
The proof of statement (i) follows from substituting the values of the parameters $b$, $I$ and $B$ given in the statement of the Corollary in the statement of Theorem \ref{Thm: PR_Convergence_Main_App}.

Next, we compute the sample and communication complexity of the algorithm.
\begin{itemize}[leftmargin = 0.6 cm]
    \item {\em Sample Complexity:} From the statement of Corollary \ref{cor: Batches_App}(i),
    total iterations, $T$, required to achieve an $\epsilon$-stationary point are:
\begin{align}
\label{Eq: Total_Iterations_App_Batch}
    \mathcal{\tilde{O}} \bigg(\frac{I}{T} \bigg) = \epsilon \qquad \Rightarrow \qquad T = \mathcal{\tilde{O}} \bigg( \frac{I}{ \epsilon} \bigg).
\end{align}
In each iteration, each WN computes $2b$ stochastic gradients, therefore, the total gradient computations at each WN are $2bT$. Using the fact that $\displaystyle b = \mathcal{O}\bigg( \frac{T^{1/2}}{I^{3/2}K} \bigg)$. The total gradients computed at each WN to reach an $\epsilon$-stationary point are:
\begin{align*}
    \mathcal{\tilde{O}}\bigg(\frac{I}{\epsilon} \times \frac{I^{1/2}}{\epsilon^{1/2} I^{3/2} K}  \bigg) = \mathcal{\tilde{O}}\bigg( \frac{1}{K \epsilon^{3/2}} \bigg).
\end{align*}
Therefore, the communication complexity if $\tilde{\mathcal{O}}(\epsilon^{-3/2})$.
\item {\em Communication Complexity:} The total rounds of communication required to reach an $\epsilon$-stationary point are $T/I$, therefore we have 
\begin{align*}
    \frac{T}{I} \overset{\eqref{Eq: Total_Iterations_App_Batch}}{=}  \mathcal{\tilde{O}} \bigg( \frac{1}{\epsilon} \bigg).
\end{align*}
\end{itemize}
Hence, the corollary is proved. 
\end{proof}


\section{Proofs of Convergence Guarantees for FedAvg}
\label{App: FedAvg}
In this section, we present the proofs for the FedAvg algorithm. Before stating the proofs in detail we first present some preliminaries lemmas which shall be used for proving the main results of the paper. We first fix some notations:

We define $\bar{t}_s \coloneqq sI + 1$ with $s \in [S]$. Note from Algorithm \ref{Algo_FedAvg} that at $(s \times I)^\text{th}$ iteration, i.e., when $t~ \text{mod}~ I = 0$, the iterates, $\{x_t^{(k)}\}_{k = 1}^K$ corresponding to $t = (\bar{t}_s)^\text{th}$ time instant are shared with the SN. We define the filtration $\mathcal{F}_t$ as the sigma algebra generated by iterates $x_1^{(k)}, x_2^{(k)}, \ldots, x_t^{(k)}$ as
$$\mathcal{F}_t = \sigma(x_1^{(k)}, x_2^{(k)}, \ldots, x_t^{(k)}, ~\text{for all}~k \in [K]).$$

Also, throughout the section we assume Assumptions \ref{Ass: Lip_Smoothness} and \ref{Ass: Unbiased_Var_Grad} to hold. Next, we present the proof of Theorem \ref{Thm: Flexible_FedAvg}. The proof follows in few steps which are discussed next. 

\subsection{Proof of Main Results: FedAvg}

\begin{lem}
\label{Lem: Grad_Variance_FedAvg}
For $\bar{d}_t \coloneqq \frac{1}{K} \sum_{k = 1}^K d_t^{(k)}$ where $d_t^{(k)}$ for all $k \in [K]$ and $t \in [T]$ is chosen according to Algorithm \ref{Algo_FedAvg}, we have:
\begin{align*}
    \mathbb{E}\Big\| \bar{d}_t - \frac{1}{K} \sum_{k=1}^K   \nabla f^{(k)}(x_t^{(k)}) \Big\|^2 \leq \frac{\sigma^2}{bK},
\end{align*}
where the expectation is w.r.t the stochasticity of the the algorithm. 
\end{lem}
\begin{proof}
The proof follows from the same argument as in Lemma \ref{Lem: e_bar_bound_Batch}.
\end{proof}

Next, we bound the error accumulated via the iterates generated by the local updates of Algorithm \ref{Algo_FedAvg}.
\begin{lem}[Error Accumulation from Iterates]
\label{lem: ErrorAccumulation_Iterates_FedAvg}
For the choice of stepsize $\eta \leq \frac{1}{9LI}$,
the iterates $x_t^{(k)}$ for each $k \in [K]$ generated from Algorithm \ref{Algo_FedAvg} satisfy:
\begin{align*}
\sum_{t = 1}^T \frac{1}{K} \sum_{k = 1}^K \mathbb{E}\| x_t^{(k)}-  \bar{x}_t \|^2 \leq 3 \eta^2 (I - 1)   \sigma^2 T  + 5 \eta^2 (I - 1)^2   \zeta^2 T,
\end{align*}
where the expectation is w.r.t the stochasticity of the algorithm.
\end{lem}
\begin{proof}
Note from Algorithm \ref{Algo_FedAvg} and the definition of $\bar{t}_s$ that at $t = \bar{t}_{s - 1}$ with $s \in [S]$, $x_{t}^{(k)} = \bar{x}_{t}$, for all $k$. 
This implies 
$$\frac{1}{K} \sum_{k = 1}^K \| x_{\bar{t}_{s-1}}^{(k)} - \bar{x}_{\bar{t}_{s-1}} \|^2 = 0.$$ 
Therefore, the statement of the lemma holds trivially. 
Moreover, for $t \in [\bar{t}_{s-1} + 1,  \bar{t}_s - 1]$, with $s \in [S]$, we have from Algorithm \ref{Algo_FedAvg}: $x_{t}^{(k)} = x_{t-1}^{(k)} - \eta  d_{t-1}^{(k)}$, this implies that:
\begin{align*}
    x_t^{(k)} = x_{\bar{t}_{s-1}}^{(k)} - \sum_{\ell = \bar{t}_{s-1}}^{t-1} \eta  d_\ell^{(k)} \quad \text{and} \quad \bar{x}_{t}  = \bar{x}_{\bar{t}_{s-1}}  - \sum_{\ell = \bar{t}_{s-1}}^{t-1} \eta  \bar{d}_\ell.
\end{align*}
This implies that for $t \in [\bar{t}_{s-1} + 1,  \bar{t}_s - 1]$, with $s \in [S]$ we have
\begin{align}
 \frac{1}{K} \sum_{k = 1}^K  \| x_t^{(k)}-  \bar{x}_t \|^2 & = \frac{1}{K} \sum_{k = 1}^K \Big\| x_{\bar{t}_{s-1}}^{(k)} - \bar{x}_{\bar{t}_{s-1}}  - \Big( \sum_{\ell = \bar{t}_{s-1}}^{t-1} \eta  d_\ell^{(k)} -   \sum_{\ell =  \bar{t}_{s-1}}^{t-1} \eta  \bar{d}_\ell  \Big) \Big\|^2 \nonumber \\
  & \overset{(a)}{=} \frac{\eta^2}{K} \sum_{k = 1}^K \Big\|  \sum_{\ell = \bar{t}_{s-1}}^{t-1} \big(  d_\ell^{(k)} -      \bar{d}_\ell  \big) \Big\|^2  \nonumber\\
  &  \overset{(b)}{=} \frac{\eta^2 }{K}  \sum_{k = 1}^K \bigg\|   \sum_{\ell = \bar{t}_{s-1}}^{t-1}  \bigg( \frac{1}{b} \!\!\! \sum_{\xi_\ell^{(k)} \in \mathcal{B}_\ell^{(k)}} \nabla f^{(k)} (x_\ell^{(k)}; \xi_\ell^{(k)}) - \frac{1}{K} \sum_{j = 1}^K \frac{1}{b}\!\!\! \sum_{\xi_\ell^{(j)} \in \mathcal{B}_\ell^{(j)} }  \nabla f^{(j)} (x_\ell^{(j)}; \xi_\ell^{(j)})      \bigg)  \bigg\|^2 \nonumber \\
  & \overset{(c)}{\leq} \frac{2 \eta^2 }{K}  \sum_{k = 1}^K \bigg\|   \sum_{\ell = \bar{t}_{s-1}}^{t-1}  \bigg[ \bigg( \frac{1}{b} \!\!\! \sum_{\xi_\ell^{(k)} \in \mathcal{B}_\ell^{(k)}} \nabla f^{(k)} (x_\ell^{(k)}; \xi_\ell^{(k)}) - \nabla f^{(k)} (x_\ell^{(k)})  \bigg) \nonumber\\
  & \qquad \qquad \qquad \qquad   - \frac{1}{K} \sum_{j = 1}^K \bigg(  \frac{1}{b} \!\!\! \sum_{\xi_\ell^{(j)} \in \mathcal{B}_\ell^{(j)} } \nabla f^{(j)} (x_\ell^{(j)}; \xi_\ell^{(j)}) - \nabla f^{(j)} (x_\ell^{(j)})  \bigg)   \bigg]  \bigg\|^2 \nonumber\\
  &    \qquad \qquad \qquad \qquad \qquad   \qquad   + \frac{2 \eta^2}{K} \sum_{k = 1}^K \bigg\| \sum_{\ell = \bar{t}_{s-1}}^{t-1}  \bigg( \nabla f^{(k)} (x_\ell^{(k)})  - \frac{1}{K} \sum_{j = 1}^K \nabla f^{(j)} (x_\ell^{(j)}) \bigg) \bigg\|^2  \nonumber\\
   & \overset{(d)}{\leq} \frac{2 \eta^2 }{K}  \sum_{k = 1}^K \bigg\|   \sum_{\ell = \bar{t}_{s-1}}^{t-1}    \bigg( \frac{1}{b} \sum_{\xi_\ell^{(k)} \in \mathcal{B}_\ell^{(k)}} \nabla f^{(k)} (x_\ell^{(k)}; \xi_\ell^{(k)}) - \nabla f^{(k)} (x_\ell^{(k)})  \bigg) \bigg\|^2 \nonumber\\
  &    \qquad \qquad        + \frac{2 \eta^2}{K} \sum_{k = 1}^K \bigg\| \sum_{\ell = \bar{t}_{s-1}}^{t-1}  \bigg( \nabla f^{(k)} (x_\ell^{(k)})  - \frac{1}{K} \sum_{j = 1}^K \nabla f^{(j)} (x_\ell^{(j)}) \bigg) \bigg\|^2,
  \label{Eq: ConsensusError_FedAvg}
\end{align}
where the equality $(a)$ follows from the fact that $x_{\bar{t}_{s-1}}^{(k)} = \bar{x}_{\bar{t}_{s-1}}$ for $t = \bar{t}_{s - 1}$; $(b)$ results from the definition of the stochastic gradient employed by FedAvg in Algorithm \ref{Algo_FedAvg}; $(c)$ uses Lemma \ref{Lem: Norm_Ineq} and $(d)$ follows from the application of Lemma \ref{Lem: Sum_Mean_Kron}.

Taking expectation on both sides and let us next consider each term of \eqref{Eq: ConsensusError_FedAvg} above separately, we have for any $k \in [K]$ from the first term of \eqref{Eq: ConsensusError_FedAvg} above
\begin{align}
    \mathbb{E} \bigg\|   \sum_{\ell = \bar{t}_{s-1}}^{t-1}    \bigg( \frac{1}{b} \!\!\! \sum_{\xi_\ell^{(k)} \in \mathcal{B}_\ell^{(k)}} \!\!\!  \nabla f^{(k)} (x_\ell^{(k)}; \xi_\ell^{(k)}) - \nabla f^{(k)} (x_\ell^{(k)})  \bigg) \bigg\|^2 & \overset{(a)}{=}  \sum_{\ell  = \bar{t}_{s-1}}^{t-1}     \mathbb{E} \bigg\|   \frac{1}{b} \!\!\! \sum_{\xi_\ell^{(k)} \in \mathcal{B}_\ell^{(k)}} \!\!\! \nabla f^{(k)} (x_\ell^{(k)}; \xi_\ell^{(k)}) - \nabla f^{(k)} (x_\ell^{(k)})   \bigg\|^2 \nonumber \\
  & \overset{(b)}{=}    \sum_{\ell  = \bar{t}_{s-1}}^{t-1}     \frac{1}{b^2} \!\!\! \sum_{\xi_\ell^{(k)} \in \mathcal{B}_\ell^{(k)}} \!\!\! \mathbb{E} \big\|    \nabla f^{(k)} (x_\ell^{(k)}; \xi_\ell^{(k)}) - \nabla f^{(k)} (x_\ell^{(k)})   \big\|^2 \nonumber \\
  &\overset{(c)}{\leq}   \frac{(I - 1)}{b}\sigma^2 \nonumber\\
   &\overset{(d)}{\leq}  (I - 1) \sigma^2 ,
   \label{Eq: GradVar_FedAvg}
\end{align}
where $(a)$ results from the fact that $\mathbb{E} \Big[ \frac{1}{b} \sum_{\xi_\ell^{(k)} \in \mathcal{B}_\ell^{(k)}} \nabla f^{(k)} (x_\ell^{(k)}; \xi_\ell^{(k)}) - \nabla f^{(k)} (x_\ell^{(k)}) \Big| \mathcal{F}_{\bar{\ell}} \Big] = 0$ for any $\bar{\ell} < \ell$; $(b)$ uses the fact that $  \mathbb{E} \big[ \nabla f^{(k)} (x_\ell^{(k)}; \xi_\ell^{(k)}) - \nabla f^{(k)} (x_\ell^{(k)}) \big| \nabla f^{(k)} (x_\ell^{(k)}; \zeta_\ell^{(k)}) - \nabla f^{(k)} (x_\ell^{(k)}) \big] = 0 $ for samples $\xi_\ell^{(k)}, \zeta_\ell^{(k)} \sim \mathcal{D}^{(k)}$ chosen independent; $(c)$ utilizes intra-node variance bound in Assumption \ref{Ass: Unbiased_Var_Grad}(ii) and the fact that $(t - 1) - \bar{t}_{s - 1}  \leq I - 1$ for $t \in [\bar{t}_{s- 1} + 1 , \bar{t}_s - 1]$; and finally, $(d)$ uses the fact that $b \geq 1$. 

Next, we consider the second term of \eqref{Eq: ConsensusError_FedAvg} for any $k \in [K]$, we have
\begin{align}
  &  \sum_{k = 1}^K \mathbb{E} \bigg\| \sum_{\ell = \bar{t}_{s-1}}^{t-1}  \bigg( \nabla f^{(k)} (x_\ell^{(k)})  - \frac{1}{K} \sum_{j = 1}^K \nabla f^{(j)} (x_\ell^{(j)}) \bigg) \bigg\|^2  \nonumber\\
  & \overset{(a)}{\leq} (I - 1)\sum_{\ell = \bar{t}_{s-1}}^{t-1} \sum_{k = 1}^K \mathbb{E} \bigg\|    \nabla f^{(k)} (x_\ell^{(k)})  - \frac{1}{K} \sum_{j = 1}^K \nabla f^{(j)} (x_\ell^{(j)})   \bigg\|^2   \nonumber\\
  & \overset{(b)}{\leq} (I - 1) \sum_{\ell = \bar{t}_{s-1}}^{t-1} \bigg[  4 \sum_{k = 1}^K  \mathbb{E}     \big\|    \nabla f^{(k)} (x_\ell^{(k)})  -  \nabla f^{(k)}(\bar{x}_\ell)   \big\|^2 + 4 \sum_{k = 1}^K \mathbb{E} \bigg\|    \nabla f (\bar{x}_\ell)  - \frac{1}{K} \sum_{j = 1}^K \nabla f (x_\ell^{(j)})   \bigg\|^2 \nonumber\\
  &   \quad \qquad \qquad \qquad \qquad \qquad \qquad \qquad \qquad \qquad \qquad \qquad \quad +  2 \sum_{k = 1}^K  \mathbb{E}\big\|    \nabla f^{(k)} (\bar{x}_\ell)  -   \nabla f (\bar{x}_\ell)   \big\|^2  \bigg] \nonumber\\
&  \overset{(c)}{\leq} (I - 1) \sum_{\ell = \bar{t}_{s-1}}^{t-1} \bigg[  8 L^2 \sum_{k = 1}^K  \mathbb{E}     \big\|x_\ell^{(k)}  -  \bar{x}_\ell \big\|^2  +  2 \sum_{k = 1}^K  \mathbb{E}\bigg\|    \nabla f^{(k)} (\bar{x}_\ell)  -   \frac{1}{K} \sum_{j = 1}^K \nabla f^{(j)} (\bar{x}_\ell)   \bigg\|^2  \bigg] \nonumber\\
& \overset{(d)}{\leq} 8 L^2 (I - 1) \sum_{\ell = \bar{t}_{s-1}}^{t-1}     \sum_{k = 1}^K  \mathbb{E}     \big\|x_\ell^{(k)}  -  \bar{x}_\ell \big\|^2  +  2 K  (I - 1)^2 \zeta^2,
\label{Eq: InterNodeVar_FedAvg}
\end{align}
where $(a)$ utilizes the fact that $(t - 1) - \bar{t}_{s - 1}    \leq I - 1$ for $t \in [\bar{t}_{s- 1} + 1, \bar{t}_s - 1]$; $(b)$ results from the application of Lemma \ref{Lem: Norm_Ineq}; $(c)$ follows from Assumption \ref{Ass: Lip_Smoothness}; and $(d)$ utilizes the inter-node variance Assumption \ref{Ass: Unbiased_Var_Grad} and the fact that $(t - 1) - \bar{t}_{s - 1}  \leq I - 1$ for $t \in [\bar{t}_{s- 1} + 1 , \bar{t}_s - 1]$. 

Substituting \eqref{Eq: GradVar_FedAvg} and  \eqref{Eq: InterNodeVar_FedAvg} in \eqref{Eq: ConsensusError_FedAvg} and taking expectation on both sides we get
\begin{align*}
     \frac{1}{K} \sum_{k = 1}^K \mathbb{E} \| x_t^{(k)}-  \bar{x}_t \|^2 & \leq 2 \eta^2 (I - 1) \sigma^2      +  4 \eta^2 (I - 1)^2 \zeta^2     + 16 L^2 (I - 1) \eta^2 \sum_{\ell = \bar{t}_{s-1}}^{t - 1} \frac{1}{K} \sum_{k = 1}^k \mathbb{E} \| x_\ell^{(k)} - \bar{x}_\ell \|^2.
\end{align*}
Summing both sides from $t = \bar{t}_{s - 1}$ to $\bar{t}_{s} - 1$, we get
\begin{align*}
    \sum_{t = \bar{t}_{s-1}}^{\bar{t}_s - 1}  \frac{1}{K} \sum_{k = 1}^K \mathbb{E} \| x_t^{(k)} -  \bar{x}_t \|^2 
   & \leq 2 \eta^2 (I - 1) \sigma^2 I     +  4 \eta^2 (I - 1)^2 \zeta^2  I + 16 L^2 (I - 1) \eta^2  \sum_{t = \bar{t}_{s-1}}^{\bar{t}_s - 1}  \sum_{\ell = \bar{t}_{s-1}}^{t - 1} \frac{1}{K} \sum_{k = 1}^K \mathbb{E} \| x_\ell^{(k)} - \bar{x}_\ell \|^2 \\
     & \overset{(a)}{\leq} 2 \eta^2 (I - 1) \sigma^2 I     +  4 \eta^2 (I - 1)^2 \zeta^2  I    + 16 L^2 (I - 1) \eta^2  \sum_{t = \bar{t}_{s-1}}^{\bar{t}_s - 1}  \sum_{\ell = \bar{t}_{s-1}}^{\bar{t}_s - 1} \frac{1}{K} \sum_{k = 1}^K \mathbb{E} \| x_\ell^{(k)} - \bar{x}_\ell \|^2 \\
   & \overset{(b)}{\leq} 2 \eta^2 (I - 1) \sigma^2 I     +  4 \eta^2 (I - 1)^2 \zeta^2  I     + 16 L^2 (I - 1) \eta^2 I \sum_{t = \bar{t}_{s-1}}^{\bar{t}_s - 1}   \frac{1}{K} \sum_{k = 1}^K \mathbb{E} \| x_t^{(k)} - \bar{x}_t \|^2,
\end{align*}
where $(a)$ uses that fact that $t \leq \bar{t}_s - 1$; $(b)$ results from $t_s - t_{s - 1} \leq I$ for all $s\in [S]$. Finally, summing over $s \in [S]$ and using $T = SI$ we get
\begin{align*}
     \sum_{t = 1}^{T}  \frac{1}{K} \sum_{k = 1}^K \mathbb{E} \| x_t^{(k)} -  \bar{x}_t \|^2 & \leq 2 \eta^2 (I - 1) \sigma^2 T     +  4 \eta^2 (I - 1)^2 \zeta^2  T       + 16 L^2 I^2 \eta^2  \sum_{t = 1}^{T}   \frac{1}{K} \sum_{k = 1}^K \mathbb{E} \| x_t^{(k)} - \bar{x}_t \|^2. 
\end{align*}
Rearranging the terms, we get
\begin{align*}
    (1 - 16L^2 I^2 \eta^2) \sum_{t = 1}^{T}  \frac{1}{K} \sum_{k = 1}^K \mathbb{E} \| x_t^{(k)} -  \bar{x}_t \|^2 & \leq 2 \eta^2 (I - 1) \sigma^2 T     +  4 \eta^2 (I - 1)^2 \zeta^2  T  .
\end{align*}
Finally, using the fact that $\eta \leq \frac{1}{9 L I}$ we have $1 - 16 L^2I^2 \eta^2 \geq 4/5$. Multiplying, both sides by $5/4$ we get
\begin{align*}
      \sum_{t = 1}^{T}  \frac{1}{K} \sum_{k = 1}^K \mathbb{E} \| x_t^{(k)} -  \bar{x}_t \|^2 & \leq 3 \eta^2 (I - 1) \sigma^2 T     +  5 \eta^2 (I - 1)^2 \zeta^2  T  .
\end{align*}
Therefore, the lemma is proved. 
\end{proof}

\begin{lem}[Descent Lemma]
\label{lem: Fn_Descent_FedAvg}
For all $t \in [\bar{t}_{s-1}, \bar{t}_s - 1]$ and $s \in [S]$, with the choice of stepsizes $\eta \leq \frac{1}{9 L I}$, the iterates generated by Algorithm \ref{Algo_FedAvg} satisfy:
\begin{align*}
    \mathbb{E} f(\bar{x}_{t + 1}) & \leq      \mathbb{E} f(\bar{x}_{t })    - \frac{\eta}{2}     \mathbb{E}\|\nabla f(\bar{x}_t) \|^2  + \frac{\eta L^2}{2K}  \sum_{k=1}^K \mathbb{E} \| x_t^{(k)} - \bar{x}_t   \|^2  +   \frac{\eta^2 L}{bK} \sigma^2, 
\end{align*}
where the expectation is w.r.t the stochasticity of the algorithm.
\end{lem}
\begin{proof}
Using the smoothness of $f$ (Assumption \ref{Ass: Lip_Smoothness}) we have:
\begin{align*}
  \mathbb{E}[  f(\bar{x}_{t + 1}) ] 
  & \leq \mathbb{E} \Big[ f(\bar{x}_{t }) + \langle \nabla f(\bar{x}_{t}),  \bar{x}_{t + 1} - \bar{x}_{t}\rangle + \frac{L}{2} \| \bar{x}_{t + 1} - \bar{x}_{t } \|^2 \Big] \nonumber\\
    &  \overset{(a)}{=}\mathbb{E} \Big[ f(\bar{x}_{t}) - \eta \langle \nabla f(\bar{x}_{t}),  \bar{d}_t \rangle + \frac{\eta^2 L}{2} \| \bar{d}_{t}  \|^2 \Big] \nonumber\\
     & \overset{(b)}{=} \mathbb{E} \Big[ f(\bar{x}_{t}) - \eta \Big\langle \nabla f(\bar{x}_{t}),  \frac{1}{K} \sum_{k = 1}^K \nabla f^{(k)}(x_t^{(k)}) \Big\rangle + \frac{\eta^2 L}{2} \| \bar{d}_{t}  \|^2  \Big] \nonumber\\
       & \overset{(c)}{=}    \mathbb{E} \bigg[ f(\bar{x}_{t}) - \frac{\eta}{2}  \Big\| \frac{1}{K} \sum_{k = 1}^K \nabla f^{(k)}(x_t^{(k)})  \Big\|^2  - \frac{\eta}{2} \| \nabla f(\bar{x}_{t}) \|^2   + \frac{\eta}{2} \Big\|  \nabla f(\bar{x}_{t})  - \frac{1}{K} \sum_{k = 1}^K \nabla f^{(k)}(x_t^{(k)})  \Big\|^2 \nonumber \\
       & \qquad \qquad \qquad   \qquad \qquad + \eta^2 L \Big\| \bar{d}_{t} - \frac{1}{K} \sum_{k = 1}^K \nabla f^{(k)}(x_t^{(k)}) \Big\|^2  + \eta^2 L   \Big\| \frac{1}{K} \sum_{k = 1}^K \nabla f^{(k)}(x_t^{(k)}) \Big\|^2   \bigg] \nonumber \\  
        & \overset{(d)}{\leq} \mathbb{E} \bigg[    f(\bar{x}_{t }) -  \left( \frac{\eta}{2} -  \eta^2 L  \right)  \Big\| \frac{1}{K} \sum_{k = 1}^K \nabla f^{(k)}(x_t^{(k)})  \Big\|^2 - \frac{\eta}{2} \|\nabla f(\bar{x}_t) \|^2  + \frac{\eta L^2}{2K}  \sum_{k=1}^K \| x_t^{(k)} - \bar{x}_t   \|^2       +
    \frac{\eta^2 L}{bK} \sigma^2 \bigg] \\
    & \overset{(e)}{\leq} \mathbb{E} \bigg[    f(\bar{x}_{t })  - \frac{\eta}{2} \|\nabla f(\bar{x}_t) \|^2  + \frac{\eta L^2}{2K}  \sum_{k=1}^K \| x_t^{(k)} - \bar{x}_t   \|^2       + 
    \frac{\eta^2 L}{bK} \sigma^2 \bigg],
\end{align*}
 where equality $(a)$ follows from the iterate update given in Step 5 of Algorithm \ref{Algo_FedAvg}; $(b)$ results from the fact that we have $\mathbb{E}[ \nabla f^{(k)}(x_t^{(k)} ; \xi_t^{(k)}) | \mathcal{F}_t ] = \nabla f^{(k)}(x_t^{(k)})$; $(c)$ uses $\langle a , b \rangle = \frac{1}{2} [\|a\|^2 + \|b\|^2 - \|a - b \|^2]$ and Lemma \ref{Lem: Norm_Ineq}; $(d)$ results from \eqref{Eq: Grad_Lip_FedAvg} below and Lemma \ref{Lem: Grad_Variance_FedAvg}; and $(e)$ results from the stepsize choice of $\eta \leq \frac{1}{9LI}$.
\begin{align}
      \mathbb{E}  \bigg\| \frac{1}{K}\sum_{k=1}^K \big(\nabla f^{(k)}(x^{(k)}_{t}) - \nabla f^{(k)}(\bar{x}_t) \big) \bigg\|^2 & \leq \frac{1}{K} \sum_{k=1}^K     \mathbb{E}\big\| \nabla f^{(k)}(x^{(k)}_{t}) - \nabla f^{(k)}(\bar{x}_t)  \big\|^2 \nonumber\\
    & \leq \frac{L^2}{K} \sum_{k=1}^K     \mathbb{E} \|x_t^{(k)} - \bar{x}_t\|^2,
    \label{Eq: Grad_Lip_FedAvg}
\end{align}
where the first inequality follows from Lemma \ref{Lem: Norm_Ineq}, and the second results from Assumption \ref{Ass: Lip_Smoothness}. 

Hence, the lemma is proved.
\end{proof}

\subsubsection{Proof of Theorem \ref{Thm: Flexible_FedAvg}}
The proof of Theorem \ref{Thm: Flexible_FedAvg} follows by replacing the choices of $b$ and $I$ given in \eqref{eq:I:b_FedAvg} in the following result. 
\begin{theorem}
Under Assumptions \ref{Ass: Lip_Smoothness} and \ref{Ass: Unbiased_Var_Grad}, with stepsize $\eta = \sqrt{\frac{bk}{T}}$. Then for $T \geq 81 L^2I^2 bK$ with any choice of minibatch sizes, $b \geq 1$, and number of local updates, $I \geq 1$, the iterates generated from Algorithm \ref{Algo_FedAvg} satisfy
\begin{align*}
    \mathbb{E} \| \nabla f(\bar{x}_a) \|^2 \leq  \frac{2 (f(\bar{x}_t)) - f^\ast)}{(bk)^{1/2} T^{1/2}} + \frac{2 L}{(bk)^{1/2} T^{1/2}} \sigma^2 + \frac{3 L^2 bK (I - 1)}{T} \sigma^2 + \frac{5 L^2 bK (I - 1)^2}{T} \zeta^2.
\end{align*}
\end{theorem}
\begin{proof}
Summing the result of Lemma \ref{lem: Fn_Descent_FedAvg} for $t = [T]$ and multiplying both sides by $2/\eta T$ we get
\begin{align*}
   \frac{1}{T}  \sum_{t = 1}^T \mathbb{E}\|\nabla f(\bar{x}_t)\|^2  & \leq \frac{2 (f(\bar{x}_t) - f(\bar{x}_{t+1}))}{\eta T} + \frac{2 \eta L}{bK} \sigma^2 + \frac{L^2 }{T} \sum_{t = 1}^T \frac{1}{K} \sum_{k = 1}^K \mathbb{E} \|x_t^{(k)} - \bar{x}_t \|^2 \\
   & \leq \frac{2 (f(\bar{x}_t) - f^\ast)}{\eta T} + \frac{2 \eta L}{bK} \sigma^2 + \frac{L^2 }{T} \sum_{t = 1}^T \frac{1}{K} \sum_{k = 1}^K \mathbb{E} \|x_t^{(k)} - \bar{x}_t \|^2
\end{align*}
where the second inequality uses $f(\bar{x}_{t - 1}) \geq f^\ast$. Next, using Lemma \ref{lem: ErrorAccumulation_Iterates_FedAvg} we get
\begin{align*}
   \frac{1}{T}  \sum_{t = 1}^T \mathbb{E}\|\nabla f(\bar{x}_t)\|^2  &   \leq \frac{2 (f(\bar{x}_t) - f^\ast)}{\eta T} + \frac{2 \eta L}{bK} \sigma^2 + 3 L^2 \eta^2 (I - 1) \sigma^2 + 5 L^2 \eta^2 (I - 1)^2 \zeta^2.
\end{align*}
Finally, using the definition of $\bar{x}_a$ from Algorithm \ref{Algo_FedAvg} and the choice of $\eta = \sqrt{ \frac{bK}{T}}$, we get
\begin{align*}
   \mathbb{E}\|\nabla f(\bar{x}_a)\|^2  & \leq \frac{2 (f(\bar{x}_t) - f^\ast)}{(bK)^{1/2} T^{1/2}} + \frac{2   L}{(bK)^{1/2} T^{1/2}} \sigma^2 + \frac{3 L^2 bK (I - 1)}{T} \sigma^2 + \frac{5 L^2 bK (I - 1)^2}{T} \zeta^2.
\end{align*}
Therefore, we have the theorem. 
\end{proof}
Finally, substituting the choice of $I$ and $b$ given in \eqref{eq:I:b_FedAvg} we get the statement of Theorem \ref{Thm: Flexible_FedAvg}. Next two remarks characterize the behavior of FedAvg for two extreme choices of $I$ and $b$.

\begin{rem}[FedAvg: multiple local updates] Choosing $\nu = 1$ in Theorem \ref{Thm: Flexible_FedAvg} implies $I = (T/b^3 K^3)^{1/4}$ and $b = \mathcal{O}(1)$, we have
\begin{align*}
\mathbb{E}\| \nabla f(\bar{x}_a)\|^2 = \mathcal{O}\bigg( \frac{f(\bar{x}_1) - f^\ast}{K^{1/2} T^{1/2}}\bigg) + \mathcal{O}\bigg( \frac{\sigma^2}{K^{1/2} T^{1/2}} \bigg)  + \mathcal{O}\bigg( \frac{\zeta^2}{K^{1/2} T^{1/2}}\bigg),
\end{align*}
while the sample and communication complexities are still $\mathcal{O}(\epsilon^{-2})$ and $\mathcal{O}(\epsilon^{-3/2})$, respectively. Note that these are the same guarantees for FedAvg analyzed in \cite{Yu_Jin_Arxiv_2019linear, Yu_Jin_PMLR_2019dynamicbatches}. \qed
\end{rem}

\begin{rem}[FedAvg: large batch] Choosing $\nu = 0$ in Theorem \ref{Thm: Flexible_FedAvg} implies
$I = \mathcal{O}(1) > 1$ (we allow multiple local updates, i.e. $I > 1$) and $b = (T/I^4 K^3)^{1/3}$, then we have
\begin{align*}
\mathbb{E}\| \nabla f(\bar{x}_a)\|^2 = \mathcal{O}\bigg( \frac{f(\bar{x}_1) - f^\ast}{T^{2/3}}\bigg) + \mathcal{O}\bigg( \frac{\sigma^2}{ T^{2/3}} \bigg)  + \mathcal{O}\bigg( \frac{\zeta^2}{ T^{2/3}} \bigg) .
\end{align*}
while the sample and communication complexities are again $\mathcal{O}(\epsilon^{-2})$ and $\mathcal{O}(\epsilon^{-3/2})$, respectively. \qed
\end{rem}
{\bf Minibatch SGD:} When the parameters are shared after each local update, for such case we have $I = 1$ and for the choice of $b = \mathcal{O}(T/K)$ we have:
\begin{align*}
\mathbb{E}\| \nabla f(\bar{x}_a)\|^2 = \mathcal{O}\bigg( \frac{f(\bar{x}_1) - f^\ast}{T}\bigg) + \mathcal{O}\bigg( \frac{\sigma^2}{ T} \bigg).
\end{align*}
This implies that the sample and communication complexitiess are $\mathcal{O}(\epsilon^{-2})$ and $\mathcal{O}(\epsilon^{-1})$. 
Again, this result is independent of the heterogeniety parameter $\zeta$ (cf. Assumption \ref{Ass: Unbiased_Var_Grad}) as the algorithm for $I = 1$ is essentially a centralized algorithm. 

\section{Useful lemmas}
\label{App: Useful_Lemmas}
\begin{lem}
\label{Lem: Sum_Mean_Kron}
For a finite sequence $x^{(k)} \in \mathbb{R}^d$ for $k \in [K]$ define $\bar{x} \coloneqq \frac{1}{K} \sum_{k = 1}^K x^{(k)}$, we then have
\begin{align*}
\sum_{k=1}^K    \| x^{(k)} - \bar{x} \|^2 \leq \sum_{k=1}^K    \| x^{(k)} \|^2. 
\end{align*}
\end{lem}
\begin{proof}
Using the notation $\mathbf{x} = \Big[{(x^{(1)})}^T,{(x^{(2)})}^T, \ldots, {(x^{(K)})}^T  \Big]^T \in \mathbb{R}^{Kd}$, denoting $\mathbf{I}_d \in \mathbb{R}^{d \times d}$ and $\mathbf{I}_{Kd} \in \mathbb{R}^{Kd \times Kd} $ as identity matrices and representing $\mathbf{1} \in \mathbb{R}^K$ as the vector of all ones. We rewrite the left hand side of the statement as 
\begin{align*}
 \sum_{k=1}^K    \| x^{(k)} - \bar{x} \|^2   & =  \bigg\|\mathbf{x} - \bigg(\mathbf{I} \otimes \frac{\mathbf{11}^T}{K} \bigg)\mathbf{x} \bigg\|^2 \\
 & =  \bigg\|\bigg(\mathbf{I}_{Kd} - \bigg(\mathbf{I}_d \otimes \frac{\mathbf{11}^T}{K} \bigg) \bigg)\mathbf{x} \bigg\|^2 \\
 & \overset{(a)}{\leq} \| \mathbf{x}  \|^2 = \sum_{k = 1}^K \|x^{(k)}\|^2,
\end{align*}
where $(a)$ follows from the fact that the induced matrix norm $\bigg\|\mathbf{I}_{Kd} - \bigg(\mathbf{I}_d \otimes \frac{\mathbf{11}^T}{K} \bigg)    \bigg\| \leq 1$.
\end{proof}

\begin{lem}[From \cite{Cutkosky_NIPS2019}]
\label{Lem: AD_Sum_1overT}
Let $a_0 > 0$ and $a_1,a_2, \ldots, a_T \geq 0$. We have
$$\sum_{t=1}^T \frac{a_t}{a_0 + \sum_{i=t}^t a_i} \leq \ln \bigg(1 + \frac{\sum_{i=1}^t a_i}{a_0} \bigg).$$
\end{lem}

	\begin{lem}
		\label{Lem: Norm_Ineq}
		For $X_1, X_2, \ldots, X_n \in \mathbb{R}^d$, we have 
		\begin{align*}
			\|X_1 + X_2 + \ldots + X_n \|^2 \leq n \| X_1\|^2 + n \| X_2\|^2+ \ldots + n \| X_n\|^2.
		\end{align*}
	\end{lem}

\end{document}